\documentclass[11pt,twoside]{article}
\usepackage{fullpage}
\usepackage{ifthen}
\usepackage{epsf}
\usepackage{fancyheadings}
\usepackage{hyperref}
\usepackage{url}
\usepackage{enumerate}

 \usepackage{amsthm}
 \usepackage{amsfonts}
 \usepackage{amsmath}
 \usepackage{amssymb}

\usepackage{pstricks,pst-node,pst-text,pst-3d}
 \usepackage{color}
\usepackage{microtype}
\usepackage{multicol}

 \usepackage{amsfonts}
 \usepackage{amsmath}
 \usepackage{amssymb}

\usepackage{graphics}
\usepackage{graphicx}
\usepackage{psfrag}
\usepackage{pgf,tikz}
\usepackage{mathrsfs}
\usepackage{float}
\usetikzlibrary{arrows}

\usepackage{epsf}
 \usepackage{subfigure}
 \usepackage{caption}

\newlength{\widebarargwidth}
\newlength{\widebarargheight}
\newlength{\widebarargdepth}
\DeclareRobustCommand{\widebar}[1]{%
  \settowidth{\widebarargwidth}{\ensuremath{#1}}%
  \settoheight{\widebarargheight}{\ensuremath{#1}}%
  \settodepth{\widebarargdepth}{\ensuremath{#1}}%
  \addtolength{\widebarargwidth}{-0.3\widebarargheight}%
  \addtolength{\widebarargwidth}{-0.3\widebarargdepth}%
  \makebox[0pt][l]{\hspace{0.3\widebarargheight}%
    \hspace{0.3\widebarargdepth}%
    \addtolength{\widebarargheight}{0.3ex}%
    \rule[\widebarargheight]{0.95\widebarargwidth}{0.1ex}}%
  {#1}}

\makeatletter
\long\def\@makecaption#1#2{
        \vskip 0.8ex
        \setbox\@tempboxa\hbox{\small {\bf #1:} #2}
        \parindent 1.5em  
        \dimen0=\hsize
        \advance\dimen0 by -3em
        \ifdim \wd\@tempboxa >\dimen0
                \hbox to \hsize{
                        \parindent 0em
                        \hfil 
                        \parbox{\dimen0}{\def\baselinestretch{0.96}\small
                                {\bf #1.} #2
                                } 
                        \hfil}
        \else \hbox to \hsize{\hfil \box\@tempboxa \hfil}
        \fi
        }
\makeatother

\theoremstyle{plain}

\newtheorem{theos}{Theorem}
\newtheorem{props}{Proposition}
\newtheorem{lems}{Lemma}

\newtheorem{cors}{Corollary}

\theoremstyle{remark}

\theoremstyle{remark}

\theoremstyle{remark}

\long\def\comment#1{}

\newcommand{\inprod}[2]{\ensuremath{\langle #1 , \, #2 \rangle}}

\newcommand{\Exs}{\ensuremath{{\mathbb{E}}}}

\newcommand{\numobs}{\ensuremath{n}}

\newcommand{\usedim}{\ensuremath{d}}

\DeclareMathOperator{\var}{var}
\DeclareMathOperator{\cov}{cov}

\DeclareMathOperator*{\argmax}{arg\, max}

\newcommand{\NORMAL}{\ensuremath{\mathcal{N}}}

\newcommand{\Indi}{\mathbf{1}}

\newcommand{\thetastar}{\ensuremath{\theta^*}}
\newcommand{\thetahat}{\ensuremath{\widehat{\theta}}}
\newcommand{\thetatil}{\ensuremath{\widetilde{\theta}}}

\newcommand{\widgraph}[2]{\includegraphics[keepaspectratio,width=#1]{#2}}

\def\d{ \mathrm{d} }
\def\NN{ \mathbb{N} }						
\def\ZN{ \mathbb{Z} }						
\def\RN{ \mathbb{R} }						

\def\EE{ \mathbb{E} }
\def\E{ \mathrm{e} }							

\newcommand{\norm}[1]{\ensuremath{\|#1\|_2}}

\newcommand{\subsize}{\numobs} 
\newcommand{\subprob}{\delta}
\newcommand{\consteps}{C_\epsilon}

\newcommand{\blocksize}{\ensuremath{m}}
\newcommand{\tk}{\ensuremath{\tilde{k}}}

\newcommand{\nstates}{s}

\newcommand{\tvnorm}[2]{\ensuremath{\left\|#1 - #2\right\|_{\mathrm{TV}}}}

\newcommand{\EEzcondx}[3]{\ensuremath{\EE_{#1|#2,#3}}}
\newcommand{\EExcondparam}[2]{\ensuremath{\EE_{#1 \mid #2}}}

\newcommand{\transprob}[2]{p(#1|#2; \paramtrans)}
\newcommand{\obsprob}[2]{p(#1|#2; \paramobs )}

\newcommand{\generalp}{\tilde{p}}

\newcommand{\Tmat}{\ensuremath{A}}

\newcommand{\pistat}{\ensuremath{\widebar{\pi}}}
\newcommand{\stat}{\pistat}
\newcommand{\statmin}{\stat_{\min}}

\newcommand{\mixcoef}{\ensuremath{\rho_{\mathrm{mix}}}}

\newcommand{\mixcoefeff}{\ensuremath{\widetilde{\rho}_{\mathrm{mix}}}}
\newcommand{\mixcoefeps}{\ensuremath{\epsilon_{\mathrm{mix}}}}
\newcommand{\mixcoefbound}{\ensuremath{b}}
\newcommand{\paramtransbound}{\ensuremath{\beta_B}}

\newcommand{\weightsingle}[2]{w_{#1}(#2)}
\newcommand{\weightdouble}[2]{v_{#1}(#2)}

\newcommand{\epsilonobs}{\epsilon^\paramobs}
\newcommand{\epsilontrans}{\epsilon^\paramtrans}

\newcommand{\paramobs}{\mu}
\newcommand{\paramobstil}{\widetilde{\paramobs}}
\newcommand{\trueparamobs}{\ensuremath{\paramobs^*}}

\newcommand{\paramobsone}{\ensuremath{\paramobs}}
\newcommand{\paramobstwo}{\ensuremath{\paramobs^\prime}}

\newcommand{\paramitsamp}{\ensuremath{\widehat{\paramobs}}} 

\newcommand{\paramobstilde}{\widetilde{\paramobs}}

\newcommand{\paramtrans}{\beta}
\newcommand{\paramtranstil}{\widetilde{\paramtrans}}
\newcommand{\paramtransone}{\paramtrans}
\newcommand{\paramtranstwo}{\paramtrans'}
\newcommand{\trueparamtrans}{\paramtrans^*}
\newcommand{\paramspacetrans}{\ensuremath{\Omega_\paramtrans}}
\newcommand{\paramtranstilde}{\widetilde{\paramtrans}}

\newcommand{\paramspacep}{\ensuremath{\Omega_\matprob}}

\newcommand{\paramjoint}{\theta}
\newcommand{\paramjointtil}{\widetilde{\theta}}

\newcommand{\trueparamjoint}{\paramjoint^*}
\newcommand{\paramjointone}{\paramjoint}
\newcommand{\paramjointtwo}{\paramjoint'}
\newcommand{\paramjointwo}{\paramjointtwo}

\newcommand{\paramjointtilde}{\widetilde{\paramjoint}}
\newcommand{\paramspacejoint}{\DomTheta}

\newcommand{\paramjointhat}{\widehat{\paramjoint}}

\newcommand{\paramtransit}[1]{\ensuremath{\paramtrans^{#1}}}
\newcommand{\paramobsit}[1]{\ensuremath{\paramobs^{#1}}}
\newcommand{\paramjointit}[1]{\ensuremath{\theta^{#1}}}

\newcommand{\paramgamma}{\gamma}

\newcommand{\paramgammatilde}{\widetilde{\paramgamma}}

\newcommand{\elltwoballr}[2]{\ensuremath{\mathbb{B}_2\big(#1;#2\big)}}

\newcommand{\PlainQfunSam}{\ensuremath{Q_\numobs}}
\newcommand{\PlainQfunPop}{\ensuremath{\widebar{Q}}}
\newcommand{\qfunsamp}[2]{\PlainQfunSam(#1 \mid #2)}
\newcommand{\qfunpop}[2]{\PlainQfunPop(#1 \mid #2)}
\newcommand{\qfunpopobs}[2]{\PlainQfunPop_1(#1 \mid #2)}

\newcommand{\qfunsampobs}[2]{\ensuremath{Q_{1,\numobs}(#1 \mid #2)}}
\newcommand{\qfunsamptrans}[2]{\ensuremath{Q_{2,\numobs}(#1 \mid #2)}}
\newcommand{\qfun}{\PlainQfunPop}

\newcommand{\qfunk}[1]{\ensuremath{\QBAR^{#1}}}
\newcommand{\qfunn}[1]{\ensuremath{Q_{#1}}}
\newcommand{\qfunnk}[2]{\qfunsampextendnk{#1}{#2}}
\newcommand{\eqfunn}[1]{\ensuremath{\mathbb{E}Q_{#1}}}

\newcommand{\qnorm}[1]{\|#1\|_{\infty}}
\newcommand{\addnorm}[1]{\| #1 \|_{\star}}

\newcommand{\qfunsamptrunc}[2]{\qfunsampextend{#1}{#2}} 
\newcommand{\qfunpoptrunc}[2]{\PlainQfunPop^k(#1 \mid  #2)}
\newcommand{\qfunpoptruncobs}[2]{\PlainQfunPop_1^k(#1 \mid  #2)}
\newcommand{\qfunpoptrunctrans}[2]{\PlainQfunPop_2^k(#1 \mid  #2)}
\newcommand{\qfunsamptruncobs}[2]{\ensuremath{Q_{1,\numobs}^k (#1 \mid #2)}}
\newcommand{\qfunsamptrunctrans}[2]{\ensuremath{Q_{2,\numobs}^k (#1 \mid #2)}}
\newcommand{\qfunsampextend}[2]{Q^k_n(#1 \mid #2)}
\newcommand{\qfunpopextend}[2]{\qfunk{k}(#1 \mid #2)}

\newcommand{\qfunsampextendnk}[2]{Q^{#2}_{#1}}
\newcommand{\qfunpopextendk}{\qfunk{k}}                 

\newcommand{\filterop}[1]{F_{#1}}
\newcommand{\filterkernel}[2]{K_{#1|#2}}

\newcommand{\MBAR}{\ensuremath{\widebar{M}}}

\newcommand{\emoppop}[1]{\MBAR(#1)}

\newcommand{\emoppoptrunc}[1]{\MBAR^k(#1)}
\newcommand{\emoppoptruncobs}[1]{\MBAR^{\paramobs,\kdim}(#1)}
\newcommand{\emoppoptrunctrans}[1]{\MBAR^{\paramtrans,\kdim}(#1)}
\newcommand{\emoppopobs}[1]{\MBAR^{\paramobs}(#1)}
\newcommand{\emoppoptrans}[1]{\MBAR^{\paramtrans}(#1)}

\newcommand{\emopsampn}[2]{M_{#1}(#2)}
\newcommand{\emopsamptruncn}[2]{M^k_{#1}(#2)}
\newcommand{\emopsamp}[1]{M_{\subsize}(#1)}
\newcommand{\emopsampobs}[1]{\ensuremath{M^\paramobs_{\subsize} (#1)}}

\newcommand{\emopsamptruncobs}[1]{\ensuremath{M_{\subsize}^{\paramobs,k}(#1)}}
\newcommand{\emopsamptrunctrans}[1]{\ensuremath{M_{\subsize}^{\paramtrans,k}(#1)}}
\newcommand{\emopsamptrunc}[1]{\ensuremath{M_{\subsize}^{k}(#1)}}

\newcommand{\MFUNSAM}{\ensuremath{M_\numobs}}
\newcommand{\MFUNSAMOBS}[1]{\ensuremath{M_{\numobs}^{\paramobs}}(#1)}
\newcommand{\MFUNSAMTRANS}[1]{\ensuremath{M_{\numobs}^{\paramtrans}}(#1)}
\newcommand{\MFUNSAMTRUNCOBS}[1]{\ensuremath{M_{\numobs}^{\paramobs,k}}(#1)}
\newcommand{\MFUNSAMTRUNCTRANS}[1]{\ensuremath{M_{\numobs}^{\paramtrans,k}}(#1)}

\newcommand{\blockleftind}{i-k}
\newcommand{\blockrightind}{i+k}
\newcommand{\oddblockindeces}[2]{H^{#1}_{#2}}
\newcommand{\evenblockindeces}[2]{R^{#1}_{#2}}

\newcommand{\expcum}{\Phi}
\newcommand{\expcumgrad}[1]{\frac{\partial \Phi}{\partial #1}}
\newcommand{\expcumhess}[2]{\frac{\partial^2 \Phi}{\partial #1 \partial #2 }}
\newcommand{\condcov}[3]{\cov(#1, #2 \mid #3)}

\newcommand{\SNR}{\eta^2}

\newcommand{\canonvec}[1]{\mathrm{e}_1}

\newcommand{\samperror}[1]{e_{#1}}

\newcommand{\factorsnr}{\varphi_1(\eta)}
\newcommand{\factormixing}{\varphi_2(\mixcoefeps)}

\newcommand{\mprob}{\ensuremath{\mathbb{P}}} 
\newcommand{\real}{\ensuremath{\mathbb{R}}}

\newcommand{\MIXCON}{\ensuremath{c_0}}

\newcommand{\defn}{: \, = }

\newcommand{\Ball}{\ensuremath{\mathbb{B}}}

\newcommand{\rad}{\ensuremath{r}}
\newcommand{\radtrans}{\ensuremath{\max_{\paramtrans \in
      \paramspacetrans} \norm{\paramtrans-\trueparamtrans}}}

\newcommand{\LikeSample}{\ensuremath{\ell_\numobs}}
\newcommand{\hprob}{\ensuremath{p}}

\newcommand{\SPECEXPI}[1]{\ensuremath{\Exs_{Z_i \mid
      x_1^\numobs, #1}}}
\newcommand{\QFUNSAM}[2]{\ensuremath{Q_\numobs (#1 \, \mid #2)}}

\newcommand{\DomTheta}{\ensuremath{\Omega}}

\newcommand{\kdim}{\ensuremath{k}}

\newcommand{\QBAR}{\ensuremath{\widebar{Q}}}

\newcommand{\BOUNDFUN}{\ensuremath{\varphi}}

\newcommand{\probpar}{\zeta}
\newcommand{\matprob}{\probpar}

\newcommand{\matprobhat}{\widehat{\matprob}}
\newcommand{\probparhat}{\matprobhat}
\newcommand{\paramobshat}{\widehat{\paramobs}}

\newcommand{\mustar}{\ensuremath{\mu^*}}
\newcommand{\paramtranshat}{\widehat{\paramtrans}}

\newcommand{\myparagraph}[1]{\paragraph{#1:}}

\newcommand{\Xtil}{\ensuremath{\widetilde{X}}}

\newcommand{\HACKG}{\ensuremath{h}}
\newcommand{\constant}{\ensuremath{c}}

\newcommand{\plaincon}{\ensuremath{c}}

\newcommand{\BIGCON}{\ensuremath{C}}

\newcommand{\minimaxrad}{\ensuremath{e_\numobs}}
\newcommand{\processobs}{\ensuremath{\tilde{V}_m}}
\newcommand{\processradobs}{\ensuremath{V_m}}
\newcommand{\funcproc}{f_{\paramjoint}}
\newcommand{\funcprocgamma}[1]{F(#1;\blockXi)}
\newcommand{\blockXi}{\Xtil_{i;2k}}

\newcommand{\lipproc}{\ensuremath{L_m}}
\newcommand{\epsilontwo}{\tilde{\epsilon}}
\newcommand{\lipcont}{\ensuremath{L}}
\newcommand{\processradM}{\ensuremath{M_m}}
\newcommand{\processradN}{\ensuremath{N_m}}
\newcommand{\paramjointgamma}{\tilde{\paramjoint}}

\newenvironment{carlist}
 {\begin{list}{$\bullet$}
 {\setlength{\topsep}{0in} \setlength{\partopsep}{0in}
  \setlength{\parsep}{0in} \setlength{\itemsep}{\parskip}
  \setlength{\leftmargin}{0.07in} \setlength{\rightmargin}{0.08in}
  \setlength{\listparindent}{0in} \setlength{\labelwidth}{0.08in}
  \setlength{\labelsep}{0.1in} \setlength{\itemindent}{0in}}}
 {\end{list}}

\newcommand{\bcar}{\begin{carlist}}
\newcommand{\ecar}{\end{carlist}}

\newcommand{\COVNUM}{\ensuremath{T}}

\newcommand{\BIGSUP}{\sup_{ \substack{\|u\|_2=1 \\ \|v\|_2 = 1}} }

\title{\LARGE \bf Statistical and Computational Guarantees for the
  Baum-Welch Algorithm }

\author{Fanny Yang$^{\star}$, Sivaraman Balakrishnan$^{\dagger}$ and
  Martin J. Wainwright$^{\dagger,\star}$\\
\begin{tabular}{c}
Department of Statistics$^\dagger$, and \\ Department of Electrical
Engineering and Computer Sciences$^\star$ \\ UC Berkeley, Berkeley, CA
94720
\end{tabular}}

\begin{document}

\maketitle


\begin{abstract}
The Hidden Markov Model (HMM) is one of the mainstays of statistical
modeling of discrete time series, with applications including speech
recognition, computational biology, computer vision and econometrics.
Estimating an HMM from its observation process is often addressed via
the Baum-Welch algorithm, which is known to be susceptible to local
optima.  In this paper, we first give a general characterization of
the basin of attraction associated with any global optimum of the
population likelihood.  By exploiting this characterization, we
provide non-asymptotic finite sample guarantees on the Baum-Welch
updates, guaranteeing geometric convergence to a small ball of radius
on the order of the minimax rate around a global optimum.  As a
concrete example, we prove a linear rate of convergence for a hidden
Markov mixture of two isotropic Gaussians given a suitable mean
separation and an initialization within a ball of large radius around
(one of) the true parameters.  To our knowledge, these are the first
rigorous local convergence guarantees to global optima for the Baum-Welch algorithm in
a setting where the likelihood function is nonconvex.  We complement
our theoretical results with thorough numerical simulations studying
the convergence of the Baum-Welch algorithm and illustrating the
accuracy of our predictions.
\end{abstract}


\section{Introduction}

Hidden Markov models (HMMs) are one of the most widely applied
statistical models of the last 50 years, with major success stories in
computational biology~\cite{durbinbook}, signal processing and speech
recognition~\cite{rabinerbook}, control theory~\cite{controlbook}, and
econometrics~\cite{econbook} among other disciplines.  At a high
level, a hidden Markov model is a Markov process split into an
observable component and an unobserved or latent component.  From a
statistical standpoint, the use of latent states makes the HMM generic
enough to model a variety of complex real-world time series, while
the Markovian structure enables relatively simple computational
procedures.

In applications of HMMs, an important problem is to estimate the state
transition probabilities and the parameterized output densities based
on samples of the observable component.  From classical theory, it is
known that under suitable regularity conditions, the maximum
likelihood estimate (MLE) in an HMM has good statistical
properties~\cite{bickel1998}.
However,
given the potentially nonconvex nature of the likelihood surface,
computing the global maximum that defines the MLE is not a
straightforward task.  In fact, the HMM estimation problem in full
generality is known to be computationally intractable under
cryptographic assumptions~\cite{cryptohmm}.  In practice, however, the
Baum-Welch algorithm~\cite{Baum70} is frequently applied and leads to
good results.  It can be understood as the specialization of the EM
algorithm \cite{Dempster77} to the maximum likelihood estimation
problem associated with the HMM.  Despite its wide use in many
applications, the Baum-Welch algorithm can get trapped in local optima
of the likelihood function.  Understanding when this undesirable
behavior occurs---or does not occur---has remained an open question
for several decades.

A more recent line of work~\cite{mossel2006, Siddiqi10,Hsu12} has
focused on developing tractable estimators for HMMs, via approaches
that are distinct from the Baum-Welch algorithm.  Nonetheless, it
has been observed that the practical performance of such methods can
be significantly improved by running the Baum-Welch algorithm using
their estimators as the initial point; see, for instance, the detailed
empirical study in Kontorovich et al.~\cite{Kontorovich13}. This
curious phenomenon has been observed in other
contexts~\cite{Chaganty13}, but has not been explained to date.
Obtaining a theoretical characterization of when and why the
Baum-Welch algorithm behaves well is the main objective of this paper.


\subsection{Related work and our contributions}

Our work builds upon a framework for analysis of EM, as previously
introduced by a subset of the current authors~\cite{BalWaiYu14}; see
also the follow-up work to regularized EM algorithms~\cite{YiCar15, WangLiu14}.
All of this past work applies to models based on i.i.d. samples, and
as we show in this paper, there are a number of non-trivial steps
required to derive analogous theory for the dependent variables that
arise for HMMs.  Before doing so, let us put the results of this paper
in context relative to older and more classical work on Baum-Welch and
related algorithms.

\begin{figure}[h!]
\begin{center}
\begin{tabular}{cc}
\widgraph{.45\textwidth}{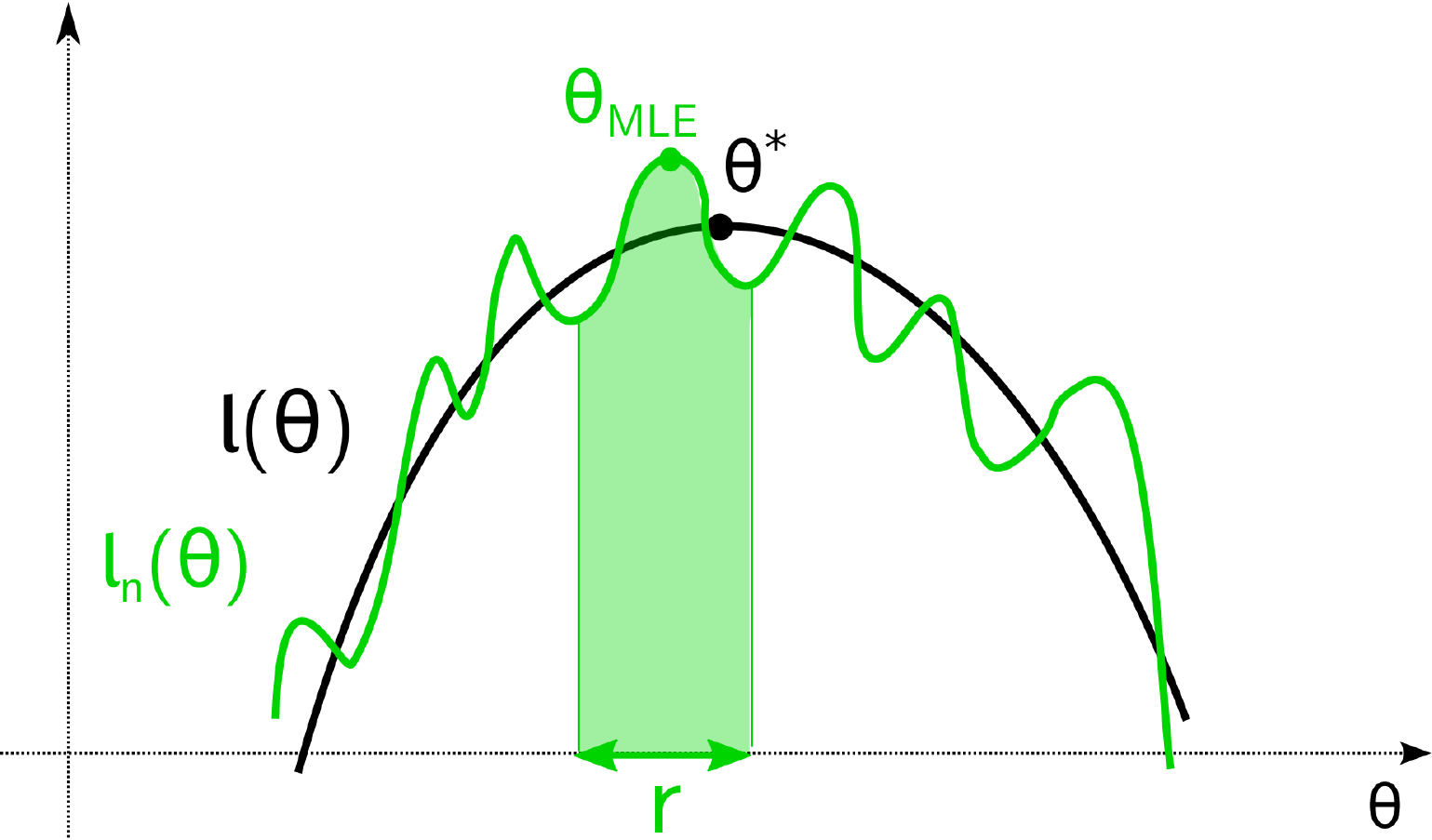} &
\widgraph{.45\textwidth}{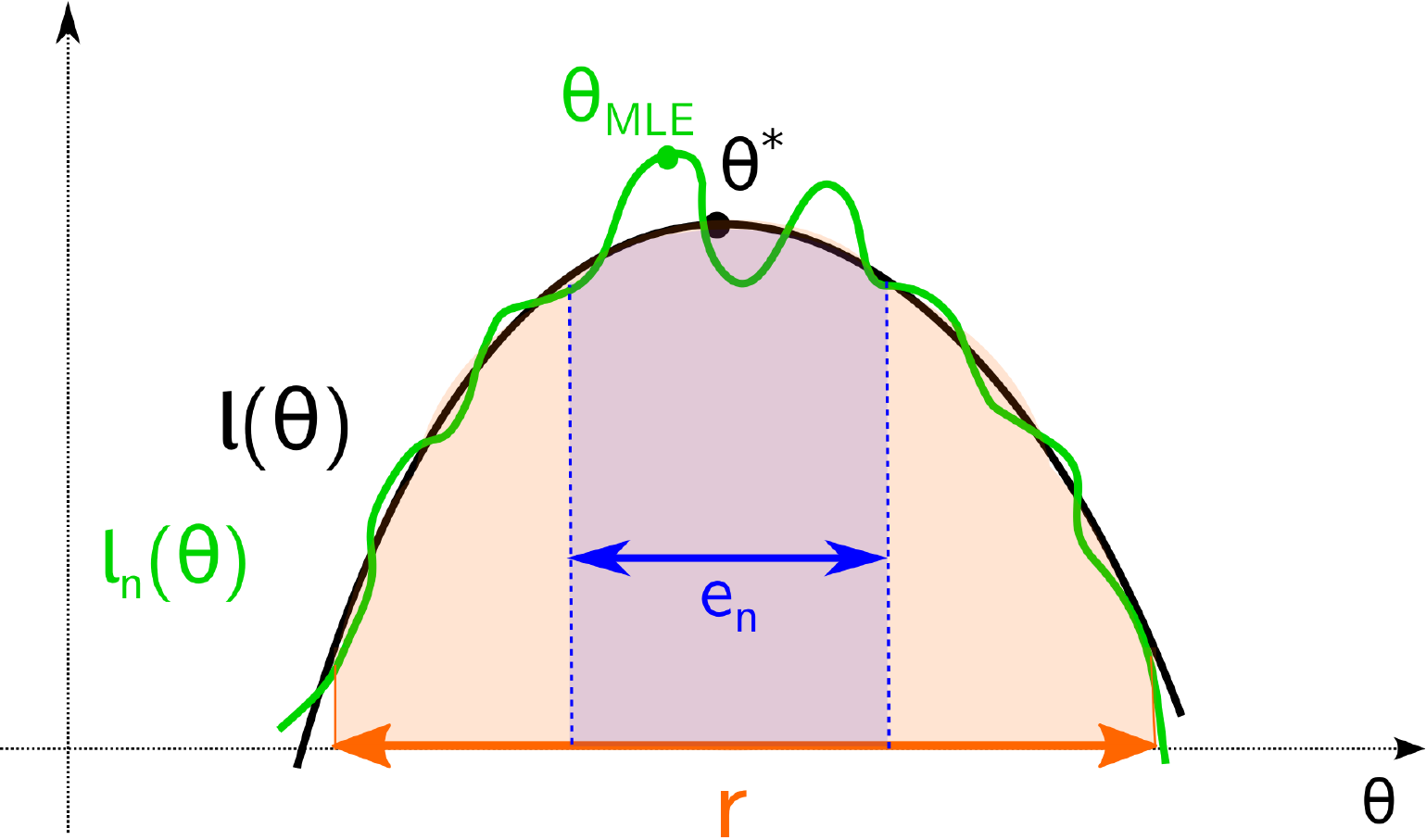} \\
(a) & (b)
\end{tabular}
\end{center}
\caption{(a) A poorly behaved sample likelihood, for which there are
  many local optima at varying distances from the MLE.  It would
  require an initialization extremely close to the MLE in order to
  ensure that the Baum-Welch algorithm would not be trapped at a
  sub-optimal fixed point.  (b) A well-behaved sample likelihood, for
  which all local optima lie within an $\minimaxrad$-ball of the MLE,
  as well as the true parameter $\thetastar$.  In this case, the
  Baum-Welch algorithm, when initialized within a ball of large radius
  $r$, will converge to the ball of much smaller radius $\minimaxrad$.
  The goal of this paper is to give sufficient conditions for when the
  sample likelihood exhibits this favorable structure.}
\label{fig:InMMBall}
\end{figure}

Under mild regularity conditions, it is well-known that the maximum
likelihood estimate (MLE) for an HMM is a consistent and
asymptotically normal estimator; for instance, see Bickel et
al.~\cite{bickel1998}, as well as the expository
works~\cite{Moulines_HMM, vanHandel_HMM}. 
On the algorithmic level,
the original papers of Baum and co-authors~\cite{Baum70, baum1966}
showed that the Baum-Welch algorithm converges to a stationary point
of the sample likelihood; these results are in the spirit of the
classical convergence analysis of the EM algorithm~\cite{Wu83,
  Dempster77}.  These classical convergence results only provide a
relatively weak guarantee---namely, that if the algorithm is
initialized sufficiently close to the MLE, then it will converge to
it.  However, the classical analysis does not quantify the size of
this neighborhood, and as a critical consequence, it \emph{does not}
rule out the pathological type of behavior illustrated in panel (a) of
Figure~\ref{fig:InMMBall}.  Here the sample likelihood has multiple
optima, both a global optimum corresponding to the MLE as well as many
local optima \emph{far away from the MLE} that are also fixed points
of the Baum-Welch algorithm.  In such a setting, the Baum-Welch
algorithm will only converge to the MLE if it is initialized in an
extremely small neighborhood.

In contrast, the goal of this paper is to give sufficient conditions
under which the sample likelihood has the more favorable structure
shown in panel (b) of Figure~\ref{fig:InMMBall}.  Here, even though
the MLE does not have a large basin of attraction, the sample
likelihood has all of its optima (including the MLE) localized to a
small region around the true parameter $\thetastar$.  Our strategy to
reveal this structure, as in our past work~\cite{BalWaiYu14}, is to
shift perspective: instead of studying convergence of Baum-Welch
updates to the MLE, we study their convergence to an
$\epsilon_\numobs$-ball of the true parameter $\thetastar$, and
moreover, instead of focusing exclusively on the sample likelihood, we
first study the structure of the population likelihood, corresponding
to the idealized limit of infinite data.  Our first main result
(Theorem~\ref{ThmPopContraction}) provides sufficient conditions under
which there is a large ball of radius $r$, over which the population
version of the Baum-Welch updates converge at a geometric rate to
$\thetastar$.  Our second main result
(Theorem~\ref{ThmSampContraction}) uses empirical process theory to
analyze the finite-sample version of the Baum-Welch algorithm,
corresponding to what is actually implemented in practice.  In this
finite sample setting, we guarantee that over the ball of radius $r$,
the Baum-Welch updates will converge to an $\epsilon_\numobs$-ball
with $\epsilon_\numobs \ll r$, and most importantly, this
$\epsilon_\numobs$-ball contains the true parameter $\thetastar$.  As
a side-note, it also contains the MLE, but our theory does \emph{not}
guarantee convergence to the MLE, but rather to a point that is close
to both the MLE and the true parameter $\thetastar$.

These latter two results are abstract, applicable to a broad class of
HMMs. We then specialize them to the case of a hidden Markov mixture
consisting of two isotropic components, with means separated by a
constant distance, and obtain concrete guarantees for this model.  It
is worth comparing these results to past work in the i.i.d. setting,
for which the problem of Gaussian mixture estimation under various
separation assumptions has been extensively studied
(e.g.,~\cite{dasgupta,vempala,belkin,moitra}). The constant distance
separation required in our work is much weaker than the separation
assumptions imposed in papers that focus on correctly labeling samples
in a mixture model.  Our separation condition is related to, but in
general incomparable with the non-degeneracy requirements in other
work~\cite{Hsu12, hsumog, moitra}.

Finally, let us discuss the various challenges that arise in studying
the dependent data setting of hidden Markov models, and highlight some
important differences with the
i.i.d. setting~\cite{BalWaiYu14,YiCar15}.  In the non-i.i.d. setting,
arguments passing from the population-based to sample-based updates
are significantly more delicate. First of all, it is not even obvious
that the population version of the $Q$-function---a central object in
the Baum-Welch updates---even exists. From a technical standpoint,
various gradient smoothness conditions are much more difficult to
establish, since the gradient of the likelihood no longer decomposes
over the samples as in the i.i.d. setting. In particular, each term in
the gradient of the likelihood is a function of all
observations. Finally, in order to establish the finite-sample
behavior of the Baum-Welch algorithm, we can no longer appeal to
standard i.i.d.  concentration and empirical process techniques.  Nor
do we pursue the approach of some past work on HMM estimation
(e.g.,~\cite{Hsu12}), in which it is assumed that there are multiple
independent samples of the HMM.\footnote{The rough argument here is that it
  is possible to reduce an i.i.d. sampling model by cutting the
  original sample into many pieces, but this is not an algorithm that
  one would implement in practice.} Instead, we directly analyze the
Baum-Welch algorithm that practioners actually use---namely, one that
applies to a single sample of an $\numobs$-length HMM.  In order to
make the argument rigorous, we need to make use of more sophisticated
techniques for proving concentration for dependent data~\cite{Yu94,
  NobDem93}.

The remainder of this paper is organized as follows.  In
Section~\ref{SecBackground}, we introduce basic background on hidden
Markov models and the Baum-Welch algorithm.
Section~\ref{sec:main_results} is devoted to the statement of our main
results in the general setting, whereas Section~\ref{sec:normal_HMM}
contains the more concrete consequences for the Gaussian output HMM.
The main parts of our proofs are given in Section~\ref{SecProofs},
with the more technical details deferred to the appendices.


\section{Background and problem set-up}
\label{SecBackground}

In this section, we introduce some standard background on hidden
Markov models and the Baum-Welch algorithm.

\subsection{Standard HMM notation and assumptions}
\label{SecWithSpecialCase}

We begin by defining a discrete-time hidden Markov model with hidden
states taking values in a discrete space.  Letting $\ZN$ denote the
integers, suppose that the observed random variables $\{X_i\}_{i \in
  \ZN}$ take values in $\real^\usedim$, and the latent random
variables $\{Z_i\}_{i \in \ZN}$ take values in the discrete space
$[\nstates] \defn \{1, \ldots, \nstates\}$.  The Markov structure is
imposed on the sequence of latent variables.  In particular, if the
variable $Z_1$ has some initial distribution $\pi_1$, then the joint
probability of a particular sequence $(z_1, \ldots, z_\numobs)$ is
given by
\begin{align}
\hprob(z_1, \ldots, z_\numobs; \paramtrans) & = \pi_1(z_1;
\paramtrans) \prod_{i=1}^{\numobs} \hprob(z_i \mid z_{i-1};
\paramtrans),
\end{align}
where the vector $\paramtrans$ is a particular parameterization of the
initial distribution and Markov chain transition probabilities.  We
restrict our attention to the homogeneous case, meaning that the
transition probabilities for step $(t-1) \rightarrow t$ are
independent of the index $t$.  Consequently, if we define the
transition matrix $\Tmat \in \real^{\nstates \times \nstates}$ with
entries
\begin{align*}
\Tmat(j,k ; \paramtrans) \defn \hprob(z_2 = k \mid z_1 = j; \paramtrans),
\end{align*}
then the marginal distribution $\pi_i$ of $Z_i$ can be described by
the matrix vector equation
\begin{align*}
\pi^T_i & = \pi_1^T \Tmat^{i-1},
\end{align*}
where $\pi_i$ and $\pi_1$ denote vectors belonging to the
$\nstates$-dimensional probability simplex.

We assume throughout that the Markov chain is aperiodic and recurrent,
whence it has a unique stationary distribution $\pistat$, defined by
the eigenvector equation $\pistat^T = \pistat^T \Tmat$.  To be clear,
both $\pistat$ and the matrix $\Tmat$ depend on $\paramtrans$, but we
omit this dependence so as to simplify notation.  We assume throughout
that the Markov chain begins in its stationary state, so that $\pi_1 =
\pistat$, and moreover, that it is reversible, meaning that
\begin{align}
\label{ass:reversible}
\pistat(j) \Tmat(j,k) & = \pistat(k) \Tmat(k,j) \qquad 
\end{align}
for all pairs $j, k \in [\nstates]$. 

A key quantity in our analysis is the mixing rate of the Markov chain.
In particular, we assume the existence of \emph{mixing constant}
$\mixcoefeps \in (0,1]$ such that
\begin{equation}
\label{ass:mixing}
\mixcoefeps \leq \frac{\transprob{z_i}{z_{i-1}}}{\pistat(z_i)} \leq
\mixcoefeps^{-1} 
\end{equation}
for all $(z_i, z_{i-1}) \in [\nstates] \times [\nstates]$.  This
condition implies that the dependence on the initial distribution
decays geometrically. More precisely, some simple algebra shows that
\begin{align}
\label{ass:mixingdef}
\sup_{\pi_1} \tvnorm{\pi_1^T \Tmat^t}{\pi_1^T} & \leq \MIXCON
\mixcoef^t \qquad \mbox{for all $t = 1, 2, \ldots$},
\end{align}
where $\mixcoef = 1 - \mixcoefeps$ denotes the \emph{mixing rate} of
the process, and $\MIXCON$ is a universal constant.  Note that as
$\mixcoefeps \rightarrow 1^-$, the Markov chain has behavior
approaching that of an i.i.d. sequence, whereas as $\mixcoefeps
\rightarrow 0^+$, its behavior becomes increasingly ``sticky''.


\begin{figure}[h]
\centering 
\includegraphics[scale= 0.5]{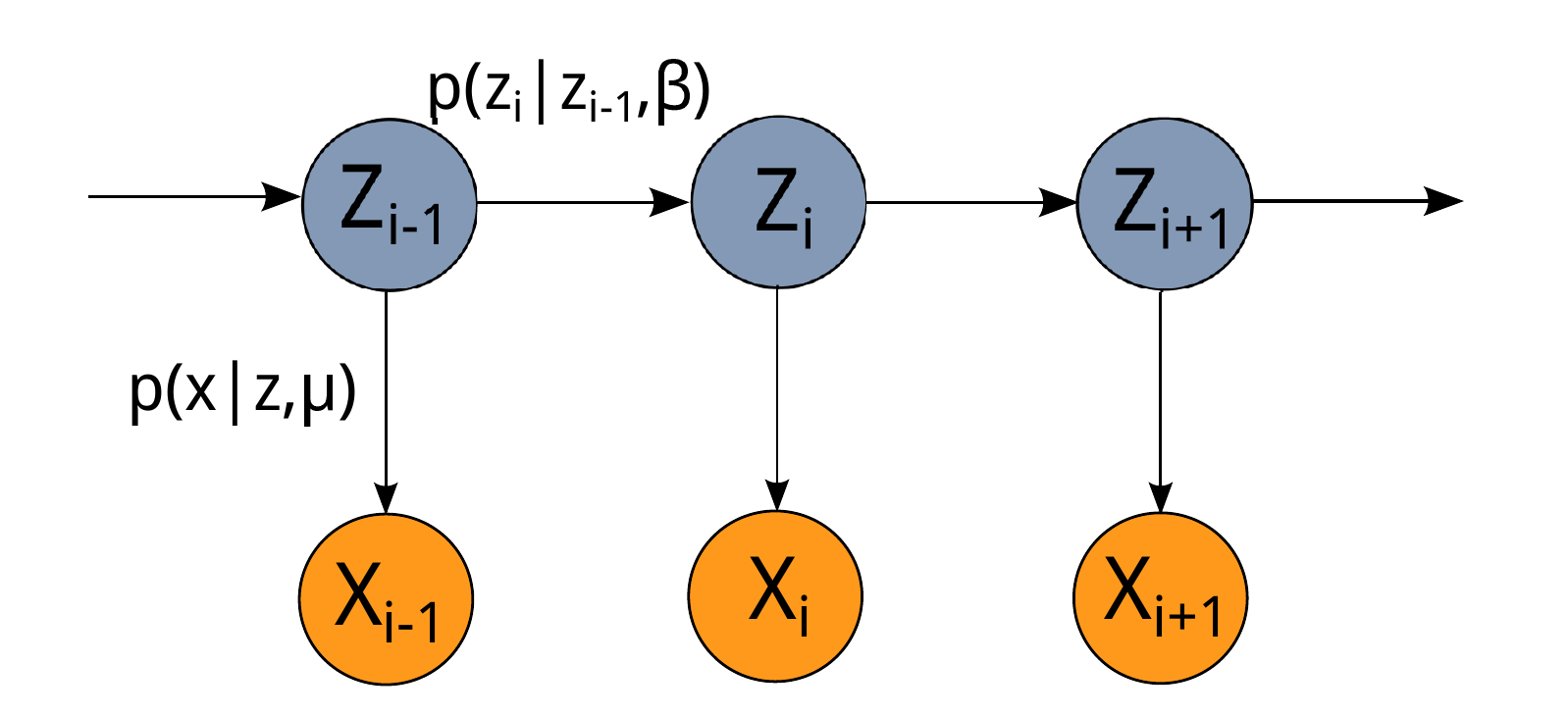}
\caption{The hidden Markov model as a graphical model. The blue circles
  indicate observed variables $Z_i$, whereas the orange circles indicate
  latent variables $X_i$.}
\label{fig:hmm_graphmodel}
\end{figure}

Associated with each latent variable $Z_i$ is an observation $X_i \in
\real^\usedim$.  We use $\obsprob{x_i}{z_i}$ to denote the density of
$X_i$ given that $Z_i = z_i$, an object that we assume to be
 parameterized by a vector $\paramobs$.  Introducing the shorthand $x_1^\numobs = (x_1,
\ldots, x_\numobs)$ and $z_1^\numobs = (z_1, \ldots, z_\numobs)$, the
joint probability of the sequence $(x_1^\numobs, z_1^\numobs)$ (also
known as the complete likelihood) can be
written in the form
\begin{align}
\label{EqnJoint}
\hprob(z_1^\numobs, x_1^\numobs; \paramjoint) & = \pi_1(z_1) 
  \prod_{i=2}^{\numobs} \hprob(z_i \mid z_{i-1}; \paramtrans)  
\prod_{i=1}^\numobs \obsprob{x_i}{z_i},
\end{align}
where the pair $\paramjoint \defn (\paramtrans, \paramobs)$
parameterizes the transition and observation functions.
Similarly we define the likelihood 
\begin{equation*}
p(x_1^\numobs; \paramjoint) = \sum_{z_1^n} \hprob(z_1^\numobs,
x_1^\numobs; \paramjoint).
\end{equation*}
We define a form of complete likelihood~\footnote{ We have defined a
  complete likelihood that involves an additional hidden variable
  $z_0$ that does not have any associated observation $x_0$.  This
  choice turns out to be convenient, but does preserve the usual
  relationship $\sum_{z_0^\numobs} \hprob(z_0^\numobs, x_1^\numobs;
  \paramjoint) = p(x_1^\numobs; \paramjoint)$ between the ordinary and
  complete likelihoods in EM problems.}
\begin{align}
\label{EqnJointTwo}
\hprob(z_0^\numobs, x_1^\numobs; \paramjoint) & = \pi_0(z_0)
\prod_{i=1}^{\numobs} \hprob(z_i \mid z_{i-1}; \paramtrans)
\prod_{i=1}^\numobs \obsprob{x_i}{z_i},
\end{align}
where $\pi_0 = \pistat$.

\textit{A simple example:} A special case helps to illustrate these
definitions.  In particular, suppose that we have a Markov chain with
$\nstates = 2$ states.  Consider a matrix of transition probabilities
$\Tmat \in \real^{2 \times 2}$ of the form
\begin{align}
\Tmat & = \frac{1}{e^{\paramtrans} + e^{-\paramtrans}}
\; \begin{bmatrix} e^\paramtrans & e^{-\paramtrans}
  \\ e^{-\paramtrans} & e^{\paramtrans}
\end{bmatrix} \; = \; \begin{bmatrix} \matprob & 1- \matprob \\
1- \matprob & \matprob
\end{bmatrix}, \label{EqnExampleTrans}
\end{align}
where $\matprob \defn \frac{e^\paramtrans}{e^{\paramtrans} +
  e^{-\paramtrans}}$.  By construction, this Markov chain is recurrent
and aperiodic with the unique stationary distribution $\pistat
= \begin{bmatrix} \frac{1}{2} & \frac{1}{2}
\end{bmatrix}^T$.  Moreover, by calculating the eigenvalues of the
transition matrix, we find that the mixing
condition~\eqref{ass:mixingdef} holds with $\mixcoef := |2 \matprob -
1| = |\tanh(\paramtrans)|$.

Suppose moreover that the observed variables in $\mathbb{R}^d$ are conditionally
Gaussian, say with
\begin{align}
\obsprob{x_t}{z_t} & = \begin{cases} \frac{1}{(2 \pi \sigma^2)^{d/2}}
  \exp \big \{ - \frac{1}{2 \sigma^2} \| x - \paramobs\|_2^2 \big \} &
  \mbox{if $z_t = 1$} \\
\frac{1}{(2 \pi \sigma^2)^{d/2}} \exp \big \{ - \frac{1}{2 \sigma^2}
\|x + \paramobs\|_2^2 \big \} & \mbox{if $z_t = 2$.}
\end{cases}
\label{EqnExampleObs}
\end{align}
With this choice, the marginal distribution of each $X_t$ is a
two-state Gaussian mixture with mean vectors $\paramobs$ and
$-\paramobs$, and covariance matrices $\sigma^2 I_\usedim$.  We
provide specific consequences of our general theory for this special
case in the sequel.


\subsection{Baum-Welch updates for HMMs}

We now describe the Baum-Welch updates for a general discrete-state
hidden Markov model.  As a special case of the EM algorithm, the
Baum-Welch algorithm is guaranteed to ascend on the
likelihood function of the hidden Markov model.  It does so
indirectly, by first computing a lower bound on the likelihood
(E-step) and then maximizing this lower bound (M-step).

For a given integer $\numobs \geq 1$, suppose that we observe a
sequence $x_1^\numobs = (x_1, \ldots, x_\numobs)$ drawn from the
marginal distribution over $X_1^\numobs$ defined by the
model~\eqref{EqnJoint}. The rescaled log likelihood of the sample path
$x_1^\numobs$ is given by
\begin{align*}
\LikeSample(\paramjoint) & = \frac{1}{\numobs} \log
\Big(\sum_{z_0^\numobs} \hprob(z_0^\numobs, x_1^\numobs; \paramjoint)
\Big)
\end{align*}
The EM likelihood is based on lower bounding the likelihood via
Jensen's inequality.  For any choice of parameter
$\paramjoint'$ and positive integers $i\leq j$ and $a<b$, let
$\EEzcondx{Z_i^{j}}{x_a^b}{\paramjoint'}$  denote the
expectation under the conditional distribution $p(Z_i^j \mid
x_a^b; \paramjoint')$.  With this notation, the concavity of the
logarithm and Jensen's inequality implies that for any choice of
$\paramjoint'$, we have the lower bound
\begin{align}
\LikeSample(\paramjoint) &\geq \underbrace{\frac{1}{\numobs}
  \EEzcondx{Z_0^{\numobs}}{x_1^\numobs}{\paramjoint'} \big[ \log p(Z_0^\numobs,x_1^\numobs;
    \paramjoint) \big]}_{\QFUNSAM{\paramjoint}{\paramjoint'}} +
\underbrace{\frac{1}{\numobs} \EEzcondx{Z_0^{\numobs}}{x_1^\numobs}{\paramjoint'} \big[ -\log
    p(Z_0^\numobs \mid x_1^\numobs;
    \paramjoint')]}_{H_\numobs(\paramjoint')}.
\end{align}
For a given choice of $\paramjoint'$, the E-step corresponds to the
computation of the function \mbox{$\paramjoint \mapsto
  \QFUNSAM{\paramjoint}{\paramjoint'}$.}  The $M$-step is defined by
the EM operator $\MFUNSAM: \DomTheta \mapsto \DomTheta$
\begin{align}
\label{EqnEMOperator}
\MFUNSAM(\paramjoint') & = \arg \max_{\theta \in \DomTheta}
\QFUNSAM{\paramjoint}{\paramjoint'},
\end{align}
where $\DomTheta$ is the set of feasible parameter vectors.  Overall,
given an initial vector $\paramjointit{0} = (\paramtransit{0},
\paramobsit{0})$, the EM algorithm generates a sequence
$\{\paramjointit{t}\}_{t=0}^\infty$ according to the recursion
$\paramjointit{t+1} = \MFUNSAM(\paramjointit{t})$.

This description can be made more concrete for an HMM, in which case
the $Q$-function takes the form
\begin{multline}
 \label{EqnDefnQfunSam} 
\qfunsamp{\paramjointone}{\paramjointtwo} = \frac{1}{\numobs}
\EEzcondx{Z_0}{x_1^\numobs}{\paramjoint'} \big[\log
  \pi_0(Z_0;  \paramtransone)\big] + \frac{1}{\numobs}
\sum_{i=1}^\numobs \EEzcondx{Z_{i-1}, Z_i}{x_1^\numobs}{\paramjoint'}
\big[ \log p(Z_i \mid Z_{i-1}; \paramtransone) \big] \\
 + \frac{1}{\numobs} \sum_{i=1}^\numobs \SPECEXPI{\paramjoint'} \big[
   \log p(x_i \mid Z_i; \paramobsone) \big],
\end{multline}
where the dependence of $\pi_0$ on $\paramtrans$ comes from the
assumption that $\pi_0 = \pistat$.  Note that the $Q$-function can be
decomposed as the sum of a term which is solely dependent on
$\paramobs$, and another one which only depends on
$\paramtrans$---that is
\begin{align}
\label{eq:QDecomposed}
\qfunsamp{\paramjointone}{\paramjointtwo} =
\qfunsampobs{\paramobs}{\paramjointtwo} +
\qfunsamptrans{\paramtrans}{\paramjointtwo}
\end{align}
where $\qfunsampobs{\paramobs}{\paramjointtwo} = \frac{1}{\numobs}
\sum_{i=1}^\numobs \SPECEXPI{\paramjoint'} \big[ \log p(x_i \mid
  Z_i,\paramobsone) \big]$, and
$\qfunsamptrans{\paramtrans}{\paramjointtwo}$ collects the remaining
terms.  In order to compute the expectations defining this function
(E-step), we need to determine the marginal distributions over the
singletons $Z_i$ and pairs $(Z_i, Z_{i+1})$ under the joint
distribution $p(Z_0^\numobs \mid x_1^\numobs; \paramjoint')$.  These
marginals can be obtained efficiently using a recursive
message-passing algorithm, known either as the forward-backward or
sum-product algorithm~\cite{Frank01,WaiJor08}.

In the $M$-step, the decomposition~\eqref{eq:QDecomposed} suggests that
the maximization over the two components
$(\paramtrans, \paramobs)$ can also be decoupled.
Accordingly, with a slight abuse of notation, we often write
\begin{align*}
 \quad \MFUNSAMOBS{\paramjoint'} = \arg \max_{\paramobs \in
  \DomTheta_\paramobs} \qfunsampobs{\paramobs}{\paramjoint'} , \quad
  \mbox{and } \quad
  \MFUNSAMTRANS{\paramjoint'} = \arg \max_{\paramtrans \in
  \DomTheta_\paramtrans} \qfunsamptrans{\paramtrans}{\paramjoint'}
\end{align*}
for these two decoupled maximization steps, where 
$\DomTheta_{\paramtrans}$ and $\DomTheta_{\paramobs}$ denote
the feasible set of transition and observation parameters respectively and
$\DomTheta \defn
\DomTheta_{\paramtrans} \times \DomTheta_{\paramobs}$.


\section{Main results}
\label{sec:main_results}

We now turn to a statement of our main results, along with a
discussion of some of their consequences.  The first step is to
establish the existence of an appropriate population analog of the
$Q$-function.  Although the existence of such an object is a
straightforward consequence of the law of large numbers in the case of
i.i.d. data, it requires some technical effort to establish existence
for the case of dependent data; in particular, we do so using
$\kdim$-truncated version of the full $Q$-function (see
Proposition~\ref{PropExistence}).  This truncated object plays a
central role in the remainder of our analysis.  In particular, we
first analyze a version of the Baum-Welch updates on the expected
$\kdim$-truncated $Q$-function for an extended sequence of
observations $x_{1-k}^{n+k}$, and provide sufficient conditions for
these population-level updates to be contractive (see
Theorem~\ref{ThmPopContraction}).  We then use non-asymptotic forms of
empirical process theory to show that under suitable conditions, the
actual sample-based EM updates---i.e., the updates that are actually
implemented in practice---are also well-behaved in this region with
high probability (see Theorem~\ref{ThmSampContraction}).  In
subsequent analysis to follow in Section~\ref{sec:normal_HMM}, we show
that this initialization radius is suitably large for an HMM with
Gaussian outputs.


\subsection{Existence of population \texorpdfstring{$Q$}{Lg}-function}

In the analysis of Balakrishnan et al.~\cite{BalWaiYu14}, the central
object is the notion of a population $Q$-function---namely, the
function that underlies the EM algorithm in the idealized limit of
infinite data.  In their setting of i.i.d. data, the standard law of
large numbers ensures that as the sample size $\numobs$ increases, the
sample-based $Q$-function approaches its expectation, namely the
function
\begin{align*}
\qfunpop{\paramjointone}{\paramjointtwo} &= \EE \big[
  \qfunsamp{\paramjointone}{\paramjointwo} \big] \; = \;
\EE \big[ \EEzcondx{Z_1}{X_1}{\paramjointwo} \big[\log
  p(X_1, Z_1; \paramjointone) \big] \big].
\end{align*}
Here we use the shorthand $\EE$ for the expectation over all samples
$X$ that are drawn from the joint distribution (in this case $\EE :=
\EExcondparam{X_1^n}{\trueparamjoint}$).

When the samples are dependent, the quantity $\EE \big[
  \qfunsamp{\paramjointone}{\paramjointwo} \big]$ is no longer
independent of $\numobs$, and so an additional step is required.  A
reasonable candidate for a general definition of the population
$Q$-function is given by
\begin{align}
\label{EqnDefnQfunPop}
\qfunpop{\paramjointone}{\paramjointtwo} & \defn \lim_{\numobs
  \rightarrow +\infty} [\EE
  \qfunsamp{\paramjointone}{\paramjointtwo}].
\end{align}
Although it is clear that this definition is sensible in the i.i.d. case, 
it is necessary for dependent sampling schemes to prove that the
limit given in definition~\eqref{EqnDefnQfunPop} actually exists.

In this paper, we do so by considering a suitably truncated version of
the sample-based $Q$-function.  Similar arguments have been used in
past work (e.g.,~\cite{Moulines_HMM, vanHandel_HMM}) to establish
consistency of the MLE; here our focus is instead the behavior of the
Baum-Welch algorithm.  Let us consider a sequence $\{(X_i, Z_i)\}_{i
  = 1-k}^{n+k}$, assumed to be drawn from the stationary distribution
of the overall chain. Recall that $\Exs_{Z_i^j \mid x_a^b,
  \paramjoint}$ denotes expectations taken over the distribution
$p(Z_i^j \mid x_a^b, \paramjoint)$.  Then, for a positive integer
$\kdim$ to be chosen, we define
\begin{multline}
\label{eq:sampletruncq} 
\qfunsampextend{\paramjointone}{\paramjointtwo} = \frac{1}{\numobs}
\Big[\Exs_{Z_0 \mid x_{-\kdim}^{\kdim}, \paramjointtwo} \log p(Z_1
  ; \paramtransone) + \sum_{i=1}^\numobs \Exs_{Z_{i-1}^i \mid
    x_{i-k}^{i+k}, \paramjointtwo} \log p(Z_i \mid Z_{i-1};
  \paramtransone) \\
 + \sum_{i=1}^{\numobs} \Exs_{Z_i \mid x_{i-k}^{i+k}, \paramjointwo}
 \log p(x_i \mid Z_i; \paramobsone) \Big].
\end{multline}
In an analogous fashion to the decomposition in
equation~\eqref{EqnDefnQfunSam}, we can decompose $\qfunnk{n}{k}$ in
the form
\begin{align*}
\qfunsamptrunc{\paramjoint}{\paramjointtwo} =
\qfunsamptruncobs{\paramobs}{\paramjointtwo} +
\qfunsamptrunctrans{\paramtrans}{\paramjointtwo}.
\end{align*}
We associate with this triplet of $Q$-functions the corresponding EM
operators $\emopsamptruncn{\numobs}{\paramjointtwo}$,
$\MFUNSAMTRUNCOBS{\paramjointtwo}$ and
$\MFUNSAMTRUNCTRANS{\paramjointtwo}$ as in Equation~\eqref{EqnEMOperator}. 
Note that as opposed to the
function $\qfunn{n}$ from equation~\eqref{EqnDefnQfunSam}, the
definition of $\qfunnk{n}{k}$ involves variables $Z_i, Z_{i-1}$ that
are not conditioned on the full observation sequence $x_1^\numobs$,
but instead only on a $2 \kdim$ window centered around the index $i$.
By construction, we are guaranteed that the $\kdim$-truncated
population function and its decomposed analogs given by
\begin{align}
\qfunpopextend{\paramjointone}{\paramjointtwo} &:= \lim_{n\to\infty} \EE
\qfunsampextend{\paramjointone}{\paramjointtwo} = \EE
\qfunsamptruncobs{\paramobs}{\paramjointtwo} + \lim_{n\to\infty} \EE
\qfunsamptrunctrans{\paramtrans}{\paramjointtwo} \nonumber \\ & :=
\qfunpoptruncobs{\paramobs}{\paramjointtwo} +
\qfunpoptrunctrans{\paramtrans}{\paramjointtwo} \label{eq:Qpopdecomposed}
\end{align}
are well-defined.  In particular, due to stationarity of the random
sequences $\{ p(z_i \mid X_{i-k}^{i+k} ) \}_{i = 1}^\numobs$ and $\{
p(z_{i-1}^i \mid X_{i-k}^{i+k} ) \}_{i = 1}^\numobs$, the expectation
over $\{(X_i, Z_i)\}_{i=1-k}^{n+k}$ is independent of the sample size
$\numobs$.

Our first result uses the existence of this truncated population
object in order to show that the standard population $Q$-function from
equation~\eqref{EqnDefnQfunPop} is indeed well-defined.  In doing so,
we make use of the sup-norm
\begin{align}
\label{eq:q-norm}
\qnorm{Q_1 - Q_2} & \defn \sup_{\paramjoint,\paramjoint' \in
  \DomTheta} \Big| Q_1(\paramjoint \mid \paramjoint') -
Q_2(\paramjoint \mid \paramjoint') \Big|.
\end{align}
For a radius $r > 0$, define the ball
$\Ball_2(r; \mustar) = \{ \mu \in \real^\usedim \, \mid \, \|\mu -
\mustar\|_2 \leq r \}$.
We require in the following that the observation
densities satisfy the following boundedness condition
\begin{align}
\label{EqnDensityBounded}
\sup_{\paramjoint \in \Ball_2(\rad; \thetastar)} \Exs \Big[ \max_{z_i
    \in [\nstates]} \big|\log \obsprob{X_i}{z_i} \big| \Big] < \infty.
\end{align}
%
\begin{props}
\label{PropExistence}
Under the previously stated assumptions, the population function
$\QBAR$ defined in equation~\eqref{EqnDefnQfunPop} exists.
\end{props}

\noindent The proof of this claim is given in
Appendix~\ref{AppPropExistence}.  It hinges on the following auxiliary
claim, which bounds the difference between $\eqfunn{\numobs}$ and the
$\kdim$-truncated $Q$-function as
\begin{align}
 \label{EqnPopulationTruncateOne}
\qnorm{\eqfunn{\numobs} - \qfunpopextendk} & \leq \frac{c \,
  \nstates^5}{\mixcoefeps^8 \statmin^2} \big (1 - \mixcoefeps \statmin
\big)^\kdim + \frac{1}{n} \log \statmin^{-1},
\end{align}
where $\statmin := \min_{\paramtrans \in \paramspacetrans, j \in
  [\nstates]} \pistat(j \mid \paramtrans)$ is the minimum probability
in the stationary distribution, and $\mixcoefeps$ is the mixing
constant from equation~\eqref{ass:mixing}. Note that the dependencies
on $\mixcoefeps$ and $\statmin$ are not optimized here since it would not
help to illustrate the main ideas more clearly. Since this bound holds for
all $\numobs$, it shows that the population function $\QBAR$ can be
uniformly approximated by $\qfunpopextendk$, with the approximation
error decreasing geometrically as the truncation level $\kdim$ grows.
This fact plays an important role in the analysis to follow.


\subsection{Analysis of updates based on \texorpdfstring{$\QBAR^k$}{Lg}}

Our ultimate goal is to establish a bound on the difference between
the sample-based Baum-Welch estimate and $\trueparamjoint$, in
particular showing contraction of the Baum-Welch update towards the
true parameter.  Our strategy for doing so involves first analyzing
the Baum-Welch iterates at the population level, which is the focus of
this section.

The quantity $\QBAR$ is significant for the EM updates because the
parameter $\trueparamjoint$ satisfies the self-consistency property
$\trueparamjoint = \argmax_{\paramjoint} \QBAR(\paramjoint \mid
\trueparamjoint)$.  In the i.i.d. setting, the function $\QBAR$ can
often be computed in closed form, and hence directly analyzed, as was
done in past work~\cite{BalWaiYu14}.  In the HMM case, this function
$\QBAR$ no longer has a closed form, so an alternative route is
needed.  Here we analyze the population version via the truncated
function $\QBAR^\kdim$ \eqref{eq:Qpopdecomposed} instead, where
$\kdim$ is a given truncation level (to be chosen in the sequel).
Although $\trueparamjoint$ is no longer a fixed point of
$\QBAR^\kdim$, the bound~\eqref{EqnPopulationTruncateOne} combined with
the assumption of strong concavity of $\QBAR^\kdim$ imply an upper bound
on the distance of the maximizers of $\QBAR^\kdim$ and $\QBAR$.

With this setup, we consider an idealized population-level algorithm that, 
based on some initialization $\thetatil^0 \in \paramspacejoint =
\elltwoballr{r}{\trueparamobs}\times \paramspacetrans$, generates the
sequence of iterates 
\begin{align}
\label{EqnIdealized}
\thetatil^{t+1} = \emoppoptrunc{\thetatil^t} := \argmax_{\paramjoint
  \in \paramspacejoint}
\qfunpoptrunc{\paramjoint}{\paramjointtil^{t}}.
\end{align}
Since $\QBAR^k$ is an approximate version of $\QBAR$, the update
operator $\MBAR^k$ should be understood as an approximation to the
idealized population EM operator $\MBAR$ where the maximum is taken
with respect to $\qfun$.  As part (a) of the following theorem shows,
the approximation error is well-controlled under suitable conditions.
We analyze the convergence of the sequence
$\{\thetatil^t\}_{t=0}^\infty$ in terms of the norm $\addnorm{\cdot} :
\Omega_{\paramobs} \times \paramspacetrans \to \RN^+$ given by
\begin{align}
\label{eq:addnorm}
\addnorm{\paramjoint - \trueparamjoint}=
\addnorm{(\paramobs,\paramtrans) - (\trueparamobs, \trueparamtrans)}
\defn \norm{\paramobs-\trueparamobs} +
\norm{\paramtrans-\trueparamtrans}.
\end{align}
Contraction in this norm implies that both parameters
$\paramobs,\paramtrans$ converge linearly to the true parameter.

\paragraph{Conditions on $\QBAR^\kdim$:} Let us 
now introduce the conditions on the truncated function $\QBAR^\kdim$
that underlie our analysis.  For a radius $r > 0$, 
we concentrate on showing conditions in the Cartesian product
\begin{align*}
\paramspacejoint & \defn \Ball_2(r; \mustar) \times \Omega_\beta,
\end{align*}
where $\Omega_\beta$ is the set of allowable HMM transition
parameters.  First, let us say that the function
$\qfunpoptrunc{\cdot}{\paramjoint'}$ is
\emph{$(\lambda_{\paramobs},\lambda_{\paramtrans})$-strongly concave }
in $\paramspacejoint$ if
\begin{subequations}
\label{EqnStrongConcavity}
\begin{align}
\qfunpoptruncobs{\paramobs_1}{\trueparamjoint} -
\qfunpoptruncobs{\paramobs_2}{\trueparamjoint} - \langle \nabla_{\paramobs}
\qfunpoptruncobs{\paramobs_2}{\trueparamjoint}, \paramobs_1 -
\paramobs_2 \rangle &\leq -\frac{\lambda_{\paramobs}}{2}
\norm{\paramobs_1 - \paramobs_2}^2
\\ \mbox{and} \qquad
\qfunpoptrunctrans{\paramtrans_1}{\trueparamjoint} -
\qfunpoptrunctrans{\paramtrans_2}{\trueparamjoint} - \langle \nabla_{\paramtrans}
\qfunpoptrunctrans{\paramtrans_2}{\trueparamjoint}, \paramtrans_1 -
\paramtrans_2 \rangle &\leq -\frac{\lambda_{\paramtrans}}{2}
\norm{\paramtrans_1 - \paramtrans_2}^2
\end{align}
\end{subequations}
for all $(\paramobs_1,\paramtrans_1), (\paramobs_2, \paramtrans_2) \in
\paramspacejoint$. 

Second, we impose \emph{first-order stability} conditions on
the gradients of each component of $\QBAR^\kdim$:
\bcar
\item For each $\paramobs \in \Omega_{\paramobs}, \paramjoint' \in
  \paramspacejoint$, we have
\begin{subequations}
\label{EqnObsFOS}
\begin{align}
\label{EqnObsFOSOne}
\|\nabla_\paramobs
\qfunpoptruncobs{\paramobs}{\paramobs',\paramtrans'} -
\nabla_\paramobs \qfunpoptruncobs{\paramobs}{\trueparamobs,
  \paramtrans'} \|_2 &\leq L_{\paramobs,1} \|\paramobs' -
\trueparamobs\|_2 \\
\label{EqnObsFOSTwo}
\|\nabla_\paramobs
\qfunpoptruncobs{\paramobs}{\paramobs',\paramtrans'} -
\nabla_\paramobs \qfunpoptruncobs{\paramobs}{\paramobs',
  \trueparamtrans} \|_2 & \leq L_{\paramobs,2} \norm{\paramtrans' -
\trueparamtrans},
\end{align}
\end{subequations}
We refer to this condition as $L_\paramobs$-FOS for short.
\item Secondly, for all $\paramtrans \in \paramspacetrans,
  \paramjoint'\in \paramspacejoint$, we require that
\begin{subequations}
\label{EqnTransFOS}
\begin{align}
\label{EqnTransFOSOne}
\|\nabla_\paramtrans
\qfunpoptrunctrans{\paramtrans}{\paramobs',\paramtrans'} -
\nabla_\paramtrans \qfunpoptrunctrans{\paramtrans}{\trueparamobs,
  \paramtrans'} \|_2 &\leq L_{\paramtrans,1} \|\paramobs' -
\trueparamobs\|_2 \\
\label{EqnTransFOSTwo}
\|\nabla_\paramtrans
\qfunpoptrunctrans{\paramtrans}{\paramobs',\paramtrans'} -
\nabla_\paramtrans \qfunpoptrunctrans{\paramtrans}{\paramobs',
  \trueparamtrans} \|_2 & \leq L_{\paramtrans,2} \norm{\paramtrans' -
\trueparamtrans}.
\end{align}
\end{subequations}
We refer to this condition as $L_\paramtrans$-FOS for short. 
\ecar
As we show in Section~\ref{sec:normal_HMM}, these conditions hold for
concrete models.


\paragraph{Convergence guarantee for $\QBAR^k$-updates:}

We are now equipped to state our main convergence guarantee for the
updates.  It involves the quantities
\begin{align}
\label{eq:Ldef}
L \defn \max \{L_{\paramobs_1},L_{\paramobs_2}\} + \max
\{L_{\paramtrans_1}, L_{\paramtrans_2}\}, \quad \lambda \defn
\min\{\lambda_{\paramobs}, \lambda_{\paramtrans}\} \quad \mbox{and }
\quad \kappa \defn \frac{L}{\lambda},
\end{align}
with $\kappa$ generally required to be smaller than one, 
as well as the additive norm $\addnorm{\cdot}$ from
equation~\eqref{eq:addnorm}.

Part (a) of the theorem controls the \emph{approximation error}
induced by using the $\kdim$-truncated function $\QBAR^\kdim$ as
opposed to the exact population function $\QBAR$, whereas part (b)
guarantees a \emph{geometric rate of convergence} in terms of $\kappa$
defined above in equation~\eqref{eq:Ldef}.

\begin{theos}
\label{ThmPopContraction}
\begin{enumerate}[(a)]
\item {\emph{Approximation guarantee:}} Under the mixing
  condition~\eqref{ass:mixingdef}, density boundedness
  condition~\eqref{EqnDensityBounded}, and
  $(\lambda_{\paramobs},\lambda_{\paramtrans})$-strong concavity
  condition~\eqref{EqnStrongConcavity}, there is a universal constant
  $c_0$ such that
\begin{align}
\label{EqnApproxBound}
\addnorm{\MBAR^k(\theta)- \MBAR(\theta)}^2 & \leq \underbrace{c_0
  \frac{\nstates^5 }{\lambda \, \mixcoefeps^8 \statmin^2}
  \big(1-\mixcoefeps \statmin \big)^\kdim}_{=: \BOUNDFUN^2(\kdim)}
\qquad \mbox{for all $\theta \in \DomTheta$,}
\end{align}
where $\nstates$ is the number of states, and $\statmin \defn \min
\limits_{\paramtrans \in \paramspacetrans} \min \limits_{j \in [\nstates]}
\pistat(j;  \paramtrans)$.

\item {\emph{Convergence guarantee:}} Suppose in addition that the
  $(L_{\paramobs},L_{\paramtrans})$-FOS
  conditions~\eqref{EqnObsFOS},\eqref{EqnTransFOS} holds with parameter
  $\kappa \in (0,1)$ as defined in~\eqref{eq:Ldef} for
  $\paramjoint,\paramjoint' \in \paramspacejoint =
  \elltwoballr{r}{\trueparamobs} \times \paramspacetrans$,
  and that the truncation parameter $\kdim$ is sufficiently large to
  ensure that 
\begin{align*}
\BOUNDFUN(\kdim) \leq \big(1 - \kappa \big) \rad - \kappa
\max_{\paramtrans \in \paramspacetrans} \norm{\paramtrans -
  \trueparamtrans}.
\end{align*}
Then given an initialization $\thetatil^0 \in \paramspacejoint$, the
iterates $\{\thetatil^t\}_{t=0}^\infty$ generated by the $\MBAR^k$
operator satisfy the bound
\begin{align}
\label{EqnFinalPopBound}
\addnorm{\thetatil^{t} - \thetastar} & \leq
\kappa^{t} \addnorm{\thetatil^0 - \thetastar} +
\frac{1}{1 - \kappa} \BOUNDFUN(\kdim).
\end{align}
\end{enumerate}
\end{theos}

Note that the subtlety here is that $\thetastar$ is no longer a fixed
point of the operator $\MBAR^k$, due to the error induced by the
$\kdim^{th}$-order truncation.  Nonetheless, under a mixing condition,
as the bounds~\eqref{EqnApproxBound} and~\eqref{EqnFinalPopBound}
show, this approximation error is controlled, and decays exponentially
in $\kdim$.  The proof of the recursive bound~\eqref{EqnFinalPopBound}
is based on first showing that
\begin{align}
\label{EqnAuxiliary}
\addnorm{\MBAR^k(\theta) - \MBAR^k(\thetastar)} & \leq \kappa
\addnorm{\paramjoint - \trueparamjoint}
\end{align}
for any $\paramjoint \in \paramspacejoint$. Inequality
\eqref{EqnAuxiliary} is equivalent to stating that the operator
$\MBAR^k$ is contractive, i.e. that applying $\MBAR^k$ to the pair
$\theta$ and $\thetastar$ always decreases the distance.

Finally, when Theorem~\ref{ThmPopContraction} is applied to a concrete
model, the task is to find the biggest $r$ and $\paramspacetrans$ such
that the conditions in the theorem are satisfied, and we do so for the
Gaussian output HMM in Section~\ref{sec:normal_HMM}.


\subsection{Sample-based results}

We now turn to a result that applies to the sample-based form of the
Baum-Welch algorithm---that is, corresponding to the updates that are
actually applied in practice.  For a tolerance parameter $\delta \in
(0,1)$, we let $\BOUNDFUN_{\subsize}(\delta, \kdim)$ be the smallest
positive scalar such that
\begin{subequations}
\begin{align}
\label{EqnVarPhiBound}
\sup_{\theta \in \Ball_2(\rad; \thetastar)} \mprob \Big[
  \addnorm{\emopsampn{\subsize}{\theta} -
  \emopsamptruncn{\subsize}{\theta}} \geq
  \BOUNDFUN_{\subsize}(\delta, \kdim) \Big] & \leq \delta.
\end{align}
This quantity bounds the approximation error induced by the
$\kdim$-truncation, and is the sample-based analogue of the quantity
$\BOUNDFUN(\kdim)$ appearing in Theorem~\ref{ThmPopContraction}(a).
For each $\delta \in (0,1)$, we let
$\epsilonobs_{\subsize}(\delta,\kdim)$ and
$\epsilontrans_{\subsize}(\delta,\kdim)$ denote the smallest positive
scalars such that
\begin{align}
\label{eq:defepsmk}
\mprob \big[ \norm{\emopsamptruncobs{\paramjoint} -
  \emoppoptruncobs{\paramjoint}} \geq
  \epsilonobs_{\subsize}(\delta, \kdim) \Big] \leq \delta, \quad
\mbox{and} \quad
\mprob \big[ \norm{\emopsamptrunctrans{\paramjoint} -
  \emoppoptrunctrans{\paramjoint}} \geq
  \epsilontrans_{\subsize}(\delta, \kdim) \Big] \leq \delta
\end{align}
for all $\theta \in \Ball_2(\mustar; r) \times \Omega_\paramtrans$.
\end{subequations}
Furthermore we define $\epsilon_{\subsize}(\delta,\kdim) \defn
\epsilonobs_{\subsize}(\delta,\kdim) +
\epsilontrans_{\subsize}(\delta,\kdim)$.  For a given truncation level
$\kdim$, these values give an upper bound on the difference between
the population and sample-based $M$-operators, as induced by having
only a finite number $\numobs$ of samples.

\begin{theos}[Sample Baum-Welch]
\label{ThmSampContraction}
Suppose that the truncated population EM operator $\MBAR^\kdim$
satisfies the local contraction bound~\eqref{EqnAuxiliary} with
parameter $\kappa \in (0,1)$ in $\paramspacejoint$.
For a given sample size $\numobs$, suppose
that $(\kdim, \numobs)$ are sufficiently large to ensure that
\begin{subequations}
\begin{align}
\label{EqnConditionSamplesplit}
\BOUNDFUN_{\subsize}(\delta, \kdim) + \epsilonobs_{\subsize}
\big(\delta, \kdim \big) + \BOUNDFUN(\kdim) & \leq (1 - \kappa) \,
\rad - \kappa \max_{\paramtrans \in \paramspacetrans}\norm{\paramtrans - \trueparamtrans} .
\end{align}
Then given any initialization $\thetahat^0 \in \paramspacejoint$, with probability at least $1 - 2 \delta$, the
Baum-Welch sequence $\{\thetahat^t\}_{t=0}^\infty$ satisfies
the bound
\begin{align}
\label{EqnSampleSplitContraction}
\addnorm{\thetahat^{t}- \thetastar} & \leq \underbrace{\kappa^t
  \addnorm{\thetahat^0 - \thetastar}}_{\mbox{Geometric decay}} +
\underbrace{\frac{1}{1-\kappa} \Big \{ 2 \BOUNDFUN_{\subsize}
  \big(\delta, \kdim \big) + \epsilon_{\subsize}\big(\delta, \kdim
  \big) + \BOUNDFUN(\kdim) \Big \}}_{\mbox{Residual error $\minimaxrad$}}.
\end{align}
\end{subequations}
\end{theos}

The bound~\eqref{EqnSampleSplitContraction} shows that the distance
between $\thetahat^t$ and $\thetastar$ is bounded by two terms: the
first decays geometrically as $t$ increases, and the second term
corresponds to a residual error term that remains independent of $t$.
Thus, by choosing the iteration number $T$ larger than $\frac{\log(2
  \rad/\epsilon)}{\log \kappa}$, we can ensure that the first term is
at most $\epsilon$.  The residual error term can be controlled by
requiring that the sample size $\numobs$ is sufficiently large, and
then choosing the truncation level $\kdim$ appropriately.  We provide
a concrete illustration of this procedure in the following section,
where we analyze the case of Gaussian output HMMs.  In particular, we
can see that the residual error is of the same order as for the MLE
and that the required initialization radius is optimal up to
constants. Let us emphasize here
that $\kdim$ as well as the truncated operators are purely theoretical
objects which were introduced for the analysis.


\section{Concrete results for the Gaussian output HMM}
\label{sec:normal_HMM}

We now return to the concrete example of a Gaussian output HMM, as
first introduced in Section~\ref{SecWithSpecialCase}, and specialize
our general theory to it.  Before doing so, let us make some
preliminary comments about our notation and assumptions.  Recall that
our Gaussian output HMM is based on $\nstates = 2$ hidden states,
using the transition matrix from equation~\eqref{EqnExampleTrans}, and
the Gaussian output densities from equation~\eqref{EqnExampleObs}.
For convenience of analysis, we let the hidden variables $Z_i$ take
values in $\{-1,1\}$. In addition, we require that the mixing
coefficient $\mixcoef = 1- \mixcoefeps$ is bounded away from $1$ in
order to ensure that the mixing condition~\eqref{ass:mixing} is
fulfilled.  We denote the upper bound for $\mixcoef$ as
$\mixcoefbound<1$ so that $\mixcoef \leq \mixcoefbound$ and $
\mixcoefeps \geq 1 - \mixcoefbound$.  The feasible set of the
probability parameter $\probpar$ and its log odds analog $\paramtrans
= \frac{1}{2} \log \big(\frac{\probpar}{1-\probpar} \big)$ are then
given by
\begin{align}
\label{eqn:probset} 
\paramspacep = \left\{\probpar \in \real \mid
\frac{1-\mixcoefbound}{2} \leq \probpar \leq \frac{1+\mixcoefbound}{2}
\right\}, \quad \mbox{and} \quad \paramspacetrans = \Big \{\paramtrans
\in \real \; \mid \; |\paramtrans| < \underbrace{\frac{1}{2} \log
  \big( \frac{1 + \mixcoefbound}{1 - \mixcoefbound}
  \big)}_{\paramtransbound} \Big\}.
\end{align}

\subsection{Explicit form of Baum-Welch updates}

We begin by deriving an explicit form of the Baum-Welch updates for
this model.  
Using this notation, the Baum-Welch updates take the form
\begin{subequations}
\label{eq:normalHMM_EMupdate} 
\begin{align}
\paramitsamp^{t+1} &= \frac{1}{\numobs} \sum_{i =1}^{\numobs} ( 2 p(Z_i = 1
\mid x_1^{\numobs}; \paramjointit{t})-1) x_i, \\
\matprobhat^{t+1} & = \Pi_{\paramspacep} \left( \frac{1}{\numobs} \sum_{i=1}^n
  \sum_{Z_i} p(Z_{i} = Z_{i+1} \mid x_1^{\numobs}; \paramjointit{t})
\right), \text{ and } \\
\paramtranshat^{t+1} &= \frac{1}{2} \log \big(
\frac{\probparhat^{t+1}}{1 - \probparhat^{t+1}} \big),
\end{align}
\end{subequations}
where $\Pi_{\paramspacep}$ denotes the Euclidean projection onto the
set $\paramspacep$. Note that the maximization steps are carried out 
on the decomposed $Q$-functions $\qfunsampobs{\cdot}{\paramjoint^t},
\qfunsamptrans{\cdot}{\paramjoint^t}$. In addition, since we are 
dealing with a one-dimensional quantity $\paramtrans$, the projection
of the unconstrained maximizer onto the interval $\paramspacep$
is equivalent to the constrained maximizer over the feasible set 
$\paramspacep$. This step is in general not valid for higher
dimensional transition parameters. 

\subsection{Population and sample guarantees}

We now use the results from Section~\ref{sec:main_results} to
show that the population and sample-based version of the Baum-Welch
updates are linearly convergent in a ball around $\trueparamjoint$
of fixed radius.  In establishing the population-level
guarantee, the key conditions which need to be fulfilled---and the one
that are the most technically challenging to establish--- are the
$(L_{\paramobs}, L_{\paramtrans})$-FOS conditions~\eqref{EqnObsFOS},
\eqref{EqnTransFOS}. In particular, we want to show that these
conditions hold with Lipschitz constants $L_{\paramobs}, L_{\paramtrans}$ that decrease
exponentially with the separation of the mixtures. As a consequence,
we obtain that for large enough separation $\frac{L}{\lambda} <1$, i.e.
 the EM operator is contractive towards the true parameter. 

Throughout this section, we use $\plaincon_0,
\plaincon_1$ to denote universal constants and $\BIGCON_0, \BIGCON_1$ 
for quantities that do not depend on $(\|\trueparamobs\|_2, \sigma)$, but
may depend on other parameters such as $\statmin$, $\mixcoef$,
$\mixcoefbound$, and so on.  In order to ease notation, our explicit
tracking of parameter dependence is limited to the standard deviation
$\sigma$ and Euclidean norm $\|\trueparamobs\|_2$, which together
determine the signal-to-noise ratio $\SNR \defn
\frac{\|\trueparamobs\|_2^2}{\sigma^2}$ of the mixture model.  We use
the notation

We begin by stating a result for the sequence
$\{\thetatil^t\}_{t=0}^\infty$ obtained by repeatedly applying the
$\kdim$-truncated population-level Baum-Welch update operator
$\MBAR^k$.  Our first corollary establishes that this sequence is
linearly convergent, with a convergence rate $\kappa = \kappa(\eta)$
that is given by
\begin{align}
\label{EqnDefnKappa}
\kappa(\eta) & \defn \frac{\BIGCON_1 \eta^2(\SNR+1) \;
\E^{-\plaincon_2 \eta^2} }{1-\mixcoefbound^2}.
\end{align}
\begin{cors}[Population Baum-Welch]
\label{cor:simultaneous_normal_lipschitz}
Consider a two-state Gaussian output HMM that is mixing
(i.e. satisfies equation~\eqref{ass:mixing}), and with its SNR lower
bounded as $\eta^2 \geq \BIGCON$ for a sufficiently
large constant $\BIGCON$. Given the radius $r =
\frac{\|\trueparamobs\|_2}{4}$, suppose that the truncation parameter
$k$ is sufficiently large to ensure that $\BOUNDFUN(\kdim) \leq (1-
\kappa) r - \kappa \max_{\paramtrans \in
  \paramspacetrans}\norm{\paramtrans - \trueparamtrans}$.  Then for
any initialization \mbox{$\paramjointtilde^0 = (\paramobstilde^0,
  \paramtranstilde^0) \in \elltwoballr{r}{\trueparamobs} \times
  \paramspacetrans$,} the sequence $\{\thetatil^t\}_{t=0}^\infty$
generated by $\MBAR^k$ satisfies the bound
\begin{align}
\label{EqnPopContraction} 
\addnorm{\paramjointtil^t - \trueparamjoint} & \leq \kappa^t \addnorm{\paramjointtilde^0 -
\trueparamjoint} + \frac{1}{1 - \kappa} \BOUNDFUN(\kdim)
\end{align}
for all iterations $t = 1, 2, \ldots$.
\end{cors}

From definition~\eqref{EqnDefnKappa} it follows that as long as the
signal-to-noise ratio $\eta$ is larger than a universal constant,
the convergence rate $\kappa(\eta) < 1$.  The
bound~\eqref{EqnPopContraction} then ensures a type of contraction and
the pre-condition $\BOUNDFUN(\kdim) \leq (1 - \kappa) r - \kappa
\max_{\paramtrans \in \paramspacetrans} \norm{\paramtrans -
  \trueparamtrans}$ can be satisfied by choosing the truncation
parameter $\kdim$ large enough.  If we use a finite truncation
parameter $\kdim$, then the contraction occurs up to the error floor
given by $\BOUNDFUN(\kdim)$, which reflects the bias introduced by
truncating the likelihood to a window of size $\kdim$.  At the
population level (in which the effective sample size is infinite), we
could take the limit $\kdim \rightarrow \infty$ so as to eliminate
this bias.  However, this is no longer possible in the finite sample
setting, in which we must necessarily have $k \ll \numobs$.

\begin{cors}[Sample Baum-Welch iterates]
\label{cor:normalsim_samplesplit}
For a given tolerance $\delta \in (0,1)$,
suppose that the sample size is lower bounded as \mbox{$\numobs \geq
  \BIGCON_1 (\sigma^2 + \|\mustar\|_2^2) \usedim \log^2( \frac{
    \usedim}{\delta})$} for a sufficiently large $\BIGCON_1$.  Then
under the conditions of
Corollary~\ref{cor:simultaneous_normal_lipschitz}, with probability at
least $1 - \delta$, we have
\begin{equation}
\label{EqnGaussSampleBound}
\addnorm{\paramjointtil^t - \trueparamjoint} \leq \kappa^t
\addnorm{\widehat{\paramjoint}^0 - \trueparamjoint} + \BIGCON \frac{
  \norm{\trueparamobs}(\frac{\norm{\trueparamobs}^2}{\sigma^2} +1)
  \log^2n \sqrt{ \frac{ \usedim \log^2 (\numobs/\delta)}{\numobs} }
}{1 - \kappa}.
\end{equation}
\end{cors}
\noindent 

\paragraph{Remarks:} 
As a consequence of the bound~\eqref{EqnGaussSampleBound}, if we are
given a sample size $\numobs \succsim \usedim \log^2 \usedim$, then
taking $T \approx \log \numobs$ iterations is guaranteed to return an
estimate $(\paramobshat^T,\paramtranshat^T)$ with error of the order
$\sqrt{\frac{\usedim \log^6(\numobs)}{\numobs}}$.

In order to interpret this guarantee, note that in the case of
symmetric Gaussian output HMMs as in Section~\ref{sec:normal_HMM},
standard techniques can be used to show that the minimax rate of
estimating $\paramobs^*$ in Euclidean norm scales as
$\sqrt{\frac{\usedim}{\numobs}}$.  If we could compute the MLE in
polynomial time, then its error would also exhibit this scaling.  The
significance of Corollary~\ref{cor:normalsim_samplesplit} is that it
shows that the Baum-Welch update achieves this minimax risk of
estimation up to logarithmic factors.

Moreover, it should be noted that the initialization radius given here
is essentially optimal up to constants. Because of the symmetric
nature of the population log-likelihood, the all zeroes vector is a
stationary point.  Consequently, the maximum Euclidean radius of any
basin of attraction for one of the observation parameters---that is,
either $\trueparamobs$ or $-\trueparamobs$---can at most be $r =
\norm{\trueparamobs}$. Note that our initialization radius only
differs from this maximal radius by only a small constant factor.


\subsection{Simulations}

In this section, we provide the results of simulations that confirm
the accuracy of our theoretical predictions for two-state Gaussian
output HMMs. In all cases, we update the estimates for the mean vector
$\paramobshat^{t+1}$ and transition probability $\hat{\matprob}^{t+1}$
according to equation~\eqref{eq:normalHMM_EMupdate}; for
convenience, we update $\zeta$ as opposed to $\paramtrans$. 
The true parameters are denoted by $\trueparamobs$ and
$\matprob^*$.
\begin{figure}[htbp]
  \begin{center}
\includegraphics[scale = 0.45]{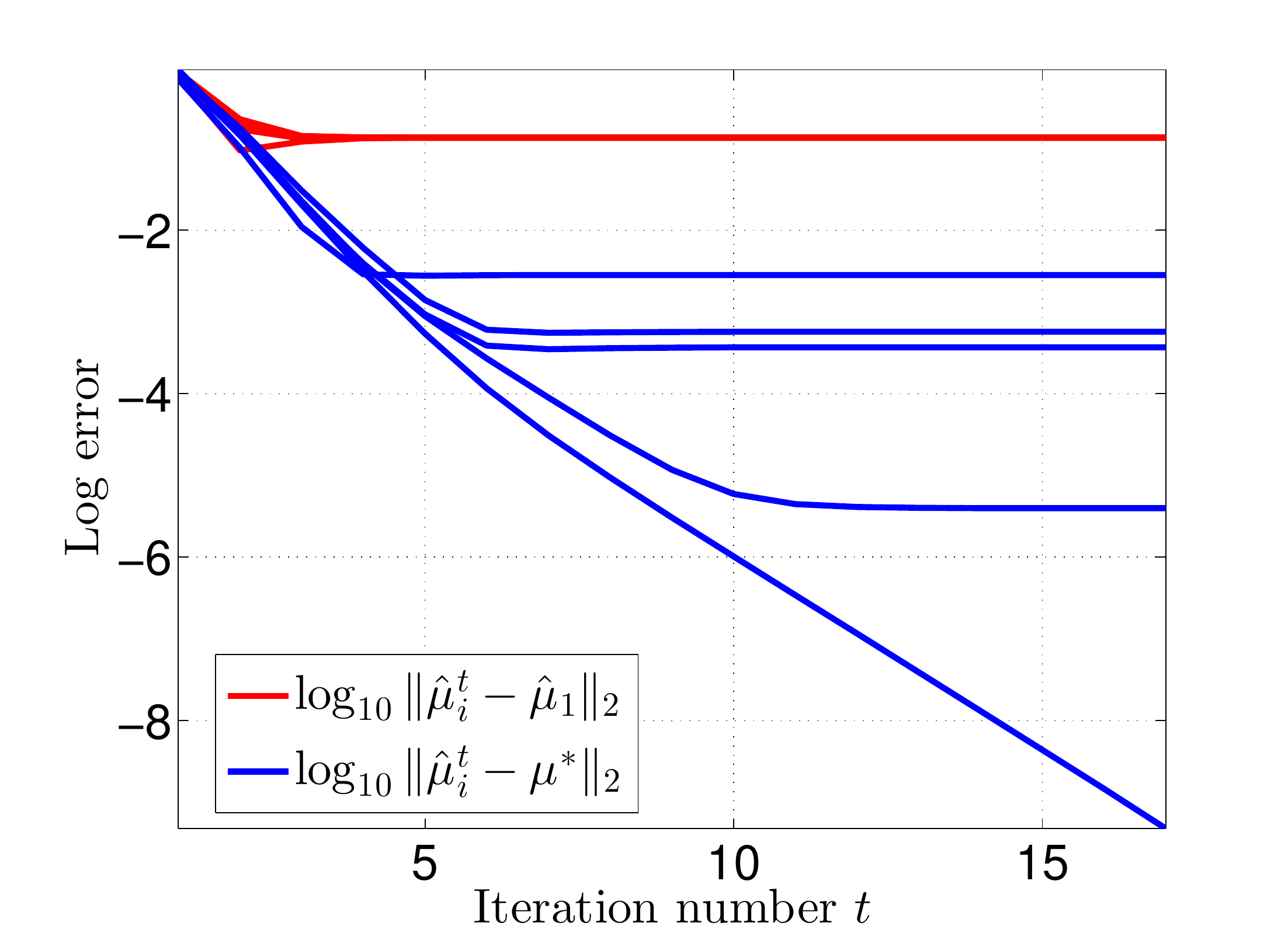}
\end{center}
\caption{Plot of the convergence of the optimization error $\log
  \|\widehat{\paramobs}_i^t - \widehat{\paramobs}_1\|_2$,
  plotted in blue, and the statistical error
  $\log\|\widehat{\paramobs}_i^t - \trueparamobs\|_2$, plotted in red,
  for 5 different initializations.  The parameter settings were $d =
  10$, $n = 1000$, $\mixcoef = 0.6$ and SNR
  $\frac{\|\trueparamobs\|_2}{\sigma} = 1.5$.  See the main text for
  further details.}
\label{fig:it}
\end{figure}

In all simulations, we fix the mixing parameter to $\mixcoef =0.6$,
generate initial vectors $\paramobshat^0$ randomly in a ball of radius
$\rad \defn \frac{\|\trueparamobs\|_2}{4}$ around the true parameter
$\trueparamobs$, and set $\widehat{\zeta}^0 = \frac{1}{2}$.  Finally,
the estimation error of the mean vector $\paramobs$ is computed as
$\log_{10}\|\hat{\mu} - \trueparamobs\|_2$. Since the transition
parameter estimation errors behave similarly to the observation
parameter in simulations, we omit the corresponding figures here.

Figure~\ref{fig:it} depicts the convergence behavior of the Baum-Welch
updates, as assessed in terms of both the optimization and the
statistical error. Here we run the Baum-Welch algorithm for a fixed
sample sequence $X_1^n$ drawn from a model with SNR $\SNR = 1.5$ and
$\zeta = 0.2$, using different random initializations in the ball
around $\trueparamobs$ with radius $\frac{\norm{\trueparamobs}}{4}$.
We denote the final estimate of the $i-$th trial by
$\paramobshat_i$. The curves in blue depict the
\emph{optimization error}---that is, the differences between the
Baum-Welch iterates $\paramobshat_i^t$ using the $i$-th
initialization, and $\paramobshat_1$.  On the other hand, the
red lines represent the \emph{statistical error}---that is, the
distance of the iterates from the true parameter $\trueparamobs$.

For both family of curves, we observe linear convergence in the first
few iterations until an error floor is reached.  The convergence of
the statistical error aligns with the theoretical prediction in upper
bound~\eqref{EqnGaussSampleBound} of
Corollary~\ref{cor:normalsim_samplesplit}. The (minimax-optimal) error
floor in the curve corresponds to the residual error and the
$\minimaxrad$--region in Figure~\ref{fig:InMMBall}.  In addition, the
blue optimization error curves show that for different
initializations, the Baum-Welch algorithm converges to \emph{different
  stationary points} $\paramobshat_i$; however, all of these points
have roughly the same distance from $\trueparamobs$.  This phenomenon
highlights the importance of the change of perspective in our
analysis---that is, focusing on the true parameter as opposed to the
MLE.  Given the presence of all these local optima in a small
neighborhood of $\trueparamobs$, the MLE basin of attraction must
necessarily be much smaller than the initialization radius guaranteed
by our theory.

Figure~\ref{fig:muerror_snr} shows how the convergence rate of the
Baum-Welch algorithm depends on the underlying SNR parameter $\SNR$;
this behavior confirms the predictions given in
Corollary~\ref{cor:normalsim_samplesplit}.  Lines of the same color
represent different random draws of parameters given a fix
SNR. Clearly, the convergence is linear for high SNR, and the rate
decreases with decreasing SNR.  

\begin{figure}[htbp]
\begin{center}
\subfigure{
  \includegraphics[scale=0.45]{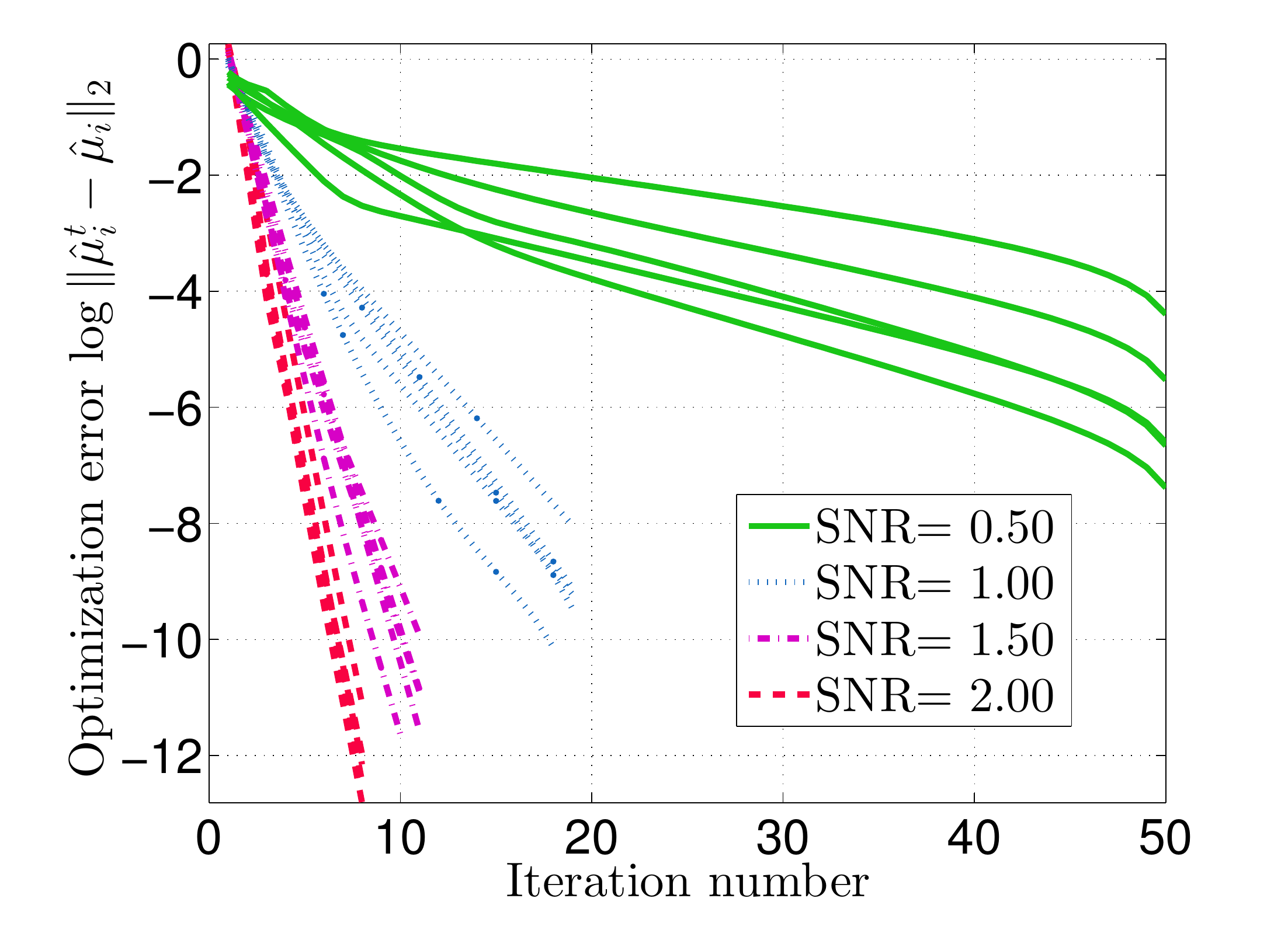}
}
\end{center}
\caption{Plot of convergence behavior for different SNR, where for
  each curve, different parameters were chosen.  The parameter
  settings are $d=10$, $n=1000$ and $\mixcoef = 0.6$.  }
\label{fig:muerror_snr}
\end{figure}


\section{Proofs}
\label{SecProofs}

In this section, we collect the proofs of our main results.  In all
cases, we provide the main bodies of the proofs here, deferring the
more technical details to the appendices.

\subsection{Proof of Theorem~\ref{ThmPopContraction}}

Throughout this proof, we make use of the shorthand $\mixcoefeff = 1 -
\mixcoefeps \statmin$. Also we denote the separate components of the
population EM operators by $\emoppop{\paramjoint} =:
(\emoppopobs{\paramjoint}, \emoppoptrans{\paramjoint})^T$ and their
truncated equivalents by $\emoppoptrunc{\paramjoint} =:
(\emoppoptruncobs{\paramjoint}, \emoppoptrunctrans{\paramjoint})^T$.  We
begin by proving the bound given in part (a).  Since $\qfun =
\lim_{\numobs \rightarrow \infty} \Exs[ Q_\numobs]$, we have
\begin{align*}
\qnorm{\qfun - \qfunpopextendk} & = \qnorm{\lim_{\numobs \rightarrow
    \infty} \Exs[ Q_\numobs] - \qfunpopextendk} \; \leq \;
\frac{C\nstates^5}{\mixcoefeps^8 \statmin^2}\mixcoefeff^k,
\end{align*}
where we have exchanged the supremum and the limit before applying
the bound~\eqref{EqnPopulationTruncateOne}.
The same holds for the separate functions $\qfun_1, \qfun_2$.

Using this bound and the fact that for $\qfun_1$ we have 
$\qfunpopobs{\emoppopobs{\paramjoint}}{\paramjoint} \geq
\qfunpopobs{\emoppoptruncobs{\paramjoint}}{\paramjoint}$, we find that
\begin{equation*}
\qfunpopobs{\emoppopobs{\paramjoint}}{\paramjoint} \geq
\qfunpoptruncobs{\emoppoptruncobs{\paramjoint}}{\paramjoint} - \frac{C
  \nstates^5}{\mixcoefeps^8 \statmin^2} \mixcoefeff^{k} .
\end{equation*}
Since $\emoppoptruncobs{\paramjoint}$ is optimal, the first-order
conditions for optimality imply that 
\begin{align*}
\inprod{\qfunpoptruncobs{\emoppoptruncobs{\paramjoint}}{\paramjoint}}
       {\theta - \emoppoptruncobs{\paramjoint}} & \leq 0 \qquad \mbox{for
         all $\paramjoint \in \elltwoballr{r}{\trueparamjoint}$.}
\end{align*}
By the strict concavity of $\qfunk{k}(\cdot|\paramjoint)$ for all
$\paramjoint$, we obtain
\begin{align*}
  \frac{C \nstates^5}{\mixcoefeps^8 \statmin^2} \mixcoefeff^{k} &\geq
  \qfunpopobs{\emoppopobs{\paramjoint}}{\paramjoint} -
  \qfunpoptruncobs{\emoppopobs{\paramjoint}}{\paramjoint} \\
& \geq \qfunpoptruncobs{\emoppoptruncobs{\paramjoint}}{\paramjoint} -
  \qfunpoptruncobs{\emoppopobs{\paramjoint}}{\paramjoint} - \frac{C
    \nstates^5}{\mixcoefeps^8 \statmin^2} \mixcoefeff^{\kdim} \\
& \geq \frac{\lambda_{\paramobs}}{2}\|\emoppopobs{\paramjoint} -
  \emoppoptruncobs{\paramjoint}\|^2_2 - \frac{C
    \nstates^5}{\mixcoefeps^8 \statmin^2} \mixcoefeff^{\kdim}
\end{align*}
and therefore $\|\emoppopobs{\paramjoint} -
\emoppoptruncobs{\paramjoint}\|^2_2 \leq 4 \frac{C \nstates^5}{\lambda
  \mixcoefeps^8 \statmin^2} \mixcoefeff^{k}$.  In particular, setting
$\paramjoint = \trueparamjoint$ and identifiability,
i.e. $\emoppopobs{\trueparamjoint} = \trueparamjoint$, yields 
\begin{equation*}
  \|\emoppoptruncobs{\trueparamjoint} - \trueparamjoint\|_2^2\leq 4 \frac{C
  \nstates^5}{\lambda_{\paramobs} \mixcoefeps^{6} \statmin^2} \mixcoefeff^{\kdim},
\end{equation*}
and the equivalent bound can be obtained for $\emoppoptrunctrans{\cdot}$
which yields the claim. \\

We now turn to the proof of part (b).  Let us suppose that the recursive
bound~\eqref{EqnAuxiliary} holds, and use it to complete the proof of
this claim.  We first show that if $\paramobstil^{t} \in \Ball_2(\rad;
\trueparamobs)$, then we must have $\paramobstil^{t+1} \in
\Ball_2(\rad; \trueparamobs)$ as well.  Indeed, if $\paramobstil^{t}
\in \Ball_2(\rad; \trueparamobs)$, then we have by triangle inequality
and the $(L_{\paramobs},L_{\paramtrans})$-FOS condition
\begin{align*}
\nonumber \norm{\emoppoptruncobs{\paramjointtil^{t}} - \trueparamobs} & \leq
 \norm{\emoppoptruncobs{\paramjointtil^t} -
  \emoppoptruncobs{\trueparamobs, \paramtranstil^t}} +
\norm{\emoppoptruncobs{\trueparamobs, \paramtranstil^t} -
  \emoppoptruncobs{\trueparamobs, \trueparamtrans}} +\norm{\emoppoptruncobs{\trueparamjoint} - \trueparamobs}\\
& \leq \frac{L_{\paramobs,1}}{\lambda_{\paramobs}}
\norm{\paramobstil^t - \trueparamobs} +
\frac{L_{\paramobs,2}}{\lambda_{\paramobs}} \norm{\paramtranstil^t -
    \trueparamtrans} + \frac{\BOUNDFUN(\kdim)}{2}\\
\nonumber & \leq \kappa (\rad + \radtrans) +
\BOUNDFUN(\kdim) \leq \rad,
\end{align*}
where the final step uses the assumed bound on $\BOUNDFUN$.  
For the joint parameter update we in turn have
\begin{align}
\label{EqnIteration}
\nonumber \addnorm{\emoppoptrunc{\paramjointtil^{t}} - \trueparamjoint} & \leq
\addnorm{\emoppoptrunc{\paramjointtil^t} - \emoppoptrunc{\trueparamjoint}} +
\addnorm{\emoppoptrunc{\trueparamjoint} - \trueparamjoint} \\
&\leq \kappa \addnorm{\paramjointtil^t - \trueparamjoint} + \BOUNDFUN(\kdim).
\end{align}
By repeatedly applying inequality~\eqref{EqnIteration} and summing the
geometric series, the claimed bound~\eqref{EqnFinalPopBound} follows.

It remains to prove the bound~\eqref{EqnAuxiliary}.  Since the vector
$\MBAR^k(\thetastar)$ maximizes the function $\theta \mapsto
\qfunpoptruncobs{\paramjoint}{\trueparamjoint}$, we have the first-order optimality
condition
\begin{align*}
\inprod{\nabla \qfunpoptruncobs{\emoppoptruncobs{\trueparamjoint}}{\thetastar}}{\emoppoptruncobs{\paramjoint} - \emoppoptruncobs{\trueparamjoint}} & \leq 0, \qquad
\mbox{valid for any $\theta$.}
\end{align*}
Similarly, we have
$\inprod{\nabla \qfunpoptruncobs{\emoppoptruncobs{\paramjoint}}{\theta}}{\emoppoptruncobs{\trueparamjoint}
  - \emoppoptruncobs{\paramjoint}} \leq 0$, and adding together these two inequalities
yields
\begin{align*}
0 & \leq \inprod{\nabla \qfunpoptruncobs{\emoppoptruncobs{\trueparamjoint}}{\trueparamjoint} -
  \nabla \qfunpoptruncobs{\emoppoptruncobs{\paramobs}}{\paramjoint}}{\emoppoptruncobs{\trueparamjoint} -
  \emoppoptruncobs{\paramjoint}}
\end{align*}
On the other hand, by the $\lambda$-strong concavity condition, we
have
\begin{align*}
 \lambda_{\paramobs} \|\emoppoptruncobs{\paramjoint} -
 \emoppoptruncobs{\trueparamjoint} \|_2^2 & \leq \inprod{\nabla
   \qfunpoptruncobs{\emoppoptruncobs{\paramjoint}}{\trueparamjoint}) -
   \nabla \qfunpoptruncobs{
     \emoppoptruncobs{\trueparamjoint}}{\thetastar}}{\emoppoptruncobs{\trueparamjoint}
   - \emoppoptruncobs{\paramjoint} }
\end{align*}
Combining these two inequalities yields
\begin{align*}
\lambda_{\paramobs} \|\emoppoptruncobs{\paramjoint} - \emoppoptruncobs{\trueparamjoint} \|_2^2 & \leq
\inprod{\nabla \qfunpoptruncobs{\emoppoptruncobs{\paramjoint}}{\trueparamjoint} - \nabla
  \qfunpoptruncobs{ \emoppoptruncobs{\paramjoint} }{\paramjoint}}{
  \emoppoptruncobs{\trueparamjoint} - \emoppoptruncobs{\paramjoint}} \\
& \leq \big[ L_{\paramobs_1} \norm{\paramobs - \trueparamobs} + L_{\paramobs_2} \norm{\paramtrans-\trueparamtrans}  \,\big] \norm{\emoppoptruncobs{\paramjoint} -
\emoppoptruncobs{\trueparamjoint}},
\end{align*}
and similarly we obtain $\lambda_{\paramtrans} \norm{\emoppoptrunctrans{\paramjoint} - \emoppoptrunctrans{\trueparamjoint}} 
 \leq \big[ L_{\paramtrans_1} \norm{\paramobs - \trueparamobs} + L_{\paramtrans_2} \norm{\paramtrans-\trueparamtrans}  \,\big]$. 
Adding both inequalities yields the claim~\eqref{EqnAuxiliary}.


\subsection{Proof of Theorem~\ref{ThmSampContraction}}

By the triangle inequality, for any step, with probability at least $1
- 2 \subprob$, we have
\begin{align*}
\addnorm{\thetahat^{t+1} - \thetastar} & \leq \addnorm{\emopsamp{\thetahat^t} -
\emopsamptrunc{\thetahat^t}} + \addnorm{\emopsamptrunc{\thetahat^t} -
\emoppoptrunc{\thetahat^t}} + \addnorm{\emoppoptrunc{\thetahat^t} - \trueparamjoint} \\
& \leq \BOUNDFUN_{\subsize}(\subprob, \kdim) +
\epsilon_{\subsize}(\subprob, \kdim) + \kappa \addnorm{\thetahat^t -
\thetastar} + \BOUNDFUN(\kdim).
\end{align*}
In order to see that the iterates do not leave $\elltwoballr{r}{\trueparamobs}$,
observe that
\begin{align}
\norm{\paramobshat^{t+1} - \trueparamobs} & \leq \norm{\emopsampobs{\paramjointhat^t} -
\emopsamptruncobs{\paramjointhat^t}} + \norm{\emopsamptruncobs{\paramjointhat^t} -
\emoppoptruncobs{\thetahat^t}} + \norm{\emoppoptruncobs{\thetahat^t} - \trueparamobs} \label{eq:hanaicecream}\\
& \leq \BOUNDFUN_{\subsize}(\subprob, \kdim) +
\epsilonobs_{\subsize}(\subprob, \kdim) + \kappa (\norm{\paramobshat^t - \trueparamobs} +
\radtrans) + \frac{\BOUNDFUN(\kdim)}{2}.\nonumber
\end{align}
Consequently, as long as $\|\paramobshat^t - \trueparamobs\|_2 \leq \rad$,
we also have $\|\paramobshat^{t+1} - \trueparamobs\|_2 \leq \rad$ whenever
\begin{align*}
\BOUNDFUN_{\subsize}(\subprob, \kdim) + \BOUNDFUN(\kdim) +
\epsilonobs_{\subsize}(\subprob, \kdim) & \leq (1 - \kappa) \, \rad -
\kappa \radtrans.
\end{align*}
Combining inequality~\eqref{eq:hanaicecream} with the equivalent bound for $\paramtrans$,
we obtain
\begin{align*}
\addnorm{\paramjointhat^t- \trueparamjoint} \leq \kappa \addnorm{\paramjointhat^{t-1} - \trueparamjoint} + 2 \BOUNDFUN_{\subsize} (\subprob, \kdim) + \epsilon_{\subsize}(\subprob, \kdim) + \BOUNDFUN(k)
\end{align*}
Summing the geometric series yields
the bound~\eqref{EqnSampleSplitContraction}.


\subsection{Proof of Corollary~\ref{cor:simultaneous_normal_lipschitz}}
\label{SecProofCorOne}

The boundedness condition (Assumption~\eqref{EqnDensityBounded}) is
easy to check since for $X \sim N(\trueparamobs, \sigma^2)$, the
quantity $\sup \limits_{\paramobs \in \Ball_2(r; \trueparamobs)} \EE
\big[ \max\{ \|X - \mu\|_2, \|X + \mu\|_2 \} \big]$ is finite for any
choice of radius $\rad < \infty$.

By Theorem~\ref{ThmPopContraction}, the $k$-truncated population EM
iterates satisfy the bound
\begin{align}
\label{EqnPopulationSeparate}
\addnorm{\paramjointtil^t -\trueparamjoint} & \leq \kappa^t \addnorm{
\paramjointtil^0 - \trueparamjoint} + \frac{1}{1-\kappa}
\BOUNDFUN(\kdim),
\end{align}
if the strong concavity~\eqref{EqnStrongConcavity} and FOS
conditions~\eqref{EqnObsFOS}, \eqref{EqnTransFOS} hold with suitable
parameters.

In the remainder of proof---and the bulk of the technical work--- we
show that:
\begin{itemize}
\item strong concavity holds with $\lambda_\paramobs = 1$ and
  $\lambda_\paramtrans = 1 - \mixcoefbound^2$;
\item the FOS conditions hold with
\begin{align*}
L_{\paramobs,1} = \constant \,  \: 
(\SNR+1)\factormixing \SNR \E^{-c \SNR}, & \quad \mbox{and} \quad
L_{\paramobs,2} = \plaincon \sqrt{\norm{\trueparamobs}^2+\sigma^2}\factormixing
\SNR \E^{-c\SNR} \\
L_{\paramtrans,1} = \plaincon \frac{1-b}{1+b} \factormixing \SNR
\E^{-c \SNR} & \quad \mbox{and} \quad L_{\paramtrans,2} = c
\sqrt{\norm{\trueparamobs}^2+\sigma^2}\factormixing \SNR \E^{-c\SNR},
\end{align*}
\end{itemize}
where $\factormixing \defn
\left(\frac{1}{\log(1/(1-\mixcoefeps))} +
\frac{1}{\mixcoefeps}\right)$.
Substuting these choices into the
bound~\eqref{EqnPopulationSeparate} and performing some algebra yields
the claim.

\subsubsection{Establishing strong concavity}
\label{sec:concavity}
We first show concavity of $\qfunpoptruncobs{\cdot}{\paramjoint'}$ and
$\qfunpoptrunctrans{\cdot}{\paramjoint'}$ separately. For strong
concavity of $\qfunpoptruncobs{\cdot}{\paramjoint'}$, observe that
\begin{align*}
\qfunpoptruncobs{\paramobsone}{\paramjointtwo} = - \frac{1}{2}\EE
\left[ p(z_0=1 \mid X_{-k}^k;\paramjointtwo) \|X_0 -
  \paramobsone\|_2^2 + ( 1- p(z_0=1|X_{-k}^k;\paramjointtwo))\|X_0 +
  \paramobsone\|_2^2 + c ]\right],
\end{align*}
where $c$ is a quantity independent of $\paramobsone$.  By inspection,
this function is strongly concave in $\paramobsone$ with parameter
$\lambda_\paramobs = 1$.

On the other hand, we have
\begin{align*}
\qfunpoptrunctrans{\paramtrans}{\paramjoint'} =
\EExcondparam{X_{-k}^k}{\trueparamjoint} \sum_{z_0, z_1} p(z_0, z_1
\mid X_{-k}^k; \paramjoint') \log \left( \frac{\E^{\paramtrans
    z_0z_1}}{\E^{\paramtrans} + \E^{-\paramtrans}} \right).
\end{align*}
This function has second derivative $\frac{\partial^2}{\partial
  \paramtrans^2} \qfunpoptrunctrans{\paramtrans}{\paramjoint'} = -4
\frac{\E^{-2\paramtrans}}{(\E^{-2\paramtrans}+1)^2}$.  As a function
of $\paramtrans \in \paramspacetrans$, this second derivative is
maximized at $\paramtrans = \frac{1}{2} \log
\big(\frac{1+b}{1-b}\big)$.  Consequently, the function
$\qfunpoptrunctrans{\cdot}{\paramjoint'}$ is strongly concave with
parameter $\lambda_\paramtrans = 1-\mixcoefbound^2$.


\subsubsection{Separate FOS conditions}

We now turn to proving that the FOS conditions in
equations~\eqref{EqnObsFOS} and~\eqref{EqnTransFOS} hold.  A key
ingredient in our proof is the fact that the conditional density
$p(z_{-\kdim}^\kdim \mid x_{-\kdim}^\kdim; \paramobs, \paramtrans)$
belongs to the exponential family with parameters $\paramtrans \in
\real$, and $\paramgamma_i \defn \frac{\inprod{\mu}{x_i}}{\sigma^2}
\in \real$ for $i=-k, \ldots, \kdim$.  (See the book~\cite{WaiJor08}
for more details on exponential families.)  In particular, we have
\begin{align}
\label{eq:specexpo}
\underbrace{p(z_{-k}^k \mid x_{-k}^k,\paramobs, \paramtrans)}_{ \defn
  p(z_{-k}^k; \paramgamma, \paramtrans)} & = \exp \left\{
\sum_{\ell=-k}^k \paramgamma_{\ell} z_{\ell} + \paramtrans \sum_{\ell =
  -k}^{k-1} z_{\ell}z_{\ell+1} - \expcum(\paramgamma, \paramtrans)
\right\},
\end{align} 
where the function $h$ absorbs various coupling terms.
Note that this exponential family is a specific case of the following
exponential family distribution
\begin{align}
\label{eq:generalexpo}
\underbrace{\generalp (z_{-k}^k \mid x_{-k}^k,\paramobs, \paramtrans)}_{ \defn
  \generalp(z_{-k}^k; \paramgamma, \paramtrans)} & = \exp \left\{
\sum_{\ell=-k}^k \paramgamma_{\ell} z_{\ell} +  \sum_{\ell =
  -k}^{k-1} \paramtrans_{\ell} z_{\ell}z_{\ell+1} - \expcum(\paramgamma, \paramtrans)
\right\}.
\end{align} 
The distribution in~\eqref{eq:specexpo} corresponds to 
\eqref{eq:generalexpo} 
with $\paramtrans_\ell = \paramtrans$ for all $\ell$
 and the
so-called partition function $\expcum$ is given by
\begin{equation*}
\expcum (\gamma, \paramtrans) = \log \sum_z \exp \left\{ \sum_{\ell=-k}^k
\paramgamma_{\ell} z_{\ell} + 
\sum_{\ell=-k}^{k-1} \paramtrans_{\ell} z_{\ell}
z_{\ell+1} \right\}.
\end{equation*}
The reason to view our distribution as a special case of the more
general one in~\eqref{eq:generalexpo} becomes clear when 
we consider the equivalence of expectations and the 
derivatives of the cumulant function
\begin{align}
 \label{EqnExpoCumGradient}
\left. \expcumgrad{\paramgamma_{\ell}}\right|_{\paramjoint'} &=
\EEzcondx{Z_{-k}^k}{x_{-k}^k}{\paramjoint'} Z_{\ell} \:\: \text{ and } \:\:
\left. \expcumgrad{\paramtrans_0} \right|_{\paramjoint'}=
\EEzcondx{Z_{-k}^k}{x_{-k}^k}{\paramjoint'} Z_0 Z_1,
\end{align}
 where we recall that $\EEzcondx{Z_{-k}^k}{x_{-k}^k}{\paramjoint'}$ is
 the expectation with respect to the distribution $\generalp(Z_{-k}^k
 \mid x_{-k}^k; \paramobs', \paramtrans')$ with $\paramtrans_\ell =
 \paramtrans'$.  Note that in the following any value $\paramjoint'$
 for $\generalp$ is taken to be on the manifold on which
 $\paramtrans_\ell = \paramtrans'$ for some $\paramtrans'$ since this
 is the manifold the algorithm works on. Also, as before, $\EE$
 denotes the expectation over the joint distribution of all samples
 $X_\ell$ drawn according to $p(\cdot; \paramjoint^*)$, in this case
 $X_{-k}^k$.

Similarly to equations~\eqref{EqnExpoCumGradient}, 
the covariances of the sufficient statistics correspond to
the second derivatives of the cumulant function
\begin{subequations}
\label{EqnExpoCumHessian}
\begin{align}
\left. \frac{\partial^2 \Phi}{\partial \paramtrans_\ell \partial
  \paramtrans_0} \right|_{\paramjoint} & = \condcov{ Z_{0}Z_1}{
  Z_{\ell} Z_{\ell+1} }{X_{-k}^k,\paramjoint} \\
 \left. \expcumhess{\paramgamma_{\ell}}{\paramgamma_0}
 \right|_{\paramjoint} &= \condcov{Z_0}{Z_{\ell}}{X_{-k}^k, \paramjoint} \\
\left. \expcumhess{\paramtrans_\ell}{\paramgamma_0} \right|_{\paramjoint} &
= \condcov{Z_0}{Z_{\ell}Z_{\ell+1}}{X_{-k}^k, \paramjoint}.
\end{align}
\end{subequations}
In the following, we adopt the shorthand
\begin{align*}
\condcov{Z_\ell}{Z_{\ell+1}}{\paramgamma', \paramtrans'} &=
\condcov{Z_\ell}{Z_{\ell+1}}{X_{-k}^k,\paramjoint'} \\ & =
\EEzcondx{Z_\ell^{\ell+1}}{X_{-k}^k}{\paramjoint'} (Z_\ell -
\EEzcondx{Z_{\ell}^{\ell+1}}{X_{-k}^k}{\paramjoint'}
Z_\ell)(Z_{\ell+1} -
\EEzcondx{Z_\ell^{\ell+1}}{X_{-k}^k}{\paramjoint'} Z_{\ell+1})
\end{align*}
where the dependence on $\paramtrans$ is occasionally omitted so as to
simplify notation.


\subsubsection{Proof of inequality~\eqref{EqnObsFOSOne}}

By an application of the mean value theorem, we have
\begin{align*}
\| \nabla_\paramobs
\qfunpoptruncobs{\paramobs}{\paramobs',\paramtrans'} -
\nabla_\paramobs
\qfunpoptruncobs{\paramobs}{\trueparamobs,\paramtrans'}\| & \leq
\underbrace{\left\| \EE \sum_{\ell = -k}^k \left. \frac{\partial^2
    \Phi}{\partial \paramgamma_\ell \partial
    \paramgamma_0}\right|_{\paramjoint = \paramjointtilde}
  (\paramgamma'_{\ell} - \paramgamma^*_{\ell}) X_0 \right\|}_{T_1}
\end{align*}
where $\paramjointtilde = \paramjointtwo +
t(\trueparamjoint-\paramjointtwo)$ for some $t\in (0,1)$.  Since
second derivatives yield covariances (see
equation~\eqref{EqnExpoCumHessian}), we can write
\begin{align*}
T_1 & = \left\| \sum_{\ell = -k}^k\left. \EE X_0 \EE \left[
  \condcov{Z_0}{Z_{\ell}}{\widetilde{\paramgamma}} \frac{\langle
    \paramobstwo- \trueparamobs,X_{\ell}\rangle}{\sigma^2}
  \right|X_0\right]\right\|_2,
\end{align*}
so that it suffices to control the expected conditional covariance.
By the Cauchy-Schwarz inequality and the fact that
$\cov(XY) \leq \sqrt{\var X} \sqrt{\var Y}$ and $\var (Z_0 \mid X)
\leq 1$, we obtain the following bound on the  expected conditional covariance
by using Lemma~\ref{lem:varreduce} (see Appendix~\ref{AppExponential})
\begin{subequations}
\begin{align}
\EE \left[ \big| \condcov{Z_0}{Z_{\ell}}{\paramgammatilde} \mid X_0
  \big| \right ] &\leq \sqrt{ \EE \left[ \var(Z_0 \mid
    \paramgammatilde)\mid X_0\right] } \sqrt{\EE \left[ \var(Z_{\ell}
    \mid \paramgammatilde) \mid X_0 \right]} \nonumber \\
%
%
\label{eq:covvar}
& \leq \sqrt{\var(Z_0 \mid \paramgammatilde_0)}.
\end{align}
Furthermore, by Lemma~\ref{lem:varsnr} and~\ref{lem:expfammixing0}
(see Appendix~\ref{AppExponential}), we have
\begin{align}
\label{eq:covmixing}
|\condcov{Z_0}{Z_{\ell}}{\paramgammatilde}| &\leq 2
\mixcoef^{\ell}, \quad \mbox{and} \quad
\left\| \EE (\var(Z_0|\paramgammatilde_0))^{1/2} X_0X_0^T
\right\|_{op} \leq C \E^{-c\SNR}.
\end{align}
\end{subequations}


From the definition of the operator norm, we have
\begin{align}
 \|\EE \cov (Z_0,Z_{\ell} \mid \paramgammatilde) X_0 X_{\ell}^T
 \big\|_{op} &= \BIGSUP \EE \cov(Z_0, Z_{\ell} \mid \paramgammatilde)
 \inprod{X_0}{v} \, \inprod{ X_{\ell}}{u} \nonumber \\
& \leq \sup_{\|v\|_2 = 1} \EE |\cov(Z_0, Z_{\ell} \mid
 \paramgammatilde)| \inprod{X_0}{v}^2 + \sup_{\|u\|_2 = 1} \EE
 |\cov(Z_0, Z_{\ell} \mid \paramgammatilde)| \inprod{X_\ell}{u}^2
 \nonumber \\
&= \| \EE X_0 X_0^T \EE \big[|\cov(Z_0, Z_{\ell} \mid
   \paramgammatilde) \mid X_0\big]\|_{op} + \|\EE X_\ell X_\ell^T \EE
 \big[ \cov(Z_0,Z_{\ell} \mid \paramgammatilde)\mid X_{\ell}
   \big]\|_{op} \nonumber \\
& \stackrel{(i)}{\leq} 2\min\{\mixcoef^{|\ell|} \|\EE X_0
 X_0^T\|_{op}, \|\EE \var (Z_0 \mid \paramgammatilde_0)^{1/2} X_0
 X_0^T \|_{op}\} \nonumber \\
\label{eq:covx}
& \stackrel{(ii)}{\leq} 2 \min\{(\norm{\trueparamobs}^2 + \sigma^2)
\mixcoef^{|\ell|} , C' \E^{-c\SNR}\},
\end{align}
where inequality (i) makes use of inequalities~\eqref{eq:covvar}
and~\eqref{eq:covmixing}, and step (ii) makes use of the second
inequality in line~\eqref{eq:covmixing}.

Combining inequalities~\eqref{eq:covvar} and~\eqref{eq:covx}, we find
that
\begin{align*}
T_1 & \leq \frac{\norm{ \paramobs' - \trueparamobs}}{\sigma^2}
\sum_{\ell =-k}^k \|\EE \cov (Z_0,Z_{\ell} \mid \paramgammatilde) X_0
X_{\ell}^T \big\|_{op}\\ 
& \leq 2\frac{\norm{ \paramobs' - \trueparamobs}}{\sigma^2}
\sum_{\ell=-k}^k \min\{ (\norm{\trueparamobs}^2 + \sigma^2)
\mixcoef^{|\ell|} , C \E^{-c\SNR} \}\\ 
&\leq 4 (\SNR+1) \big(m C \E^{-c\SNR} +
\frac{\mixcoef^m}{1-\mixcoef}\big) \norm{ \paramobs' - \trueparamobs}.
\end{align*}
The last inequality follows from the proof of Corollary 1 in the
paper~\cite{BalWaiYu14} if $\SNR >C$ for some universal constant $C$.
By choosing $m = \frac{c \SNR}{\log (1/\mixcoef)}$ so as to optimize
the tradeoff, we have shown that
\begin{align*}
\| \nabla_\paramobs
\qfunpoptruncobs{\paramobs}{\paramobs',\paramtrans'} -
\nabla_\paramobs
\qfunpoptruncobs{\paramobs}{\trueparamobs,\paramtrans'}\| & \leq
L_{\paramobs,1} \|\paramobstwo - \trueparamobs\|_2,
\end{align*}
where $L_{\paramobs,1} = \constant \, \: \factorsnr \factormixing \SNR
\E^{-c \SNR}$ as claimed.


\subsubsection{Proof of inequality~\eqref{EqnObsFOSTwo}}

The same argument via the mean value theorem guarantees that
\begin{align*}
\| \frac{\partial}{\partial \paramtrans}
\qfunpoptrunctrans{\paramtrans}{\paramobs',\paramtrans'} -
\frac{\partial}{\partial \paramtrans}
\qfunpoptrunctrans{\paramtrans}{\paramobs',\trueparamtrans} \| & \leq
\left \| \EE \sum_{\ell=-k}^k \left. \frac{\partial^2 \Phi}{\partial
  \paramtrans_\ell \partial \paramgamma_0}\right|_{\paramjoint =
  \paramjointtilde} (\paramtranstwo - \trueparamtrans) X_0 \right \|_2.
\end{align*}
In order to bound this quantity, we again use the
equivalence~\eqref{EqnExpoCumHessian} and bound the expected
conditional covariance. 
Furthermore, Lemma~\ref{lem:varsnr} and \ref{lem:expfammixing0} yield
\begin{align}
\label{EqnHanaBirthday}
 \condcov{Z_0}{Z_{\ell}Z_{\ell+1}}{\paramgammatilde}
 \stackrel{(i)}{\leq} 2\mixcoef^{\ell} \quad \mbox{and} \quad \left\|
 \EE \var(Z_0\mid \paramgammatilde_0) X_0X_0^T \right\|_{op}
 \stackrel{(ii)}{\leq} \plaincon \E^{-c\eta^2}.
\end{align}
Here inequality (ii) follows by combining
inequality~\eqref{eq:varTwoSNR} from Lemma~\ref{lem:varsnr} with the
fact that $\var (Z_0 \mid \paramgammatilde_0) \leq 1$.

\begin{align*}
\| \EE X_0 \condcov{Z_0}{Z_{\ell}Z_{\ell+1}}{\paramgammatilde} \|_2 &=
\sup_{\|u\|_2=1} \EE \langle X_0,u\rangle
\condcov{Z_0}{Z_{\ell}Z_{\ell+1}}{\paramgammatilde}\\
& \leq \sup_{\norm{u}=1} \EE |\langle X_0,u\rangle| \EE
\big[|\condcov{Z_0}{Z_{\ell}Z_{\ell+1}}{\paramgammatilde}| \mid
  X_0\big] \\ 
& \stackrel{(iii)}{\leq} \sup_{\norm{u}=1} \EE |\langle X_0,u\rangle|
\min\{\mixcoef^{|\ell|}, (\var(Z_0 \mid \paramgammatilde_0))^{1/2} \}
\\
& \stackrel{(iv)}{\leq} \min\{ \sup_{\norm{u}=1} \sqrt{ \EE \langle
  X_0,u\rangle^2} \mixcoef^{|\ell|}, \sup_{\norm{u}=1}\sqrt{\EE
  \langle X_0,u\rangle^2 \var(Z_0\mid \paramgammatilde_0))} \} \\
& \stackrel{(v)}{\leq} \min\{ \mixcoef^{|\ell|} \sqrt{ \| \EE
  X_0X_0^T\|_{op}}, \sqrt{\|\EE \var(Z_0 \mid \paramgammatilde_0)
  X_0X_0^T\|_{op}}\}\\ 
& \stackrel{(vi)}{\leq} \min\{ \mixcoef^{|\ell|} \sqrt{
  \norm{\trueparamobs}^2+\sigma^2 }, C \E^{-c\SNR}\}
\end{align*}
where step (iii) uses inequality~\eqref{EqnHanaBirthday}; step (iv)
follows from the Cauchy-Schwarz inequality; step (v) follows from the
definition of the operator norm; and step (vi) uses
inequality~\eqref{EqnHanaBirthday} again.

Putting together the pieces, we find that
\begin{align*}
\left\| \EE \sum_{\ell =-k}^k \frac{\partial^2 \Phi}{\partial
  \paramtrans_\ell \partial \paramgamma_0} X_0 \right\|_2
|\paramtranstwo - \trueparamtrans| & \leq \sum_{\ell=-k}^k \left\| \EE
X_0 \EE [\condcov{Z_0}{Z_{\ell} Z_{\ell+1}}{\paramgammatilde}\mid X_0]
\right\|_2 |\paramtranstwo-\trueparamtrans|\\
& \leq 4 \sqrt{\norm{\trueparamobs}^2 + \sigma^2} \left(\plaincon \, m
\, \E^{-c\SNR}+ \frac{\mixcoef^m}{1-\mixcoef}\right) |\paramtrans' -
\trueparamtrans|.
\end{align*}
Setting $m = \frac{c \SNR }{\log (1/\mixcoef)}$ to balance the
tradeoff between the two terms, we find that
inequality~\eqref{EqnObsFOSTwo} holds with $L_{\paramobs,2} =
\plaincon \factormixing \sqrt{\norm{\trueparamobs}+\sigma^2} \SNR
\E^{-c\SNR}$, as claimed.


\subsubsection{Proof of inequality~\eqref{EqnTransFOSOne}}

By the same argument via the mean value theorem, we find that
\begin{align*}
\big\| \frac{\partial}{\partial \paramtrans}\qfunpoptrunctrans{\paramtrans}{\paramtrans',\paramobs'} -
\frac{\partial}{\partial \paramtrans}
\qfunpoptrunctrans{\paramtrans}{\paramtrans',\trueparamobs} \big\| &
\leq \Big| \EE \sum_{\ell = -k}^k \left. \frac{\partial^2
  \Phi}{\partial \paramgamma_\ell \partial
  \paramtrans_0}\right|_{\paramjoint= \paramjointtilde} \frac{\langle
  \paramobstwo - \trueparamobs, X_\ell\rangle}{\sigma^2} \Big|.
\end{align*}
Equation~\eqref{EqnExpoCumHessian} guarantees that $\frac{\partial^2
  \Phi}{\partial \paramgamma_\ell \partial \paramtrans_0} =
\condcov{Z_0Z_1}{Z_\ell}{\gamma}$.  Therefore, by similar
arguments as in the proof of inequalities~\eqref{EqnObsFOS}, we
have
\begin{align*}
T & \defn \Big| \sum_{\ell = -k}^k \EE \langle \paramobstwo -
\trueparamobs, X_{\ell}\rangle \EE [\condcov{Z_0
    Z_1}{Z_{\ell}}{\paramgammatilde_\ell,\paramtrans'}| X_{\ell}]
\Big| \\
& \leq \Big | \sum_{\ell=-k}^k\EE |\langle \paramobstwo -
\trueparamobs, X_{\ell}\rangle| \min\{ \mixcoef^{|\ell|}, (\var (Z_{\ell}
\mid \paramgammatilde_{\ell},\paramtrans'))^{1/2}\} \Big| \\
& \leq \Big| \sum_{\ell=-k}^k \min \big \{\mixcoef^{|\ell|} , \sqrt{\EE
  \var( Z_{\ell} \mid \paramgammatilde_{\ell},\paramtrans') } \big \}
\sqrt{ \EE \langle \paramobstwo - \trueparamobs, X_{\ell}\rangle^2 }
\Big|\\
& \leq \sqrt{\|\trueparamobs\|^2_2 + \sigma^2} \big( m \plaincon \,
  \E^{-c\SNR} + 2\sum_{\ell=m+1}^k \mixcoef^{\ell} \big).
\end{align*}
where we have used inequality~\eqref{eq:varSNR} from
Lemma~\ref{lem:expfammixing0}.  Finally, setting $m =
\frac{c\SNR}{\log(1/\mixcoef)}$ yields that the FOS condition holds
with $L_{\paramtrans,2} = c \sqrt{\norm{\trueparamobs}+\sigma^2}
\factormixing \SNR \E^{-c\SNR}$, as claimed.


\subsubsection{Proof of inequality~\eqref{EqnTransFOSTwo}}

By the same mean value argument, we find that
\begin{align*}
\| \frac{\partial}{\partial \paramtrans}
\qfunpoptrunctrans{\paramtrans}{\paramtrans',\paramobs'} -
\frac{\partial}{\partial \paramtrans}
\qfunpoptrunctrans{\paramtrans}{\trueparamtrans,\paramobs'} \big\| &
\leq \Big| \EE \sum_{\ell = -k}^k \left. \frac{\partial^2
  \Phi}{\partial \paramtrans_\ell \partial
  \paramtrans_0}\right|_{\paramjoint = \paramjointtilde}
(\paramtranstwo - \trueparamtrans) \Big|.
\end{align*}
By the exponential family view in equality~\eqref{EqnExpoCumHessian}
it suffices to control the expected conditional covariance.
Lemma~\ref{lem:varsnr} and \ref{lem:expfammixing0} guarantee that
\begin{align}
\label{eq:firstone} 
|\condcov{Z_0 Z_1}{Z_{\ell} Z_{\ell+1}}{X_{-k}^k, \paramgammatilde}| &
\leq \mixcoef^{|\ell|}, \quad \mbox{and} \quad
\EE \var(Z_0 Z_1 \mid \paramgammatilde_0^1, \paramtranstil) \leq
\plaincon \, \frac{1 + b}{1-b}\; \E^{-c\SNR}.
\end{align}
Furthermore, the Cauchy-Schwarz inequality combined with the
bound~\eqref{EqnVarOwn} from Lemma~\ref{lem:varreduce} yields
\begin{align}
\EE \big| \condcov{Z_0 Z_1}{Z_{\ell} Z_{\ell+1}}{\paramgammatilde} \big|
& \leq \sqrt{\EE \var(Z_0 Z_1 \mid \paramgammatilde_{-k}^k,
  \paramtranstilde) } \sqrt{\EE \var(Z_{\ell}Z_{\ell+1} \mid
  \paramgammatilde_{-k}^k, \paramtranstilde)} \nonumber \\
& \leq \sqrt{\EE \var(Z_0 Z_1 \mid \paramgammatilde_0^1,
  \paramtranstil) } \sqrt{\EE \var(Z_{\ell} Z_{\ell+1} \mid
  \paramgammatilde_{\ell}^{\ell+1}, \paramtranstil)} \nonumber \\
\label{EqnBoundCondCov}
& \leq \EE \var(Z_0 Z_1 \mid \paramgammatilde_0^1,
\paramtranstil).
\end{align}
Combining the bounds~\eqref{eq:firstone} and~\eqref{EqnBoundCondCov}
yields
\begin{align*}
\left |\sum_{\ell=-k}^k \EE \frac{\partial^2 \Phi}{\partial
  \paramtrans_{\ell} \partial \paramtrans_0} (\paramtranstwo -
\trueparamtrans) \right| & \leq \sum_{\ell=-k}^k \big|\EE
\condcov{Z_0Z_1}{Z_{\ell}Z_{\ell+1}}{\paramgammatilde_{-k}^k,\paramtranstil} \big| \;
|\paramtranstwo - \trueparamtrans| \\
& \leq \sum_{\ell=-k}^k \min \left\{
  \mixcoef^{|\ell|},\EE
  \var(Z_0Z_1\mid \paramgammatilde_0^1, \paramtranstil) \right\}
  |\paramtranstwo - \trueparamtrans| \\
& \leq 2 \left(\plaincon \frac{1+b}{1-b} \, m \E^{-c \SNR} +
  \sum_{l=m+1}^k \mixcoef^{\ell} \right)
  |\paramtranstwo-\trueparamtrans| \\
& \leq 2 \plaincon \frac{1 + b}{1-b} \factormixing \SNR \E^{-c \SNR}
  |\paramtranstwo-\trueparamtrans|
\end{align*}
where the final inequality follows by setting $m=
\frac{c\SNR}{\log(1/\mixcoef)}$. Therefore, the FOS condition holds
with $L_{\paramtrans,1} = \plaincon \frac{1-b}{1+b} \factormixing \SNR
\E^{-c \SNR}$, as claimed.


\subsection{Proof of Corollary~\ref{cor:normalsim_samplesplit}}
\label{SecProofCorTwo}

In order to prove this corollary, it is again convenient to separate
the updates on the mean vectors $\paramobs$ from those applied to the
transition parameter $\paramtrans$.  
Recall the definitions of
$\BOUNDFUN$, $\BOUNDFUN_{\subsize}$ and $\epsilon_{\subsize}$ from
equations~\eqref{EqnApproxBound} and~\eqref{EqnVarPhiBound}
respectively.

Using Theorem~\ref{ThmSampContraction} we readily have 
that given any initialization $\thetahat^0 \in \paramspacejoint$, with
probability at least $1- 2\delta$, we are guaranteed that
\begin{equation}
\label{EqnHanaFinal}
\addnorm{\paramjointhat^T - \trueparamjoint} \leq \kappa^T \addnorm{\paramjointhat^0 -
\trueparamjoint} + \frac{2 \varphi_{\subsize}(\delta, k)
  + \epsilon_{\subsize}(\delta, k) + \BOUNDFUN
  (\kdim)}{1-\kappa}.
\end{equation}

In order to leverage the bound~\eqref{EqnHanaFinal}, we need to find
appropriate upper bounds on the quantities $\varphi_{\subsize}(\delta,\kdim)$, 
$\epsilon_{\subsize}(\delta,\kdim)$.
\begin{lems}
\label{lem:second_tech}
Suppose that the truncation level satisifes the lower bound
\begin{subequations}
\begin{align}
\label{EqnKlower}
\kdim & \geq \log \left( \frac{\consteps \numobs}{\delta} \right) \;
\Big (\log \frac{1}{(1- \frac{\mixcoefeps}{2})} \Big)^{-1} \qquad
\mbox{where $\consteps \defn \frac{36}{\mixcoefeps^3}$.}
\end{align}
Then with the radius $r = \frac{\|\mustar\|_2}{4}$, we have
\begin{align}
\label{eq:epsobs}
\epsilonobs_{\subsize} \big(\delta, \kdim \big) & \leq \BIGCON_0 
\norm{\trueparamobs}\big(\frac{\norm{\trueparamobs}^2}{\sigma^2} + 1\big) \log ( k^2/\delta) \sqrt{\frac{\kdim^3 \usedim \log n }{\numobs}},
\quad \mbox{and} \\
\label{eq:epstrans}
\epsilontrans_{\subsize} \big(\delta, \kdim \big) & \leq \BIGCON_0 \;
\frac{\sqrt{\norm{\trueparamobs}^2 + \sigma^2}}{\sigma^2} \sqrt{\frac{ \kdim^3 \log
    (\kdim^2/\delta)}{\numobs}}.
\end{align}
\end{subequations}
\end{lems}

\begin{lems}
\label{lem:first_tech}
Suppose that the SNR $\eta^2 = \frac{\|\mustar\|_2^2}{\sigma^2}$ is
lower bounded as $\SNR \geq C (\log \log \usedim + \log
\|\mustar\|_2)$ for a sufficiently large $C$. Then with the radius $r
= \frac{\|\mustar\|_2}{4}$, we have
\begin{align}
\label{EqnBoundFunControl}
\BOUNDFUN^2_{\subsize}(\delta, \kdim) & \leq \BIGCON_1 \; \Big\{
\frac{1}{\sigma} \sqrt{ \frac{\usedim \log^2 (\consteps \numobs/
    \delta)}{\numobs} } + \frac{\norm{\mustar}}{\sigma} \sqrt{ \frac{
    \log ^2(\consteps \numobs/ \delta)}{\numobs}} +
\frac{\norm{\trueparamobs}^2}{\sigma^2} \Big\} \; \BOUNDFUN^2(\kdim).
\end{align}
\end{lems}
\noindent See Appendices~\ref{sec:proof_second_tech}
and~\ref{sec:proof_first_tech}, respectively, for the proofs of these
two lemmas. \\

Using these two lemmas, we can now complete the proof of the
corollary.  From the definition~\eqref{EqnDefnKappa} of $\kappa$,
under the stated lower bound on $\SNR$, we can ensure that $\kappa
\big \{ \frac{4 \log \frac{1+b}{1-b}}{\|\mustar\|_2} \big\} \leq 1/2$.
Under this condition, inequality~\eqref{EqnConditionSamplesplit} with \mbox{$r
  = \|\mustar\|_2/4$} reduces to showing that
\begin{align}
\label{EqnComplicatedTwo}
\BOUNDFUN_{\subsize}(\delta, \kdim) + \epsilonobs_{\subsize}
(\delta, \kdim) + \BOUNDFUN (\kdim) \leq (1 - 2 \kappa)
\frac{\|\trueparamobs\|_2}{8}.
\end{align}
As long as $\numobs \geq \BIGCON_1 (\sigma^2 + \|\mustar\|_2^2) 
\usedim \log^2(  \usedim/\delta)$ for a sufficiently large
$\BIGCON_1$, we are guaranteed that the
bound~\eqref{EqnComplicatedTwo} holds. 
Now the specific choice~\eqref{EqnKlower} of $k$ guarantees that
\begin{align}
\label{EqnHanaTiger}
\varphi_{\subsize} (\delta, \kdim) + \epsilonobs_{\subsize} (\delta,
\kdim) + \epsilontrans_{\subsize} (\delta, \kdim) + \BOUNDFUN(\kdim)
\leq C \norm{\trueparamobs}
(\frac{\norm{\trueparamobs}^2}{\sigma^2}+1) \log^2 n \sqrt{\frac{ d
    \log^2 (\numobs/\delta)}{\numobs}}.
\end{align}
Substituting the bound~\eqref{EqnHanaTiger} into
inequality~\eqref{EqnHanaFinal} completes the proof of the corollary.


\section{Discussion}

In this paper, we provided general global convergence guarantees for
the Baum-Welch algorithm as well as specific results for a hidden
Markov mixture of two isotropic Gaussians.  In contrast to the
classical perspective of focusing on the MLE, we focused on bounding
the distance between the Baum-Welch iterates and the true parameter.
Under suitable regularity conditions, our theory guarantees that the
iterates converge to an $\minimaxrad$-ball of the true parameter,
where $\minimaxrad$ represents a form of statistical error.  It
is important to note that our theory does not guarnatee convergence to
the MLE itself, but rather to this ball that contains both the MLE and
the true parameter.  When applied to the Gaussian mixture HMM, we
proved that the Baum-Welch algorithm achieves estimation error that is minimax
optimal up to logarithmic factors.  To the best of our knowledge,
these are the first rigorous guarantees for the Baum-Welch algorithm
that allow for a large initialization radius.


\subsection*{Acknowledgements}
This work was partially supported by NSF grant CIF-31712-23800,
ONR-MURI grant DOD-002888, and AFOSR grant FA9550-14-1-0016 to MJW.


\appendix

\section{Proof of Proposition~\ref{PropExistence}}
\label{AppPropExistence}

In order to show that the limit $\lim_{n\to\infty} \EE
\qfunsamp{\paramjointone}{\paramjointtwo}$ exists, it suffices to show
that the sequence of functions $\{\eqfunn{1}, \eqfunn{2}, \ldots,
\eqfunn{n}\}$ is Cauchy in the sup-norm (as defined previously in
equation~\eqref{eq:q-norm}).  In particular, it suffices to show that
for every $\epsilon > 0$ there is a positive integer $N(\epsilon)$
such that for $m,n \geq N(\epsilon)$,
\begin{align*}
\qnorm{\eqfunn{m} - \eqfunn{n}} \leq \epsilon.
\end{align*}
In order to do so, we make use of the previously stated
bound~\eqref{EqnPopulationTruncateOne} relating $\eqfunn{\numobs}$ to
$\qfunpopextendk$.  Taking this bound as given for the moment, an
application of the triangle inequality yields
\begin{align*}
\qnorm{\eqfunn{m} - \eqfunn{n}} \leq \qnorm{\eqfunn{m} -
  \qfunpopextendk} + \qnorm{\eqfunn{n} - \qfunpopextendk} \; \leq \;
\epsilon,
\end{align*}
the final inequality follows as long as we choose $N(\epsilon)$ and
$\kdim$ large enough (roughly proportional to $\log(1/\epsilon)$).


It remains to prove the claim~\eqref{EqnPopulationTruncateOne}.  In
order to do so, we require an auxiliary lemma:

\begin{lems}[Approximation by truncation]
\label{lem:twosided_approx}
For a Markov chain satisfying the mixing condition~\eqref{ass:mixing},
we have 
\begin{align}
\sup_{\paramjoint' \in \DomTheta} \sup_x \sum_{z_i} |p(z_i \mid x_1^n; \paramjointtwo) - p(z_i \mid
x_{\blockleftind}^{\blockrightind}; \paramjointtwo)| \leq
\frac{C\nstates^4}{\mixcoefeps^8 \statmin}\big(1- \mixcoefeps \statmin \big)^\kdim
\end{align}
for all $i \in [0,n]$, where $\statmin = \min_{j\in [\nstates], \paramtrans\in \paramspacetrans} \pistat( j ; \paramtrans)$.
\end{lems}
\noindent See Appendix~\ref{AppLemTwoSided} for the proof of this
lemma.\\

Using Lemma~\ref{lem:twosided_approx}, let us now prove the
claim~\eqref{EqnPopulationTruncateOne}.  Introducing the shorthand
notation
\begin{align*}
\HACKG(X_i,z_i,\paramjoint,\paramjointtwo) & \defn \log p(X_i \mid
z_i; \paramjoint) + \sum_{z_{i-1}} p(z_i \mid z_{i-1}; \paramjointtwo)
\log p(z_i|z_{i-1},\paramjoint),
\end{align*}
we can verify by applying Lemma~\ref{lem:twosided_approx} that
\begin{align*}
\qnorm{\eqfunn{n} - \qfunpopextendk} & = \Big |
\sup_{\paramjoint,\paramjointtwo} \frac{1}{n} \sum_{i=1}^n \sum_{z_i}
\EE (p(z_i \mid X_{1}^n,\paramjointtwo) - p(z_i \mid
X_{i-k}^{i+k},\paramjointtwo)) \HACKG( X_i, z_i, \paramjointone,
\paramjointtwo) \Big | \\
& + \Big | \frac{1}{n} \sup_{\paramjoint, \paramjoint'} \EE \sum_{z_0}
p(z_0 \mid X_1^\numobs, \paramjoint') \log p(z_0 ;\paramjoint) \Big|
\\
& \leq \sup_{\paramjointone,\paramjointtwo} \sup_{x} \sum_{z_i}\big|
p(z_i \mid x_{1}^{n},\theta') - p(z_i \mid
x_{i-k}^{i+k},\paramjointtwo) \big| \left[ \EE \frac{1}{\numobs}
  \sum_{i=1}^\numobs \max_{z_i \in
    [m]}|\HACKG(X_i,z_i,\paramjointone,\paramjointtwo)| \right] +
\frac{1}{\numobs} \log \statmin^{-1}\\
& \leq \frac{c \nstates^4}{\mixcoefeps^8 \statmin}\big(1- \mixcoefeps
\statmin \big)^k \left( \EE \max_{z_i \in [m]} \big| \log p(X_i \mid
z_i,\theta) \big| + \nstates \log \statmin^{-1} \right) +
\frac{1}{\numobs} \log \statmin^{-1}\\
& \leq \frac{c' \, \nstates^5}{\mixcoefeps^8 \statmin^2} \big(1-
\mixcoefeps \statmin \big)^k + \frac{1}{\numobs} \log \statmin^{-1},
\end{align*}
where the last inequality follows from the boundedness
condition~\eqref{EqnDensityBounded} on the log output densities.











\section{Technical details for Corollary~\ref{cor:simultaneous_normal_lipschitz}}

The proof of the corollary mainly involves proving the
bound~\eqref{EqnPopulationSeparate}, and then showing that the
conditions in the theorem hold for the special Gaussian output HMM.

\subsection{Bounds on conditional variances}
\label{AppExponential}

In this section, we collect some auxiliary bounds on conditional
covariances in hidden Markov models.  These results are used in the
proof of Corollary~\ref{cor:simultaneous_normal_lipschitz}.

\begin{lems}
\label{lem:varreduce}
For any HMM with observed-hidden states $(X_i, Z_i)$, we have
\begin{subequations}
\begin{align}
\label{EqnVarOwn} 
\EE \left[ \var(Z_0 Z_1\mid X_{-\kdim}^\kdim) \right] & \leq \EE
\var(Z_0 Z_1 \mid X_0^1) \\
\label{EqnVarOwn2}
\EE \left[ \var(Z_0 \mid X_{-\kdim}^\kdim) \mid X_0 \right] &\leq
\var(Z_0 \mid X_0)
\end{align}
\end{subequations}
where we have omitted the dependence on the parameters. 
\end{lems}

\begin{proof}
We use the law of total variance, which guarantees that $\var Z =
\EE \big[ \var (Z \mid X) \big] + \var \EE [ Z \mid X]$.
Using this decomposition, we have
\begin{align*}
\EE [\var(Z_0 \mid X_0^1) \mid X_0] &\leq \var (Z_0
\mid X_0)\\
\EE [\var(Z_0 Z_1 \mid X_0^2) \mid X_0^1] 
&\leq \var(Z_0Z_1 \mid X_0^1).
\end{align*} 
The result then follows by induction.
\end{proof}


We now show that the expected conditional variance of the hidden state
(or pairs thereof) conditioned on the corresponding observation (pairs
of observations) decays exponentially with the SNR.
\begin{lems}
\label{lem:varsnr}
For a $2$-state Markov chain with true parameter $\trueparamjoint$, we
have for $\paramobs \in
\elltwoballr{\frac{\|\trueparamobs\|_2}{4}}{\trueparamobs}$ and
$\paramtrans \in \paramspacetrans$
\begin{subequations}
\begin{align}
\label{eq:covXSNR} 
 \left\| \EE X_0 X_0^T (\var(Z_0 \mid \paramgamma_0,
 \paramtrans))^{1/2} \right\|_{op} & \leq \plaincon_0 \, \E^{-c\SNR}\\
\label{eq:varSNR} 
\EE \var (Z_\ell \mid \paramgamma_\ell, \paramtrans) & \leq
\plaincon_0 \; \E^{-c\SNR} \\
\label{eq:varTwoSNR}
\EE \var(Z_0 Z_1 \mid \paramgamma_0^1, \paramtrans) & \leq \plaincon_0
\frac{1 + b}{1-b}\; \E^{-c\SNR}.
\end{align}
\end{subequations}
\end{lems}
\noindent The last lemma provides rigorous confirmation of the
intuition that the covariance between any pair of hidden states should
decay exponentially in their separation $\ell$:
\begin{lems}
\label{lem:expfammixing0}
For a $2$-state Markov chain with mixing coefficient $\mixcoefeps$
and uniform stationary distribution, we have
\begin{align}
\label{eq:expfamonestatemixing}
\max \Big \{ \condcov{Z_0}{Z_{\ell}}{\paramgamma}, \;
\condcov{Z_0Z_1}{Z_{\ell}Z_{\ell+1}}{\paramgamma}, \;
\condcov{Z_0}{Z_{\ell}Z_{\ell+1}}{\paramgamma} \Big \} \leq 2
\mixcoef^{\ell}
\end{align}
with $\mixcoef = 1 - \mixcoefeps$ for all $\paramjoint \in \DomTheta$.
\end{lems}
\noindent 

Lemma~\ref{lem:varsnr} is proven in Section~\ref{sec:lemvarsnr}
whereas Lemma~\ref{lem:expfammixing0} is a mixing result and its proof
is found in Section~\ref{sec:expfammixing}.


\subsection{Proof of Lemma~\ref{lem:varsnr}}
\label{sec:lemvarsnr}

By definition of the Gaussian HMM example, we have $\var(Z_i \mid
\paramgammatilde_i) = \frac{4}{(\E^{\paramgammatilde_i} +
  \E^{-\paramgammatilde_i})^2}$.  Moreover, following the proof of
Corollary 1 in the paper~\cite{BalWaiYu14}, we are guaranteed that
$\EE \var(Z_i \mid \paramgammatilde_i) \leq 8
\E^{-\frac{\eta^2}{32}}$ and ${\|\EE X_iX_i^T (\var(Z_i|\paramgammatilde_i))^{1/2}\|_{op} \leq c_0\E^{-\frac{\SNR}{32}}}$, from which inequalities~\eqref{eq:covXSNR} and~\eqref{eq:varSNR}
follow.\\
We now prove inequality~\eqref{eq:varTwoSNR} for $\paramtrans \in
\paramspacetrans$ and $\paramobs \in
\elltwoballr{\frac{\norm{\trueparamobs}}{4}}{\trueparamobs}$. Note
that
\begin{align*}
\frac{1}{4} \var(Z_0 Z_1 \mid \paramgamma_0^1,\paramtrans) &=
\frac{\E^{2\paramgamma_1} +\E^{-2\paramgamma_1} + \E^{2\paramgamma_0}
  + \E^{-2\paramgamma_0}}{\left[\E^{\paramtrans}(\E^{\paramgamma_0 +
      \paramgamma_1}+ \E^{-(\paramgamma_0+ \paramgamma_1)}) +
    \E^{-\paramtrans} (\E^{\paramgamma_0- \paramgamma_1}+
    \E^{-(\paramgamma_0-\paramgamma_1)})\right]^2}\\
 &\leq \E^{2|\paramtrans|} \frac{\E^{2\paramgamma_1}
  +\E^{-2\paramgamma_1} + \E^{2\paramgamma_0} +
  \E^{-2\paramgamma_0}}{(\E^{\paramgamma_0 + \paramgamma_1}+
  \E^{-(\paramgamma_0+ \paramgamma_1)}+\E^{\paramgamma_0-
    \paramgamma_1}+ \E^{-(\paramgamma_0-\paramgamma_1)})^2}\\
&\leq \left( \frac{1 + \mixcoefbound}{1 - \mixcoefbound}\right) \left[
  \frac{\E^{|\paramgamma_0|}}{\E^{2\paramgamma_0} +
    \E^{-2\paramgamma_0}} +
  \frac{\E^{|\paramgamma_1|}}{\E^{2\paramgamma_1} +
    \E^{2\paramgamma_1}}\right]
\end{align*}
where $\paramgamma$ are now random variables and we used
\begin{align*}
&(\E^{\paramgamma_0 + \paramgamma_1}+ \E^{-(\paramgamma_0+
    \paramgamma_1)}+\E^{\paramgamma_0- \paramgamma_1}+
  \E^{-(\paramgamma_0-\paramgamma_1)})^2 \\
& \geq \E^{-|\paramgamma_0|}(\E^{-\paramgamma_0} +
  \E^{\paramgamma_0})(\E^{2\paramgamma_1}+ \E^{-2\paramgamma_1}) +
  \E^{-|\paramgamma_1|}(\E^{-\paramgamma_1} +
  \E^{\paramgamma_1})(\E^{2\paramgamma_0}+ \E^{-2\paramgamma_0})\\
& \geq (\E^{-|\paramgamma_0|} +
  \E^{-|\paramgamma_1|})(\E^{2\paramgamma_0}+
  \E^{-2\paramgamma_0})(\E^{2\paramgamma_1}+ \E^{-2\paramgamma_1})
\end{align*}
It directly follows that
\begin{align*}
\frac{1}{4} \EE \var(Z_0 Z_1 \mid \paramgamma_0^1,\paramtrans) &\leq 2
\left( \frac{1 + \mixcoefbound}{1 - \mixcoefbound}\right) \EE\left[
  \frac{1}{\E^{\paramgamma_0} + \E^{-3 \paramgamma_0}}
  \Indi_{\paramgamma_0\geq 0} +\frac{1}{\E^{3\paramgamma_0 } +
    \E^{-\paramgamma_0}} \Indi_{\paramgamma_0 \leq 0} \right] \\
& \leq 2 \left( \frac{1 + \mixcoefbound}{1 - \mixcoefbound}\right)
(\EE [\E^{-\paramgamma_0} \Indi_{\paramgamma_0\geq 0}] + \EE
     [\E^{-\paramgamma_0}\Indi_{\paramgamma_0\leq 0}]) \\
& \leq 4 \left( \frac{1 + \mixcoefbound}{1 - \mixcoefbound} \right)
     \EE [\E^{-\paramgamma_0} \Indi_{\paramgamma_0\geq 0}]
\end{align*}
where the last inequality follows from symmetry of the random
variables $X_i$.  It is easy to bound
\begin{align*}
\EE [\E^{-\paramgamma_0} \Indi_{\paramgamma_0\geq 0}] = \EE \:
\E^{-\frac{\|\paramobs\|_2 V_1}{\sigma^2}}\Indi_{V_1 \geq 0} \leq 2
\E^{-\frac{\eta^2}{32}}
\end{align*}
by employing a similar procedure as in the proof of Corollary 1
in~\cite{BalWaiYu14}. Inequality~\eqref{eq:varTwoSNR} then follows.


\section{Technical details for Corollary~\ref{cor:normalsim_samplesplit}}

In this section we prove Lemmas~\ref{lem:second_tech}
and~\ref{lem:first_tech}.  In order to do so, we leverage the
independent blocks approach used in the analysis of dependent data
(see, for instance, the papers~\cite{Yu94, NobDem93}).  For future
reference, we state here an auxiliary lemma that plays an important
role in both proofs.

Let $X_{-\infty}^\infty$ be a sequence sampled from a Markov chain
with mixing rate $\mixcoef = 1- \mixcoefeps$, and let $\statmin$ be
the minimum entry of the stationary distribution.  Given some
functions $f_1 : \RN^{2 \kdim} \to \RN^\usedim$ and $f_2: \RN \to
\RN^{\usedim}$ respectively, our goal is to control the difference
between the functions
\begin{subequations}
\begin{align}
\label{EqnGfun}
g_1(X)  \defn \frac{1}{\subsize} \sum_{i=1}^{\subsize} f_1(X_{i-k}^{i+k}), \qquad 
g_2(X) \defn \frac{1}{\subsize} \sum_{i=1}^{\subsize} f_2 (X_i)
\end{align}
and their expectation.  Defining $\blocksize_1 \defn \lfloor \subsize
/4k \rfloor$ and $\blocksize_2 \defn \lfloor \subsize/ k\rfloor$, we
say that $f_1$ respectively $f_2$ is $(\delta, \kdim)$-concentrated if
\begin{align}
\label{EqnVenkatConcentrate}
\mprob \Big[ \|\frac{1}{\blocksize_1} \sum_{i=1}^{\blocksize_1}
  f_1(\Xtil_{i;2k}) - \EE f_1(\Xtil_{1;2k})\|_2 \geq \epsilon \Big] & \leq
\frac{\delta}{8\kdim},\\
 \mprob \Big[  \|\frac{1}{\blocksize_2} \sum_{i=1}^{\blocksize_2}
  f_2(\Xtil_{i}) - \EE f_2(\Xtil_{1})\|_2 \geq \epsilon  \Big] &\leq \frac{\delta}{2\kdim}\nonumber
\end{align}
\end{subequations}
where $\{ \Xtil_{i; 2 \kdim}\}_{i\in \NN}$ are a collection of
i.i.d. sequences of length $2 \kdim$ drawn from the same Markov chain
and $\{\Xtil_{i}\}_{i\in\NN}$ a collection of i.i.d. variables drawn
from the same stationary distribution. In our notation, $\{\Xtil_{i;2k}\}_{i\in \NN}$ 
under $\mprob$ are distributed identically distributed as $\{X_{i;2k} \}_{i\in\NN}$ 
under $\mprob_0$.
\begin{lems}
\label{lem:IB}
Consider functions $f_1, f_2$ that are $(\delta,
\kdim)$-concentrated~\eqref{EqnVenkatConcentrate} for a truncation
parameter $\kdim \geq \log \left( \frac{6 \subsize}{\statmin^2
  \mixcoefeps^3 \delta} \right) (\log \frac{1}{1 -
  \mixcoefeps})^{-1}$.  Then the averaged functions $g_1, g_2$ from
equation~\eqref{EqnGfun} satisfy the bounds
\begin{align}
\label{EqnGfunBound}
\mprob \Big[ \|g_1(X) - \EE g_1(X)\|_2 \geq \epsilon \Big] \leq \delta
\quad \text{ and } \quad \mprob \Big[ \|g_2(X) - \EE g_2(X)\|_2 \geq
  \epsilon \Big] \leq \delta.
\end{align}
\end{lems}
\begin{proof}
We prove the lemma for functions of the type $(f_1, g_1)$; the proof
for the case $(f_2, g_2)$ is very similar.  In order to simplify
notation, we assume throughout the proof that the effective sample
size $\subsize$ is a multiple of $4\kdim $, so that the block size
$\blocksize = \frac{\subsize}{4 \kdim}$ is integral.  By
definition~\eqref{EqnGfun}, the function $g$ is a function of the
sequences $\{X_{1-k}^{1+k},
X_{2-k}^{2+k},\ldots,X_{\subsize-k}^{\subsize+k}\}$.  We begin by
dividing these sequences into blocks.  Let us define the subsets of
indices
\begin{align*}
\oddblockindeces{j}{i} &= \{ 4k(i-1)+k+j \mid 4k(i-1) +3k +j\}, \quad
\mbox{and} \\
 \evenblockindeces{j}{i} &= \{4k(i-1) -k+j \mid
4k(i-1) +k - 1 +j\}.
\end{align*}
With this notation, we have the decomposition
\begin{align*}
g(X) & = \frac{1}{2} \left(\frac{1}{2k}\sum_{j=1}^{2k}
\underbrace{\frac{1}{\blocksize} \sum_{i=1}^\blocksize
  f(X_{H_i^j})}_{g^{H^j}(X)} + \frac{1}{2k} \sum_{j=1}^{2k}
\underbrace{\frac{1}{\blocksize} \sum_{i=1}^\blocksize
  f(X_{R_i^j})}_{g^{R^j}(X)} \right),
\end{align*}
from which we find that
\begin{align*}
\mprob \big[ \| g(X) - \EE g(X)\|_2 \leq \epsilon \big] & \geq \mprob
\big( \bigcap_{j=1}^{2k} \{\| g^{H^j}(X)- \EE g(X) \|_2 \leq
\epsilon\}\cap \{\| g^{R^j}(X) - \EE g(X)\|_2 \leq \epsilon\} \big)\\
&\stackrel{\mathrm{(i)}}{\geq} 1 - 4 \kdim \mprob(\| g^{H^1}(X) - \EE
g(X)\|_2 \geq \epsilon),
\end{align*}
where $\mathrm{(i)}$ follows using stationarity of the underlying
sequence combined with the union bound.

In order to bound the probability $\mprob \big[ \| g^{H^1}(X) - \EE
  g(X)\|_2 \geq \epsilon\big]$, it is convenient to relate it to the
probability of the same event under the product measure $\mprob_0$ on
the blocks $\{H^1_1,\ldots, H^1_{\blocksize}\}.$ In particular, we
have $\mprob(\| g^{H^1}(X) - \EE g(X)\|_2 \geq \epsilon) \leq T_1 +
T_2$, where
\begin{align*}
T_1 & \defn \mprob_0(\| g^{H^1}(X) - \EE g(X)\|_2 \geq \epsilon),
\quad \mbox{and} \\
T_2 &\defn |\mprob(\| g^{H^1}(X) - \EE g(X)\|_2 \geq \epsilon) -
\mprob_0(\| g^{H^1}(X) - \EE g(X)\|_2 \geq \epsilon)|.
\end{align*}
By our assumed concentration~\eqref{EqnVenkatConcentrate}, we have
$T_1 \leq \frac{\delta}{8k}$, and so it remains to show that $T_2
\leq \frac{\delta}{8k}$.

Now following standard arguments (e.g., see the papers~\cite{NobDem93,
  Yu94}), we first define
\begin{align}
\label{eqn:betadef}
\paramtrans(k) = \sup_{A \in \sigma(\mathcal{S}_{-\infty:0},
  \mathcal{S}_{k:\infty})} |\mprob(A) - \mprob_{-\infty}^0 \times
\mprob_1^\infty(A)|,
\end{align}
where $\mathcal{S}_{-\infty:0}$ and $\mathcal{S}_{k:\infty}$ are the
$\sigma$-algebras generated by the random vector $X_{-\infty:0}$ and
$X_{k:\infty}$ respectively, and $\mprob_{-\infty}^0 \times
\mprob_1^\infty$ is the product measure under which the sequences
$X_{-\infty:0}$ and $X_{1:\infty}$ are independent. Define
$\mathcal{S}_i$ to be the $\sigma$-algebra generated by $X_{H_i^j}$
for $i = \{1,\ldots,\blocksize\}$; it then follows by induction that
$\sup_{A \in \sigma(\mathcal{S}_1, \dots, \mathcal{S}_{\blocksize})}
|\mprob(A) - \mprob_0(A)|\leq \blocksize \paramtrans(k)$.  An
identical relation holds over the blocks $R_i^j$.

For our two-state HMM, Lemma~\ref{lem:dependonfar0}
implies that
\begin{align}
\paramtrans(k) = |p(x) - p(x_{k:\infty})p(x_{-\infty}^{0})| & \leq
|p(x_{-\infty}^{0}|x_{k}^{n}) - p(x_{-\infty}^{0})| \nonumber \\
&\leq |p(z_0|x_{k}^{n}) - p(z_0)| \nonumber \\
\label{eqn:betabound}
& \stackrel{\mathrm{(i)}}{\leq} \frac{3}{\statmin^2 \mixcoefeps^3} \mixcoef^k =
\frac{3}{\statmin^2 \mixcoefeps^3} \E^{- k \log(1/\mixcoef)},
\end{align}
where step (i) follows from inequality~\eqref{EqnDependOnFar}. From our
assumed lower bound on $\kdim$, we conclude that $\blocksize
\paramtrans(k) \leq \frac{\delta}{8k}$, which completes the proof.
\end{proof}

\noindent 
In the following sections we apply it in order to prove the bounds
on the approximation and sample error of the $M$-operators.

\subsection{Proof of  Lemma~\ref{lem:second_tech}}
\label{sec:proof_second_tech}

\noindent We prove each of the two inequalities in
equations~\eqref{eq:epsobs} and~\eqref{eq:epstrans} in turn.  With
suitable choices of the function $f$ in Lemma~\ref{lem:IB}, we can
prove the upper bounds stated in equations~\eqref{eq:epsobs}
and~\eqref{eq:epstrans}.



\myparagraph{Proof of inequality~\eqref{eq:epsobs}} We use the
notation from the proof of Lemma~\ref{lem:IB} and furthermore define
the weights $\weightsingle{\paramjoint}{X_{i-k}^{i+k-1}} = p(Z_0 =1
\mid X_{i-k}^{i+k-1},\theta)$, as well as the function
\mbox{$f_0(X_{i-k}^{i+k-1}) = (2
  \weightsingle{\paramjoint}{X_{i-k}^{i+k-1}} - 1) X_0$.}  It is then
possible to write the sample splitting EM operator explicitly as the
average
\begin{align*}
\emopsamptruncobs{\paramjoint} & = \argmax_{\paramobs}
\frac{1}{\subsize} \Big[ \sum_{i=1}^{\subsize}
  \EEzcondx{Z_i}{X_{i-\kdim}^{i+\kdim}}{\paramjointwo} \log p(x_i \mid
  Z_i,\paramobsone) \Big] = \frac{1}{\subsize} \sum_{i=1}^{\subsize}
f_0 (X_{-k}^{i+k-1}).
\end{align*}
We are now ready to apply Lemma~\ref{lem:IB} with the choices $f_1 =
f_0$, $g_1(X) = \emopsamptruncobs{\paramjoint} $ and $\delta =
\subprob$, and $n = \subsize$.  According to Lemma~\ref{lem:IB}, given
that the lower bound on the truncation parameter $k$ holds, we now
need to show that $f_0$ is $(\subprob,\kdim)$-concentrated, that means
finding $\epsilonobs_{\numobs}$ such that
\begin{align*}
\mprob_0\left [\big\| \frac{1}{\blocksize}\sum_{i=1}^\blocksize
  f_0(\Xtil_{i;2k}) - \EE f_0(\Xtil_{i;2k}) \big\|_2 \geq
  \epsilonobs_{\numobs} \right] & \leq \frac{\subprob}{8 \kdim},
\end{align*}
where $\mprob_0$ denotes the product measure over the independent blocks
and $\blocksize \defn \blocksize_1 = \lfloor \subsize/ 4 \kdim \rfloor$.

Let $X_i$ be the middle element of the block $\Xtil_{i;2k}$ and $Z_i$,
$V_i$ the corresponding latent and noise variable. We can then write
$X_i = Z_i + V_i$ where $V_i$ are zero-mean Gaussian random variables
with covariance matrix $\sigma^2 I$.

With a minor abuse of notation, let us use $X_{i,\ell}$ to denote
$\ell^{th}$ element in the block $\Xtil_{i;2k} = (X_{i,1},\dots,
X_{i,2k})^T$, and write $\Xtil = \{\blockXi\}_{i=1}^n$.  In view of
Lemma~\ref{lem:IB}, our objective is to find the smallest scalar
$\epsilonobs_{\numobs}$ such that
\begin{equation}
\label{EqnUniform}
\mprob \big [\sup_{\paramjoint\in\paramspacejoint}\norm{\frac{1}{m}
    \sum_{i=1}^m \underbrace{(2
      \weightsingle{\paramjoint}{\Xtil_{i;2k}} -1)X_{i,k} - \EE (2
      \weightsingle{\paramjoint}{\Xtil_{i;2k}} -1)X_{i,k}}_{
      \funcproc(\blockXi)}} \geq \epsilonobs_\numobs \big] \leq
\frac{\delta}{8k}
\end{equation}
For each unit norm vector $u \in \real^\usedim$, define the random
variable 
\begin{align*}
\processobs(\Xtil;u) = \sup_{\paramjoint\in\paramspacejoint}
\frac{1}{m}\sum_{i=1}^m (2 \weightsingle{\paramjoint}{\Xtil_{i;2k}}
-1) \langle X_{i,k},u \rangle - \EE (2
\weightsingle{\paramjoint}{\Xtil_{i;2k}} -1)\langle X_{i,k}, u\rangle.
\end{align*}
Let $\{u^{(1)}, \ldots, u^{(\COVNUM)} \}$ denote a $1/2$-cover of the
unit sphere in $\real^\usedim$; by standard arguments, we can find
such a set with cardinality $\log \COVNUM \leq \usedim \log 5$.  Using
this covering, we have
\begin{align*}
\sup_{\paramjoint\in\paramspacejoint}\norm{\frac{1}{m} \sum_{i=1}^m
  \funcproc(\blockXi)} & = \sup_{\norm{u}\leq 1} \processobs(\Xtil; u)
\leq 2 \max_{j\in [\COVNUM]} \processobs(\Xtil ;u^{(j)}),
\end{align*}
where the inequality follows by a discretization argument.  
Consequently, we have
\begin{align*}
\mprob \big[ \sup_{\paramjoint\in\paramspacejoint}\norm{\frac{1}{m}
    \sum_{i=1}^m \funcproc(\blockXi)} \geq \epsilonobs_\numobs \big] &
\leq \mprob \big[ \max_{j \in [\COVNUM]} \processobs(\Xtil;u^{(j)})
  \geq \frac{\epsilonobs_\numobs}{2} \big] \\
& \leq \COVNUM \, \max_{j \in [\COVNUM]} \mprob \big[
  \processobs(\Xtil;u^{(j)}) \geq \frac{\epsilonobs_\numobs}{2} \big].
\end{align*} 

The remainder of our analysis focuses on bounding the tail probability
for a fixed unit vector $u$, in particular ensuring an exponent small
enough to cancel the $\COVNUM \leq e^{\usedim \log 5}$ pre-factor.  By
Lemma 2.3.7 of van der Vaart and Wellner~\cite{vdVWell}, for any $t >
0$, we have
\begin{align*}
\mprob_X \big[\processobs(\Xtil; u) \geq t \big] & \leq c
\mprob_{X,\epsilon} \big[ \processradobs(\Xtil; u) \geq \frac{t}{4}
  \big],
\end{align*}
where $\processradobs(\Xtil;u) = \sup_{\paramjoint\in\paramspacejoint}
\big| \frac{1}{m}\sum_{i=1}^m \epsilon_i (2
\weightsingle{\paramjoint}{\Xtil_{i;2k}} -1) \langle X_{i,k},u \rangle
\big|$, and $\{\epsilon_i\}_{i=1}^m$ is a sequence of
i.i.d. Rademacher variables.

We now require a sequence of technical lemmas; see
Section~\ref{SecNewTechnical} for their proofs.  Our first lemma shows
that the variable $\processradobs(\Xtil; u)$, viewed as a function of
the Rademacher sequence, is concentrated:
\begin{lems}
\label{LemConcLipproc}
For each fixed $(\Xtil,u)$, we have
\begin{align}
\mprob_{\epsilon} \big[ \processradobs(\Xtil;u) \geq \EE_{\epsilon}
  \processradobs(\Xtil;u) + t \big] & \leq 2 \E^{-\frac{t^2}{16
    \lipproc^2(\Xtil;u)}},
\end{align}
where $\lipproc(\Xtil;u) = \frac{1}{m} \sqrt{\sum_{i=1}^m \langle
  X_{i,k},u\rangle^2}$.
\end{lems}

\noindent Our next lemma bounds the expectation with respect to the
Rademacher random vector:
\begin{lems}
\label{LemContractions}
There exists a universal constant $c$ such that for each fixed
$(\Xtil;u)$, we have
\begin{align}
\label{EqnExp}
\EE_{\epsilon} \processradobs(\Xtil;u) &\leq \underbrace{c
  \frac{\norm{\trueparamobs}}{\sigma^2} \sqrt{\log m} \; \Big[
    \sum_{\ell=1}^{2k} \EE_{\epsilontwo}\norm{\frac{1}{m} \sum_{i=1}^m
      \epsilontwo_{i,\ell} X_{i,\ell} \langle X_{i,k},u\rangle}
    \Big]}_{ \processradM (\Xtil;u)} + \underbrace{\EE_g
  \big|\frac{1}{m}\sum_{i=1}^m g_{i,2k+1} \langle X_{i,k},
  u\rangle\big|}_{\processradN(\Xtil;u)}
\end{align}
where $\epsilon, \epsilontwo \in \real^m$ are random vectors with
i.i.d. Rademacher components, and $g$ is a random vector with
i.i.d. $N(0,1)$ components.
\end{lems}

We now bound the three quantities $\lipproc(\Xtil;u)$,
$\processradM(\Xtil;u)$, and $\processradN(\Xtil;u)$ appearing in the
previous two lemmas.  In particular, let us introduce the quantities
$L' = c \lipcont \norm{\trueparamobs}\big(
\frac{\norm{\trueparamobs}^2}{\sigma^2}+1\big)$, $L'' = \lipcont
\sqrt{\norm{\trueparamobs}^2+\sigma^2} $ and $\lipcont =
\frac{\sqrt{8}}{1-\mixcoef}$.

\begin{lems}
\label{LemEventBound}
Define the event
\begin{align*}
\mathcal{E} =\Biggr \{ &\lipproc(\Xtil;u) \leq
\sqrt{\frac{2(\norm{\trueparamobs}^2 + \sigma^2)\log
    \frac{1}{\delta}}{m}}, \quad \processradM(\Xtil;u) \leq L'
k\sqrt{\frac{d\log m \log \frac{k}{\delta}}{m}}\\ & \qquad \qquad
\mbox{ and } \processradN(\Xtil;u) \leq c L'' \sqrt{\frac{d \log
    \frac{1}{\delta}}{m}} \Biggr\}.
\end{align*}
Then we have $\mprob \big[ \mathcal{E} \big] \geq 1 - \delta \E^{-c'
  \; d}$ for a universal constant $c' > 0$.
\end{lems}

In conjunction, Lemmas~\ref{LemConcLipproc} and \ref{LemContractions}
imply that conditionally on the event $\mathcal{E}$, we have
\begin{align*}
\EE_{\epsilon} \big[\processradobs(\Xtil;u)\big] & \leq
c\norm{\trueparamobs} ( \frac{\norm{\trueparamobs}^2}{\sigma^2} + 1)
k\sqrt{\frac{d \log m \log \frac{k}{\delta}}{m}}.
\end{align*}
Combining this bound with Lemma~\ref{LemEventBound} yields
\begin{align*}
\COVNUM \; \mprob_{\Xtil,\epsilon} \big[ \processobs(\Xtil;u)\geq t
  \big] & \leq \COVNUM \; \mprob_{X,\epsilon} \big[
  \processradobs(\Xtil;u) \geq \frac{t}{4} \mid \mathcal{E} \big] +
\COVNUM \: \mprob \big[ \mathcal{E}^c \big] \\
& \leq 2 \E^{4d + \log k-c k^2d \log m \log\frac{k}{\delta}} + \E^{4d-
  \tilde{c}d \log \frac{k}{\delta}} \\
& \leq \delta,
\end{align*}
where the final inequality follows by setting $t/4 = c
\norm{\trueparamobs}( \frac{\norm{\trueparamobs}^2}{\sigma^2}+ 1)
k\log (\frac{k}{\delta})\sqrt{\frac{d \log m}{m}}$.  After rescaling
$\delta$ by $8 k$ and setting $m = \frac{\numobs}{4k}$, the result
follows after an application of Lemma~\ref{lem:IB}.


\myparagraph{Proof of inequality~\eqref{eq:epstrans}}

In order to bound $|\emopsamptrunctrans{\paramjoint} -
\emoppoptrunctrans{\paramjoint}|$, we need a few extra steps. 
First, let us define new weights 
\begin{equation*}
\weightdouble{\paramjoint}{X_{i-k}^{i+k-1}} = p(Z_0 = Z_1 = 1 \mid
X_{i-k}^{i+k-1}, \paramjoint) + p(Z_0 = Z_1 = -1 \mid X_{i-k}^{i+k-1},
\paramjoint),
\end{equation*}
and also write the update in the form
\begin{align*}
M_{\probpar,\subsize}^\kdim(\paramjoint) &= \argmax_{\probpar \in
  \paramspacep} \Big \{ \EEzcondx{Z_1}{X_{i-k}^{i+k}}{\paramjoint}
\log p(Z_1 \mid \probpar) + \sum_{i=2}^{\subsize}
\EEzcondx{Z_{i-1}^i}{X_{i-k}^{i+k}}{\paramjoint} \log p(Z_i \mid
Z_{i-1}, \probpar) \Big \} \\
&= \argmax_{\probpar \in \paramspacep} \Big \{ \frac{1}{2} +
\sum_{i=2}^{\subsize} \EEzcondx{Z_{i-1}^i}{X_{i-k}^{i+k}}{\paramjoint}
\log p(Z_i \mid Z_{i-1}, \probpar) \Big \} \\
& = \Pi_{\Omega_{\probpar}} \Big( \frac{1}{\subsize}
\sum_{i=2}^{\subsize} \weightdouble{\paramjoint}{X_{i-k}^{i+k-1}}
\Big),
\end{align*} 
where we have reparameterized the transition probabilities with
$\probpar$ via the equivalences \mbox{ $\paramtrans = h(\probpar)
  \defn \frac{1}{2} \log \left( \frac{\probpar}{1-\probpar} \right)$.}
Note that the original EM operator is obtained via the transformation
$\emopsamptrunctrans{\paramjoint'} =
h(M_{\probpar,\subsize}^\kdim(\paramjoint'))$ and we have
\mbox{$\MBAR_{\probpar,\subsize}^\kdim(\paramjoint) =
  \Pi_{\paramspacep} \EE \weightdouble{\paramjoint}{X_{i-k}^{i+k-1}}$}
by definition.

Given this set-up, we can now pursue an argument similar to that of
inequality~\eqref{eq:epsobs}.  The new weights remain Lipschitz with
the same constant---that is, we have the bound
\mbox{$|\weightdouble{\paramjoint}{\blockXi} -
  \weightdouble{\paramjoint'}{\blockXi}| \leq \lipcont
  \norm{\paramjointgamma_i - \paramjointgamma_i'}$.}  As a
consequence, we can write
\begin{align*}
\mprob \big[ \sup_{\paramjoint \in \paramspacejoint} \mid
  \frac{1}{m}\sum_{i=1}^m\weightdouble{\paramjoint}{\blockXi} - \EE
  \weightdouble{\paramjoint}{\blockXi}| \geq \epsilontrans_{\numobs}
  \big] & \leq \frac{\delta}{8\kdim},
\end{align*}
with $\epsilontrans_{\numobs}$ defined as in the lemma statement.  In
this case, it is not necessary to perform the covering step, nor to
introduce extra Rademacher variables after the Gaussian comparison
step; consequently, the two constants $\epsilontrans_\numobs$ and
$\epsilonobs_{\numobs}$ differ by a factor of $\sqrt{d \log n}$ modulo
constants.  

Applying Lemma~\ref{lem:IB} then yields a tail bound for the quantity
$\big| \frac{1}{\subsize}\sum_{i=1}^{\subsize}
\weightdouble{\paramjoint}{\blockXi} - \EE
\weightdouble{\paramjoint}{\blockXi} \big|$ with probability $\delta$.
Since projection onto a convex set only decreases the distance, we
find that
\begin{equation*}
\mprob \Biggr[ | M_{\probpar,\subsize}^\kdim(\paramjoint) -
  \MBAR_{\probpar,\subsize}^\kdim(\paramjoint) |\geq C
  \frac{\sqrt{\norm{\trueparamobs}^2+\sigma^2}}{\sigma^2}
  \sqrt{\frac{\kdim^3 \log (k^2/\delta)}{\subsize} } \Biggr] \leq
\delta.
\end{equation*}
In order to prove the result, the last step needed is the fact that
\begin{equation*}
\frac{1}{2} \Big| \log \frac{x}{1-x} - \log \frac{y}{1-y}\Big| \leq
\frac{1}{\tilde{x} (1-\tilde{x})}|x - y| \leq
\frac{2}{1-\mixcoefbound^2} |x-y| =: L|x-y|
\end{equation*}
for $x,y,\tilde{x} \in \paramspacep$.  Since
$M_{\probpar,\subsize}^\kdim(\paramjoint) \in \paramspacep$ we finally
arrive at
\begin{align*}
\mprob \Big[ |\emopsamptrunctrans{\paramjoint} -
  \emoppoptrunctrans{\paramjoint}| \geq C (1-\mixcoefbound^2)
  \frac{\sqrt{\norm{\trueparamobs}^2+\sigma^2}}{\sigma^2} \sqrt{
    \frac{\kdim^3 \log \left(\frac{k^2}{\delta}\right)}{\subsize}}
  \Big] \leq \delta
\end{align*}
and the proof is complete.


\subsection{Proof of Lemma~\ref{lem:first_tech}}
\label{sec:proof_first_tech}

Since $\norm{\paramjoint} \leq \addnorm{\paramjoint} \leq \sqrt{2}
\norm{\paramjoint}$, it is sufficient to show that
\begin{align*}
\sup_{\theta \in \Ball_2(\rad; \thetastar)} \mprob \Big[
  \|\emopsampn{\subsize}{\theta} -
  \emopsamptruncn{\subsize}{\theta}\|^2_2 \geq c_1
  \xi(\subsize,\subprob) \BOUNDFUN^2(\kdim)
  \Big] & \leq \subprob
\end{align*}
with $\BOUNDFUN^2 (k) \defn \frac{c_o
  \nstates^5}{ \lambda \mixcoefeps^8 \statmin^2} (1- \mixcoefeps
\statmin)^\kdim$
We first claim that 
\begin{align}
\label{EqnDetour} 
\|\emopsampn{\subsize}{\theta} -
\emopsamptruncn{\subsize}{\theta}\|^2_2 \leq \frac{2
  \qnorm{\qfunn{\subsize} - \qfunsampextendnk{\subsize}{k}}}{\lambda},
\quad \mbox{where} \quad \lambda \defn 4 (\exp(\paramtransbound) +
\exp(-\paramtransbound)\big)^{-2}.
\end{align}
In Section~\ref{sec:concavity}. we showed that population operators
are strongly concave with parameter at least $\lambda$.  We make the
added observation that using our parameterization, the sample $Q$
functions $\qfunsamptrunc{\cdot}{\paramjoint'},
\qfunsamp{\cdot}{\paramjoint'}$ are also strongly
concave. This is because the concavity results for the
population operators  did not use any property of the covariates
in the HMM, in particular not the expectation operator,
and the single term $\frac{1}{n} \EE \sum_{z_0}
p(z_0 \mid X_1^\numobs,\paramtrans') \log p(z_0 ; \paramtrans) =
\frac{1}{n} \log \frac{1}{2}$ is constant for all $\paramtrans \in
\paramspacetrans$. 
From this $\lambda$-strong concavity, the
bound~\eqref{EqnDetour} follows immediately.

Given the bound~\eqref{EqnDetour}, the remainder of the proof focuses
on bounding the difference $\qnorm{\qfunn{\subsize} -
  \qfunsampextendnk{\subsize}{k}}$. Recalling the shorthand notation
\begin{align*}
\HACKG(X_i,z_i,\theta,\theta') = \log p(X_i| z_i, \theta) + \sum_{z_{i-1}}
p (z_i | z_{i-1}, \theta') \log p (z_i | z_{i-1}, \theta),
\end{align*}
we have
\begin{align}
\label{EqnQsampApprox}
\qnorm{\qfunn{\subsize} - \qfunsampextendnk{\subsize}{k}} & = \Big |
\sup_{\paramjoint,\paramjointtwo \in \DomTheta}\frac{1}{\subsize}
\sum_{i=1}^{\subsize} \sum_{z_i}
(p(z_i|X_{1}^{\subsize},\paramjointtwo) -
p(z_i|X_{i-k}^{i+k},\paramjointtwo))
\HACKG(X_i,z_i,\paramjointone,\paramjointtwo) \Big | \\ 
& \leq \sup_{\paramjointone,\paramjointtwo \in \DomTheta} \sup_{X}
\sum_{z_i} \big| p(z_i \mid X_{1}^{\subsize},\theta') - p(z_i \mid
X_{i-k}^{i+k},\paramjointtwo) \big| \; \left[ \frac{1}{\subsize}
  \sum_{i=1}^{\subsize} \max_{z_i \in [m]} \big|
  \HACKG(X_i,z_i,\paramjointone,\paramjointtwo) \big|
  \right]. \nonumber
\end{align}
It is easy to verify that $\max_{z_i \in [m]} \left|
\HACKG(X_i,z_i,\paramjointone,\paramjointtwo) \right| \leq \max_{z_i}
\mid \log p(X_i | z_i, \theta) \mid + 2 \log (\statmin^{-1})$, and
moreover, Lemma~\ref{lem:twosided_approx} implies that
\begin{align*}
\sup_{\paramjointone,\paramjointtwo \in \DomTheta} \sup_{X} \sum_{z_i}
\mid p(z_i \mid X_{1}^{\subsize},\theta') - p(z_i \mid
X_{i-k}^{i+k},\paramjointtwo)| \leq \frac{32 C\big(1-\mixcoefeps
  \statmin \big)^k}{\mixcoefeps^8}.
\end{align*}
Combining these inequalities yields the upper bound
\begin{align*}
\qnorm{\qfunn{\subsize} - \qfunsampextendnk{\subsize}{k}} & \leq
\frac{32 C \big(1-\mixcoefeps \statmin \big)^k}{\mixcoefeps^8} \left[
  \EE \max_{z_i \in [m]} \mid \log p(X_i|z_i,\paramjoint)| +
  \samperror{\subsize}(X) + 2 \log (\statmin^{-1}) \right]
\end{align*}
where 
\begin{align*}
\samperror{\subsize}(X) & \defn \Big |
\frac{1}{\subsize}\sum_{i=1}^{\subsize} \max_{z_i \in [m]} \Big | \log
p(X_i \mid z_i,\paramjoint) \Big | - \EE \max_{z_i \in [m]} \Big| \log
p(X_i \mid z_i,\paramjoint) \Big| \Big |.
\end{align*}
By assumption, we have that $ \EE \max_{z_i \in [m]} |\log p(X_i \mid
z_i,\paramjoint)|$ is bounded by an appropriately large universal
constant.  Putting these together, we find that
\begin{align*}
 \|\emopsampn{\subsize}{\theta} -
 \emopsamptruncn{\subsize}{\theta}\|^2_2 & \leq \frac{64 \plaincon
   \big(1-\mixcoefeps \statmin \big)^k}{\lambda \mixcoefeps^8} ( 1 +
 \samperror{\subsize}(X)).
\end{align*}
In order to complete the proof, it suffices to show that
\begin{align*}
\mprob \Big(\samperror{\subsize}(X) \geq c_0 \Big( \frac{1}{\sigma}
\sqrt{ \frac{d \log^2 (\consteps \subsize/\subprob)}{\subsize} } +
\frac{\norm{\mustar}}{\sigma} \sqrt{ \frac{\log
    ^2(\consteps\subsize/\subprob)}{\subsize}} +
\frac{\norm{\mustar}^2}{\sigma^2} \Big)\Big) \leq \subprob,
\end{align*}
where $c_0$ is a universal constant and $\consteps =
\frac{36}{\mixcoefeps^3}$.

Observe that we have
\begin{align*}
\samperror{\subsize}(X) & = \frac{1}{2\subsize
  \sigma^2}\sum_{i=1}^{\subsize} \left[ \max \{ \| X_i +
  \paramobs\|_2^2, \|X_i -\paramobs\|^2_2\} - \EE \max \{\|X_i +
  \paramobs\|_2^2, \|X_i -\paramobs\|^2_2\} \right] \\ 
& = \frac{1}{2 \subsize \sigma^2} \sum_{i=1}^{\subsize} \left(
\|X_i\|_2^2 - \EE \|X_i\|_2^2\right) + \frac{1}{\subsize \sigma^2}
\sum_{i=1}^{\subsize} \left( | X_i^T \paramobs | - \EE |X_i^T
\paramobs| \right).
\end{align*}
Note that we are again dealing with a dependent sequence so that we
cannot use usual Hoeffding type bounds.  For some $\tk$ to be chosen
later on, and $\blocksize = \subsize/\tk$ using the proof idea of
Lemma~\ref{lem:IB} with $f_2(X_i) = |X_i^T\paramobs|$ and $f_2(X_i) =
\| X_i\|_2^2$, we can write
\begin{align*}
\mprob(\samperror{\numobs}(X) \geq \frac{t}{2\sigma^2}) \leq \tk \Big(
\underbrace{\mprob_0 \Big( | \frac{1}{\blocksize}\sum_{i=1}^\blocksize
  \|X_i\|_2^2 - \EE \|X_i\|_2^2 | \geq \frac{t}{2} \Big)}_{T_1} +
\underbrace{\mprob_0 \Big( | \frac{1}{\blocksize}
  \sum_{i=1}^\blocksize |X_i^T\paramobs| - \EE |X_i^T\paramobs| | \geq
  \frac{t}{4} \Big)}_{T_2} + \blocksize \beta(\tk) \Big),
\end{align*}
where $\beta(\tk)$ was previously defined in
equation~\eqref{eqn:betadef}.  We claim that the choices
\begin{align*}
t \defn c_1 \Big( \sigma \sqrt{ \frac{d \log (\tk/\subprob)}{\blocksize} } +
\sigma \norm{\mustar} \sqrt{ \frac{\log (\tk/\subprob)}{\blocksize}} +
\norm{\mustar}^2 \Big), \quad \mbox{and} \quad
\tk \defn \frac{C_2 \log ( \frac{36 \subsize}{\mixcoefeps^3 \subprob})}{
  \log 1/(1-\mixcoefeps)} ,
\end{align*}
suffice to ensure that $\mprob(\samperror{\subsize}(X) \geq
t/(2\sigma^2)) \leq \subprob$. Given this claim, the lemma
follows immediately since $\frac{t}{2\sigma^2} \geq \xi (\subsize, \subprob)$
for big enough constant $c_0$.

Accordingly, it remains to verify the sufficiency of the chosen $t$
and $\tk$.  The bound~\eqref{eqn:betabound} implies that
\begin{align*}
\blocksize \beta(\tk) \leq \frac{12 \blocksize
  \mixcoef^{\tk}}{\mixcoefeps^3} \leq \frac{\subprob}{3\tk}.
\end{align*}
\noindent In the sequel we develop bounds on $T_1$ and $T_2$.  For
$T_1$, observe that since $X_i \sim Z_i \mu^* + \epsilon_i$
where $\epsilon_i$ is a Gaussian vector with covariance $\sigma^2 I$ and
$Z_i$ independent under $\mprob_0$,
standard $\chi^2$ tail bounds imply that
\begin{align*}
\mprob_0 \Big[ | \frac{1}{\blocksize}\sum_{i=1}^\blocksize \|X_i\|_2^2
  - \EE \|X_i\|_2^2 | \geq \frac{t}{2} \Big] & \leq
\frac{\subprob}{3\tk}.
\end{align*}
Finally, we turn our attention to the term $T_2$. Observe that,
\begin{align*}
X_i^T \mu \sim \frac{1}{2} \NORMAL( \mu^T \mu^*, \sigma^2
\|\paramobs\|_2^2) + \frac{1}{2} \NORMAL( -\mu^T \mu^*, \sigma^2
\|\paramobs\|_2^2),
\end{align*}
so that $ \mid X_i^T \mu \sim |\NORMAL( \mu^T \mu^*, \sigma^2
\|\paramobs\|_2^2) |$.  Denote $U_i = \mid X_i^T \mu \mid$. Letting
$\epsilon$ denote a Rademacher random variable, observe that
\begin{align*}
\Exs \exp (t U_i) \stackrel{\mathrm{(i)}}{\leq} \mathbb{E} \exp (2 t
\epsilon U_i) \stackrel{\mathrm{(ii)}}{\leq} \exp \big( 2 t^2 \sigma^2
\|\paramobs\|_2^2 + 2 t \mu^T \mu^* \big),
\end{align*}
where (i) follows using symmetrization, and (ii) follows since the
random variable $\epsilon U_i$ is a Gaussian mixture.  Observe that
\begin{align*}
\mathbb{E} U_i \stackrel{\mathrm{(iii)}}{\leq} |\mu^T \mu^*| + \sigma
\|\paramobs\|_2 \stackrel{\mathrm{(iv)}}{\leq} \underbrace{2 ( \sigma
  + \norm{\mustar}) \norm{\mustar}}_{M},
\end{align*}
where we have used for (iii) that $U_i$ is a folded normal, and for
(iv) that $\|\mu - \mustar\|_2 \leq \frac{\|\mustar\|_2}{4}$.  Setting
$D \defn 4 \sigma \norm{\mustar} \sqrt{ \frac{\log
    (6\tk/\subprob)}{\blocksize}}$ observe that $\frac{t}{4} \geq 2 M
+ D$ for big enough $c_1$.  Thus, applying the Chernoff bound yields
\begin{align*}
T_2 \leq \mprob_0 \Big[ |\frac{1}{\blocksize} \sum_{i=1}^\blocksize
  U_i - \mathbb{E} U_i| \geq 2M + D \Big] & \leq \mprob_0 \Big( \mid
\frac{1}{\blocksize} \sum_{i=1}^\blocksize U_i \mid \geq M+ D \Big) \\
& \leq 2 \inf_{t \geq 0} \Big\{ \mathbb{E} \exp \Big(
\frac{t}{\blocksize} \sum_{i=1}^\blocksize U_i - Mt - Dt \Big) \Big\},
\\
& \leq 2 \exp \Big( - \frac{\blocksize D^2}{ 8 \sigma^2
  \norm{\mu}^2}\Big) \leq \frac{\subprob}{3\tk}.
\end{align*}
By combining the bounds on $T_1$ and $T_2$, some algebra shows that
our choices of $t,\tk$ yield the claimed bound---namely, that $\mprob
\big[ \samperror{\subsize}(X) \geq t/(2\sigma^2) \big] \leq \subprob$.


\subsection{Proofs of technical lemmas}
\label{SecNewTechnical}

In this section, we collect the proofs of various technical lemmas
cited in the previous sections.

\subsubsection{Proof of Lemma~\ref{LemConcLipproc}}

We use the following concentration theorem (e.g.,~\cite{Led96}):
suppose that the function $f:\RN^n \to \RN$ is coordinate-wise convex
and $L$-Lipschitz with respect to the Euclidean norm.  Then for any
i.i.d. sequence of variables $\{X_i\}_{i=1}^\numobs$ taking values in
the interval $[a,b]$, we have
\begin{align}
\mprob \big[f(X) \geq \EE f(X) + \delta \big] & \leq
\E^{-\frac{\delta^2}{4L^2(b-a)^2}}
\end{align}

We consider the process without absolute values
(which introduces the factor of two in the lemma) and see that
$\epsilon \defn (\epsilon_1, \dots, \epsilon_{n})$ is a random vector
with bounded entries and that the supremum over affine functions is
convex. \\
It remains to show that the function $\epsilon \mapsto
\processradobs(\Xtil,u)$ is Lipschitz with $\lipproc(\Xtil;u)$ as follows
\begin{align*}
\big| \sup_{\paramjoint} \frac{1}{m}\sum_{i=1}^m \epsilon_i \funcproc(\blockXi) - \sup_{\paramjoint} \frac{1}{m}\sum_{i=1}^m \epsilon'_i \funcproc(\blockXi) \big|  &\leq \frac{1}{m}|\sum_{i=1}^m (\epsilon_i -\epsilon'_i)\funcproc(\blockXi)| \\
&\leq \frac{1}{m}\sqrt{\sum_{i=1}^m (2\weightsingle{\paramjoint}{\blockXi}-1)^2 \langle X_{i,k},u\rangle^2}  \norm{\epsilon - \epsilon'}\\
&\leq \lipproc(\Xtil;u) \norm{\epsilon -\epsilon'}
\end{align*}
where $\paramjoint = \argmax_{\paramjoint \in \paramspacejoint} \sum_i
\epsilon_i \funcproc(\blockXi)$ in the last line and we use that
$|2\weightsingle{\paramjoint}{\blockXi}-1| \leq 1$.


\subsubsection{Proof of Lemma~\ref{LemContractions}}

The proof consists of three steps. First, we observe that the
Rademacher complexity is upper bounded by the Gaussian
complexity. Then we use Gaussian comparison inequalities to reduce the
process to a simpler one, followed by a final step to convert it back
to a Rademacher process.

\paragraph{Relating the Gaussian and Rademacher complexity:}

Let $g_i \sim \NORMAL(0,1)$. It is easy to see that using Jensen's
inequality and the fact that $\epsilon_i |g_i| \overset{d}{=} g_i$
\begin{align*}
\EE_{\epsilon} \sup_{\paramjoint} \frac{1}{m} \sum_{i=1}^m \epsilon_i
\funcproc(\blockXi) &= \sqrt{\frac{2}{\pi}} \EE_{\epsilon}
\sup_{\paramjoint} \frac{1}{m} \sum_{i=1}^m\EE_g[|g_i|]
\funcproc(\blockXi)\\ 
& \leq \sqrt{\frac{2}{\pi}} \EE_{g} \sup_{\paramjoint}\frac{1}{m}
\sum_{i=1}^m g_i \funcproc(\blockXi).
\end{align*}

\paragraph{Lipschitz continuity:}
Define
\begin{equation}
\paramjointgamma_i \defn (\paramgamma_i, \paramtrans) = (\frac{\langle
  \paramobs, X_{i,1} \rangle}{\sigma^2}, \dots,
\frac{\langle\paramobs, X_{i,2k}\rangle}{\sigma^2}, \paramtrans),
\end{equation}
Now we can use results in the proof of Corollary 1 to see that
$\paramjointgamma \mapsto \funcprocgamma{\paramjoint} \defn
\funcproc(\blockXi)$ is Lipschitz in the Euclidean norm, i.e. there
exists an $\lipcont$, only dependent on $\mixcoef$ such that
\begin{align}
\label{EqnLipschitz}
|\funcprocgamma{\paramjointgamma_i} -
\funcprocgamma{\paramjointgamma'_i}| \leq \lipcont
\norm{\paramjointgamma_i- \paramjointgamma'_i} |\langle X_{i,k}
,u\rangle|
\end{align}

For this we directly use results (exponential family representation)
that were used to show Corollary 1. We overload notation and write
$X_\ell \defn X_{1,\ell}$ and analyze Lipschitz continuity for the
first block. First note that $F(\paramjoint, X_{0;2k}) =
\EEzcondx{Z_k}{X_{1}^{2k}}{\paramjoint} Z_k$.  By Taylor's theorem, we
then have
\begin{align*}
|\funcprocgamma{\paramjointgamma_i} -
\funcprocgamma{\paramjointgamma'_i}| &= |\langle X_{i,k},u \rangle|
|\EEzcondx{Z_k}{\blockXi}{\paramjoint} Z_k -
\EEzcondx{Z_k}{\blockXi}{\paramjoint'} Z_k|\\ &\leq |\langle X_{i,k},u
\rangle| |\EEzcondx{Z_k}{\blockXi}{(\paramobs,\paramtrans)} Z_k -
\EEzcondx{Z_k}{\blockXi}{(\paramobs',\paramtrans)} Z_k| \\ &+ |\langle
X_{i,k},u \rangle| |\EEzcondx{Z_k}{\blockXi}{(\paramobs',\paramtrans)}
Z_k - \EEzcondx{Z_k}{\blockXi}{(\paramobs',\paramtrans')} Z_k|
\end{align*}
Let us examine  each of the summands separately.  By the
Cauchy-Schwartz inequality and Lemma~\ref{lem:expfammixing0}, we have
\begin{align*}
|\EEzcondx{Z_k}{\blockXi}{(\paramobs,\paramtrans)} Z_k -
\EEzcondx{Z_k}{\blockXi}{(\paramobs',\paramtrans)} Z_k| &=
\frac{1}{\sigma}|\sum_{\ell =1}^{2k} \frac{\partial^2 \Phi}{\partial
  \paramgamma_{\ell} \partial \paramgamma_0}\Big|_{\paramjoint =
  \tilde{\paramjoint}} (\paramgamma_{\ell} -
\paramgamma'_{\ell})|\\ 
&= \big| \sum_{\ell=1}^{2k}
\condcov{Z_0}{Z_\ell}{\blockXi,\tilde{\paramjoint}} (\langle
\paramobs,X_{\ell}\rangle - \langle \paramobs', X_{\ell}\rangle)
\big|\\
&\leq \sqrt{(\sum_{\ell=1}^{2k} 4 \mixcoef^{2\ell}) \sum_{\ell=1}^{2k}
(\paramgamma_\ell - \paramgamma_\ell')^2},
\end{align*}
as well as
\begin{align*}
|\EEzcondx{Z_k}{\blockXi}{(\paramobs',\paramtrans)} Z_k -
\EEzcondx{Z_k}{\blockXi}{(\paramobs',\paramtrans')} Z_k| &=
|\sum_{\ell =1}^{2k} \frac{\partial^2 \Phi}{\partial
  \paramtrans_{\ell} \partial\paramgamma_0}\Big|_{\paramjoint =
  \tilde{\paramjoint}} (\paramtrans - \paramtrans')|\\ &= \big|
\sum_{\ell=1}^{2k} \condcov{Z_0}{Z_\ell
  Z_{\ell+1}}{\blockXi,\tilde{\paramjoint}} (\paramtrans-\paramtrans')
\big|\\ &\leq \frac{2}{1-\mixcoef} |\paramtrans-\paramtrans'|.
\end{align*}

Combining these two bounds yields
\begin{equation*}
|\funcprocgamma{\paramjointgamma_i} -
\funcprocgamma{\paramjointgamma'_i}|^2 \leq \langle X_{i,k},u
\rangle^2 \lipcont \big(\sum_{\ell=1}^{2k} (\paramgamma_\ell -
\paramgamma_\ell')^2 + (\paramtrans-\paramtrans')^2\big) = \langle
X_{i,k},u\rangle^2 \lipcont^2 \norm{\paramjointgamma_i -
  \paramjointgamma'_i}^2
\end{equation*}
with $ \lipcont^2 = \frac{8}{(1-\mixcoef)^2}$.


\paragraph{Applying the Sudakov-Fernique Gaussian comparison:} 

Let us introduce the shorthands \mbox{$X_\paramjoint = \frac{1}{m}
  \sum_i g_i \funcproc(\blockXi)$,} and
\begin{equation*}
Y_{\paramjoint} = \frac{1}{m} \lipcont \sum_i \big(\sum_{\ell=1}^{2k}
g_{i \ell} \frac{\langle \paramobs, X_{i,\ell} \rangle}{\sigma^2} +
g_{i,2k+1} \paramtrans \big) \langle X_{i,k}, u\rangle.
\end{equation*}
By construction, the random variable $X_\paramjoint - X_\paramjoint'$
is a zero-mean Gaussian variable with variance
\begin{align}
\EE_g (X_{\paramjoint}-X_{\paramjoint'})^2 &= \sum_i
(\funcprocgamma{\paramjointgamma} -
\funcprocgamma{\paramjointgamma'})^2 \nonumber\\ &\leq \lipcont^2
\sum_i \langle X_{i,k},u\rangle^2 \big( \sum_{\ell=1}^{2k}
(\paramgamma_{i,\ell} - \paramgamma'_{i,\ell})^2 +
(\paramtrans-\paramtrans')^2 \big) \nonumber\\ &= \EE_g
(Y_{\paramjoint} - Y_{\paramjoint'})^2
\end{align}
By the Sudakov-Fernique comparison~\cite{LedTalBanach},we are then
guaranteed that $\EE \sup_{\paramjoint} X_{\paramjoint} \leq \EE
\sup_{\paramjoint} Y_{\paramjoint}$.  Therefore, it is sufficient to
bound
\begin{align*}
\EE_g\sup_{\paramjoint\in\paramspacejoint} Y_{\paramjoint}=
\underbrace{\EE_g \sup_{\paramjoint} \frac{\lipcont}{\sigma^2 n}
  \sum_{i=1}^m \sum_{\ell=1}^{2k} g_{i \ell} \langle \paramobs,
  X_{i,\ell} \rangle\langle X_{i,k}, u\rangle}_{T_1} +
\underbrace{\EE_g \sup_{\paramjoint} \frac{\lipcont}{n}\sum_{i=1}^m
  g_{i,2k+1} \paramtrans \langle X_{i,k}, u\rangle}_{T_2}.
\end{align*}


\paragraph{Converting back to a Rademacher process:}

We now convert the term $T_1$ back to a Rademacher process, which
allows us to use sub-exponential tail bounds. We do so by
re-introducing additional Rademacher variables, and then removing the
term $\max_i|g_i|$ via the Ledoux-Talagrand contraction
theorem~\cite{LedTalBanach}.  Given a Rademacher variable
$\epsilon_{il}$ independent of $g$, note the distributional
equivalence $\epsilon_{il} g_{il} \overset{d}{=} g_{i\ell}$.  Then
consider the function $\phi_{i\ell} (g_{i\ell}) \defn g_{i\ell}
h_{i\ell}$ with $h_{i\ell} \defn \langle \paramobs,X_{i\ell}\rangle
\langle X_{i,k},u \rangle $ for which it is easy to see that
\begin{equation}
|\phi_{i\ell}(g_{i\ell},h_{i\ell}) - \phi_{i\ell}(g_{i\ell},
h_{i\ell}')| \leq |g_{i\ell}| |h_{i\ell}-h_{i\ell}'|
\end{equation}
Applying Theorem 4.12. in Ledoux and Talgrand~\cite{LedTalBanach}
yields
\begin{align*}
\EE \sup_{\paramjoint} \frac{1}{m}\sum_{i=1}^m \sum_{\ell=1}^{2k}
\epsilon_{i\ell} g_{i\ell} \langle \paramobs, X_{i\ell}\rangle \langle
X_{i,k},u\rangle &\leq \EE_g \|g\|_{\infty} \EE_\epsilon
\sup_{\paramjoint} \frac{1}{m}\sum_{i=1}^m \sum_{\ell=1}^{2k}
\epsilon_{i\ell} \langle \paramobs, X_{i,\ell}\rangle \langle
X_{i,k},u\rangle.
\end{align*}
Putting together the pieces yields the claim~\eqref{EqnExp}.


\subsubsection{Proof of Lemma~\ref{LemEventBound}}

We prove that the probability of each of the events corresponding to
the inequalities is smaller than $\frac{\delta}{3}\E^{-\tilde{c}d}$.

\paragraph{Bounding $\lipproc$:}

We start by bounding $\lipproc(\Xtil;u)$, for which we note that
\begin{align*}
\sum_{i=1}^{n} \langle X_{i,k},u\rangle^2 &\leq \sum_{i=1}^m
\norm{\trueparamobs}^2 + \sum_{i=1}^m \langle n_{i,k},u\rangle^2 +
\sum_{i=1}^m \langle \trueparamobs, u\rangle \langle n_{i,k},u\rangle
\end{align*}
implies that $\lipproc(\Xtil;u) \leq \sqrt{\frac{2
    (\norm{\trueparamobs}^2 + \sigma^2) \log \frac{1}{\delta}}{m}}$
with probability at least $1 - \frac{\delta}{3}\E^{-\tilde{c}d}$.

\paragraph{Bounding $\processradN$:}

In order to bound $\processradN(\Xtil;u)$, we first introduce an extra
Rademacher random variable into its definition; doing so does not
change its value (now defined by an expectation over both $g$ and the
Rademacher variables). We now require a result for a product of the
form $\epsilon g h$ where $g,h$ are independent Gaussian random
variables.
\begin{lems}
\label{LemProductGauss}
Let $(\epsilon, g, h)$ be independent random variables, with
$\epsilon$ Rademacher, \mbox{ $g \sim \NORMAL(0, \sigma_g^2)$,} and
\mbox{$h \sim \NORMAL(0, \sigma_h^2)$.}  Then the random variable $Z =
\epsilon g h$ is a zero-mean sub-exponential random variable with
parameters $(\frac{\sigma_g^2\sigma_h^2}{2}, \frac{1}{4})$.
\end{lems}

\begin{proof}
Note that $g' = \alpha h$ with $\alpha = \frac{\sigma_g}{\sigma_h}$ is
identically distributed as $g$. Therefore, we have
\begin{align*}
g h = \frac{1}{\alpha} g g' = \frac{1}{4 \alpha}[(g-g')^2+ (g+g')^2]
\end{align*}
The random variables $g-g'$ and $g+g'$ are independent and therefore
$(g-g')^2, (g+g')^2$ are sub-exponential with parameters $\nu^2 =
4\sigma_g^4$, $b=\frac{1}{4}$. This directly yields
\begin{equation*}
\EE \E^{\lambda \epsilon [(g+g')^2 - (g-g')^2]} \leq \E^{4 \lambda^2 \sigma_g^4} 
\end{equation*}
for $|\lambda|\leq \frac{1}{b}$. Therefore $\EE \E^{\lambda \epsilon
  gh} \leq \E^{\frac{\lambda^2 \sigma_g^2 \sigma_h^2}{4}}$, which
shows that $\epsilon g h$ is sub-exponential with parameters
$(\frac{\sigma_g^2\sigma_h^2}{2}, \frac{1}{4})$.
\end{proof}

Returning to the random variable $\processradN(\Xtil;u)$, each term
$\epsilon_i g_{i,2k+1} \langle X_{i,k},u\rangle$ is a sub-exponential
random variable with mean zero and parameter $\nu^2 =
\norm{\trueparamobs}^2 + \frac{\sigma^2}{2}$.  Consequently, there are
universal constants such that $\processradN(\Xtil;u) \leq c
\lipcont\nu\sqrt{\frac{d \log \frac{1}{\delta}}{m}}$ with probability
at least $1- \frac{\delta}{3}\E^{-\tilde{c}d}$.

\paragraph{Bounding $\processradM$:}
Our next claim is that with probability at least $\frac{\delta}{3}$,
we have
\begin{equation}
\label{EqnTerm4}
\EE_{\epsilon} \norm{\frac{1}{m}\sum_{i=1}^m \epsilon_{i\ell}
  X_{i,\ell} \langle X_{i,k},u\rangle} \leq (\norm{\trueparamobs}^2 +
\sigma^2)\sqrt{\frac{ d \log \frac{k}{\delta}}{m}},
\end{equation}
which then implies that $\processradM(\Xtil;u) \leq c
\norm{\trueparamobs}\big(\frac{\norm{\trueparamobs}}{\sigma^2}+1\big)
k \sqrt{\frac{ d \log m \log \frac{k}{\delta}}{n}}$.  In order to
establish this claim, we first observe that by
Lemma~\ref{LemProductGauss}, the random variable $\epsilon_\ell
\langle X_{i,\ell}, u\rangle\langle X_{i,k}, u\rangle$ is zero mean,
sub-exponential with parameter at most $\nu^2 =
(\norm{\trueparamobs}^2 + \sigma^2)^2$.  The bound then follows by the
same argument used to bound the quantity $\processradN$.




\section{Mixing related results}

In the following we will use the shorthand notations
$w_\theta(x_{-k}^k) \defn p(z_0= -1 \mid x_{-k}^{k}, \theta)$ and
$\pi_k^\theta \defn p(z_0 \mid x_{-k}^{0},\theta)$ which we refer to
as weights and the filtering distribution respectively.

Introducing the shorthand notation $p_{\mu}(x_k) \defn \sum_{z_k}
\sum_{z_{k-1}} p(x_k \mid z_k)p(z_k \mid z_{k-1}) \mu (z_{k-1})$, we
define the filter operator
\begin{align*}
F_i \nu (z_i) & \defn \frac{ \sum_{z_{i-1}} p(x_i \mid z_i)p(z_i \mid
  z_{i-1}) \nu(z_{i-1})}{\sum_{z_i} \sum_{z_{i-1}} p(x_i \mid
  z_i)p(z_i \mid z_{i-1}) \nu(z_{i-1})} \; = \; \sum_{z_{i-1}}
\frac{p(x_i \mid z_i)p(z_i \mid z_{i-1})}{p_\nu(x_i)} \nu(z_{i-1}).
\end{align*}
where the observations $x$ are fix. 
Using this notation, the filtering distribution can then be rewritten
in the more compact form \mbox{$\pi_k^\theta = p(z_0 \mid
  x_{-k}^0,\theta)= F_k\dots F_1 \mu$.}  Similarly, we define
\begin{align*}
K_{j \mid i}(z_j \mid z_{j-1}) & \defn \frac{p(z_j \mid z_{j-1}) p(x_j
  \mid z_j) p(x_{j+1}^{i} \mid z_j)}{\sum_{z_j}p(z_j \mid z_{j-1}) p(x_j
  \mid z_j) p(x_{j+1}^{i}|z_j) }, \quad \mbox{and} \quad \nu_{\ell \mid
  i} \, \defn \, \frac{p(x_{\ell+1}^{i} \mid z_{\ell})
  \nu(z_{\ell})}{\sum_{z_{\ell}} p(x_{\ell+1}^{i} \mid z_{\ell})
  \nu(z_{\ell})}
\end{align*}
Note that $\mixcoefeps C_0 \leq p(x_{\ell+1}^i \mid z_{\ell}) \leq
\mixcoefeps^{-1} C_0$ where
\begin{equation*}
C_0 = \sum_{z_{i} \dots z_{\ell+1}} p(x_{\ell+1} \mid z_{\ell+1})
\dots p(x_i \mid z_i) p(z_i \mid z_{i-1}) \dots p(z_{\ell+2} \mid
z_{\ell+1}) \pi(z_{\ell+1})
\end{equation*} 
and therefore
\begin{equation}
\label{eq:obsratio}
\sup_x\frac{\sup_{z} p(x_{\ell+1}^{i} \mid z_{\ell})}{\inf_z
  p(x_{\ell+1}^{i} \mid z_{\ell})} \leq \mixcoefeps^{-2}.
\end{equation}
With these definitions, it can be verified (e.g., see Chapter 5 of van
Handel~\cite{vanHandel_HMM}) that $F_i\dots F_{\ell+1}\nu =
\nu_{\ell +1 \mid i}^T K_{\ell + 1 \mid i} \dots K_{i \mid i}$, where
$\nu^T K \defn \int \nu(x') K(x|x') \d x'$.  (In the discrete setting,
this relation can be written as the row vector $\nu$ being right
multiplied by the kernel matrix $K$.)

We also observe that
\begin{align*}
\pi^\theta_k - \pi^{\theta'}_k &= F^\theta_k \cdots F^\theta_1
\mu^\theta - F^{\theta'}_k \cdots F^{\theta'}_1 \mu^{\theta'} \;\\
& = \;
\sum_{i=0}^{k-1} F_k^{\theta}\dots F_{i+1}^{\theta} (F^{\theta}_i
\pi_{i-1}^{\theta'} - F_i^{\theta'} \pi_{i-1}^{\theta'}) +
F_k^\theta\pi_{k-1}^{\theta'} - F_k^{\theta'}\pi_{k-1}^{\theta'},
\end{align*}
where $F^\theta_0 \pi_0^{\theta'} \defn \mu^\theta$ and $F^{\theta'}_0
\pi_0^{\theta'} \defn \mu^{\theta'}$. For simplicity, we write $F_k
\defn F_k^{\theta}$ and $F_k' \defn F_k^{\theta'}$.


\subsection{Consequences of mixing}
In this technical appendix we derive several useful consequences of
the geometric mixing condition on the stochastic process $Z_i$.

\begin{lems}
\label{lem:dependonfar0}
For any geometrically $\mixcoef$-mixing and time reversible Markov chain with
$\nstates$ states, there is a universal constant $c$ such that
\begin{subequations}
\begin{align}
\label{eqn:dependonfar0}
\sup_x |p(z_i \mid x_{i + k}) - p(z_i)| & \leq c \, \mixcoef^k, \quad \mbox{and} \\
\label{EqnDependOnFar}
\sup_{z_0} \big| p(z_0 \mid x_{\kdim}^{\numobs}) - p(z_0) \big | & <
\frac{c  (\nstates+1)}{\statmin^3\mixcoefeps^3} \mixcoef^\kdim.
\end{align}
\end{subequations}
\end{lems}
\begin{proof}
We first prove the bound~\eqref{eqn:dependonfar0}.
Using time reversibility and the defining of mixing we obtain
\begin{align*}
\max_x (p(z_0 \mid x_k) - \pi(z_0)) &= \sum_{z_k}
(p(z_0 \mid z_k)- \pistat(z_0)) p(z_k \mid x_k) \\
&\leq \sum_{z_k} p(z_k \mid x_k) \max_{z_k} | (p(z_0 \mid z_k) -
\pistat(z_0))|\\ 
&\leq \max_{z_k} \Big| \frac{p(z_k \mid z_0)
  \pistat(z_0)}{\pistat (z_k)} -
\frac{\pistat(z_0)\pistat(z_k)}{\pistat(z_k)} \Big|\\
&\leq \frac{\pistat(z_0)}{\pistat(z_k)} \max_{z_k}
|p(z_k \mid z_0) - \pistat(z_0)| \leq
\frac{1}{\statmin} \mixcoef^k
\end{align*}
where $p(z_k \mid z_0) = P(Z_k = z_k \mid Z_0 = z_0)$ and $p(z_0 \mid
z_k) = P(Z_0 = z_0 \mid Z_k = z_k)$.

Using this result we can now prove inequality~\eqref{EqnDependOnFar}.
By definition, we have
\begin{align*}
p(z_0) = \frac{p(x_{k+1}^n \mid x_k)p(x_k)p(z_0)}{p(x_{k+1}^n,x_k)}, \quad
\mbox{and} \quad p(z_0 \mid x_k,x_{k+1}^n) =
\frac{p(x_{k+1}^n \mid x_k,z_0)p(x_k \mid z_0)p(z_0)}{p(x_{k+1}^n,x_k)}
\end{align*}
and therefore
\begin{multline}
\label{eq:Hana1}
|p(z_0) - p(z_0 \mid x_k^n)| \leq \frac{ p(x_k) p(z_0)}{p(x_{k}^{n})} |p(x_{k+1}^n\mid x_k) - p(x_{k+1}^n \mid x_k, z_0)| \\
+ \frac{p(x_{k+1}^n \mid x_k, z_0) p(z_0)}{p(x_{k+1}^n \mid x_k)} |p(x_k) - p(x_k\mid z_0)|
\end{multline}

In the following we bound each of the two differences. 
Note that
\begin{align}
|p(x_{k+1}^n \mid x_{k},z_0) - p(x_{k+1}^n \mid x_k)| &=
\sum_{z_k}\sum_{z_{k+1}}
p(x_{k+1}^n \mid z_{k+1})p(z_{k+1} \mid z_k)|p(z_k \mid x_k,z_0) -
p(z_k \mid x_k)| \nonumber \\ 
&\leq \sup_{z_k,x_k} |p(z_k \mid x_k,z_0) - p(z_k \mid x_k)|
\sum_{z_k} p(x_{k+1}^n \mid z_k)  \label{eq:Hana2}
\end{align}
The last term $\sum_{z_k} p(x_{k+1}^n \mid z_k) $ 
is bounded by $\nstates$ for $\nstates$-state models.
Using the bound~\eqref{eqn:dependonfar0}, we obtain
\begin{align}
\label{eq:Hana3}
|p(x_k \mid z_0) - p(x_k)| &= \frac{|p(z_0 \mid
  x_k)- \pistat(z_0)|p(x_k)}{\pistat(z_0)} \leq \frac{p(x_k) }{\statmin^2}
\mixcoef^k
%
\end{align}
which yields
\begin{align} 
|p(z_k \mid x_k, z_0) - p(z_k \mid x_k)| &= p(x_k \mid z_k)\left|\frac{
  p(z_k \mid z_0)}{p(x_k \mid z_0)} -\frac{\pistat(z_k)}{p(x_k)}\right| \nonumber\\ 
&\leq
\frac{p(x_k \mid z_k)}{p(x_k \mid z_0)} |p(z_k \mid z_0)-\pistat(z_k)| +
\frac{\pistat(z_k)p(x_k \mid z_k) }{p(x_k)p(x_k \mid z_0)}|p(x_k \mid z_0) - p(x_k)| \nonumber\\
&\leq \frac{p(x_k \mid z_k) }{p(x_k \mid z_0)} \left( \mixcoef^k + \frac{1}{\statmin^2} \mixcoef^k \right) \nonumber \\
&\leq \frac{1}{p(z_k \mid z_0)} \frac{\nstates+1}{\statmin} \mixcoef^k 
\leq  \frac{2}{\statmin^3 \mixcoefeps} \mixcoef^k. \label{eq:Hana4}
\end{align}
The last statement is true because one can check that for all $t \in \NN$ we have
\begin{align*}
\min_{z_k, z_0} p(z_k \mid z_0) = \min_{ij} (A^t)_{ij} \geq \min_{ij} (A)_{ij} \geq \mixcoefeps \statmin
\end{align*}
for any stochastic matrix $A$ which satisfies the mixing condition~\eqref{ass:mixing}. 

Substituting \eqref{eq:Hana2} with \eqref{eq:Hana4} and \eqref{eq:Hana3} into \eqref{eq:Hana1}, we obtain
\begin{align*}
|p(z_0) - p(z_0 \mid x_k^n)| &\leq \frac{ \sum_{z_k} p(x_{k+1}^n \mid z_k) p(z_0)}{\sum_{z_k} p(x_{k+1}^n \mid z_k) p(z_k \mid x_k)} \frac{2}{\statmin^2 \mixcoefeps} \mixcoef^k + \frac{p(x_{k+1}^n\mid x_k, z_0) p(z_0)}{p(x_{k+1}^n \mid x_k)} \frac{ \mixcoef^k}{\statmin}\\
&\leq \left(\frac{2 \nstates}{\statmin^3 \mixcoefeps^3}  + \frac{\nstates }{\mixcoefeps^2 \statmin}\right) \mixcoef^k \leq \frac{2 \nstates+1}{\statmin^3 \mixcoefeps^3} \mixcoef^k
\end{align*}
where we use~\eqref{eq:obsratio} to see that
\begin{align*}
\frac{\sum_{z_k} p(x_{k+1}^n \mid z_k) p(z_k \mid x_k, z_0)}{\sum_{z_k} p(x_{k+1}^n \mid z_k) p(z_k \mid x_k)}\leq \frac{\max_{z_k} p(x_{k+1}^n \mid z_k) }{\min_{z_k} p(x_{k+1}^n \mid z_k)} \leq \mixcoefeps^{-2}
\end{align*}
and similarly for the first term.
\end{proof}

\begin{lems}[Filter stability]
\label{lem:filterstab}
For any mixing Markov chain which fulfills condition~\eqref{ass:mixing}, 
the following holds
\begin{equation*}
\|\filterop{i}\dots \filterop{1}(\nu - \nu')\|_\infty \leq \mixcoefeps^{-2} \mixcoefeff^i \|\nu - \nu'\|_1
\end{equation*}
where $\mixcoefeff = 1  - \mixcoefeps \statmin$. 
In particular we have
\begin{equation}
\label{eq:filterfar}
\sup_{z_i} |p(z_i \mid x_{i}^1) - p(z_i \mid x_{-n}^i)| \leq 2 \mixcoefeps^{-2} \mixcoefeff^i.
\end{equation}
\end{lems}

\begin{proof}
Given the mixing assumption~\eqref{ass:mixing} we can show that
$\filterkernel{j}{i}(x|y) \geq \epsilon p_{j|i}(x)$ with $\epsilon =
\mixcoefeps \statmin$ for some probability distribution
$p_{j|i}(\cdot)$. This is because we can lower bound
\begin{align*}
K_{j|i}(z_j \mid z_{j-1}) & = \frac{p(z_j \mid z_{j-1}) p(x_j \mid
  z_j) p(x_{j+1}^{i} \mid z_j)}{\sum_{z_j}p(z_j \mid z_{j-1}) p(x_j \mid
  z_j) p(x_{j+1}^{i} \mid z_j) } \\
& \geq \underbrace{\frac{\mixcoefeps
    \pistat(z_j) p(x_j \mid z_j) p(x_{j+1}^{i} \mid z_j)}{\sum_{z_j}
    \frac{\pistat(z_j)}{\statmin} p(x_j \mid z_j) p(x_{j+1}^{i}\mid z_j)}}_{= : \,
  \epsilon p_j(z_j)}
\end{align*}
with $\epsilon= \mixcoefeps \statmin$. 
This allows us to define the stochastic matrix
\begin{equation*}
Q_{j|i} = \frac{1}{1-\epsilon} (\filterkernel{j}{i} - \epsilon P_{j|i}) \text{
  or } \filterkernel{j}{i} = \epsilon P_{j|i} + (1-\epsilon)Q_{j|i}.
\end{equation*} 
where $(P_{j|i})_{k\ell} = p_{j|i}(\ell)$.
Using $\mixcoefeff = 1 - \epsilon$ we then obtain by induction and
using inequality~\eqref{eq:obsratio}
\begin{align*}
\|(\nu_{1|i} - \nu'_{1\mid i}) \filterkernel{1}{i} \dots
\filterkernel{i}{i}\|_{\infty} & \leq \prod_{j=1}^i (1-\epsilon)
\|(\nu_{1|i} - \nu'_{1 \mid i}) \otimes_{j=1}^i Q_{j|i}\|_2 \\
& \leq \mixcoefeff^i \|\nu_{1 \mid i} - \nu'_{1 \mid i}\|_2
\prod_{j=1}^i\|Q_{j|i}^T\|_{op} \leq \mixcoefeff^i \|\nu_{1 \mid i} -
\nu'_{1 \mid i}\|_2 \\
& \leq \mixcoefeff^i \left\| \frac{p(x_{2}^i \mid
  \cdot)\nu(\cdot)}{\sum_{z_1} p(x_{2}^i \mid z_1)\nu(z_1)}
- \frac{p(x_{2}^i \mid z_1)\nu'(z_1)}{\sum_{z_{\ell}} p(x_{2}^i
  \mid \cdot)\nu'(\cdot)} \right\|_2 \\
&\leq \mixcoefeff^i \Big[ \Big\| \frac{p(x_{2}^i \mid
    \cdot)}{\sum_{z_1} p(x_{2}^i \mid z_1)\nu(z_1)} (\nu(\cdot) -
  \nu'(\cdot))\Big\|_2 \\ 
&+ \Big| \Big(\frac{\sup_z p(x_2^i \mid z_1)}{\sum_{z_1} p(x_{2}^i
    \mid z_1)\nu(z_1)} -\frac{\sup_z p(x_2^i \mid z_1)}{\sum_{z_1} p(x_{2}^i \mid
    z_1)\nu'(z_1)} \Big)\Big| \norm{\nu'(\cdot)}\Big]\\
& \leq \mixcoefeff^i \Big( \frac{\sup_z p(x_{2}^i \mid z_1)}{\inf_z
  p(x_{2}^i \mid z_1)}\Big)^2 \|\nu - \nu'\|_1 \leq
\mixcoefeps^{-2}\mixcoefeff^i \|\nu - \nu'\|_1,
\end{align*}
since $Q_{j|i}$ are stochastic matrices and $\|\nu\|_2\leq \|\nu\|_1$ for
probability vectors. The second statement is readily derived by
substituting $\nu(z_1) = p(z_1)$ and $\nu'(z_1) = p(z_1 \mid
x_{1}^\numobs)$.
\end{proof}


\subsection{Proof of Lemma~\ref{lem:twosided_approx}}
\label{AppLemTwoSided}

Recall the shorthand $\mixcoefeff = 1 - \mixcoefeps \statmin$.  For
the general case observe that
\begin{align*}
& \sup_{z_i} |p(z_i \mid x_{1}^n) - p(z_i \mid
  x_{\blockleftind:\blockrightind}) | \\ 
&\leq
  |p(z_i|x_{i+1}^n)p(z_i|x_{1}^i) -
  p(z_i|x_{i+1}^{\blockrightind})p(z_i|x_{\blockleftind}^i)|
  \frac{p(x_{i+1}^n)}{p(x_{i+1}^n|x_{1}^n)p(z_i)}\\ &+\left|
  \frac{p(x_{\blockrightind+1}^{n}|x_{i+1}^{\blockrightind})p(x_{1}^{\blockleftind-1}|x_{\blockleftind}^{i})
  }{p(x_{\blockrightind+1}^{n},x_{1}^{\blockleftind-1}|x_{\blockleftind}^{\blockrightind})}
  -  1\right|\frac{p(x_{i+1}^{\blockrightind})}{p(x_{\blockleftind}^{\blockrightind}|x_{\blockleftind}^{i})}
  \frac{1}{p(z_i)}.
\end{align*}
From Lemma~\ref{lem:filterstab} we directly obtain the following upper bounds
\begin{align*}
\sup_{z,x} |p(z_i \mid x_{1}^i) - p(z_i\mid x_{i-k}^i)| &\leq \mixcoefeps^{-2} \mixcoefeff^k\\
\sup_{z,x} |p(z_i \mid x_{i+1}^{n}) - p(z_i\mid x_{i+1}^{i+k})| &\leq \mixcoefeps^{-2} \mixcoefeff^k
\end{align*}
where the latter follows because of reversibility assumption
\eqref{ass:reversible} of the Markov chain.
Inequality~\eqref{eq:obsratio} can also be used to show that
$\frac{p(x_{i+1}^{n})}{p(x_{i+1}^{n}|x_{1}^i)} \leq  \nstates
\mixcoefeps^{-2}$.  The first term of the sum is therefore bounded
above by $2 \frac{\nstates  \mixcoefeff^k}{\statmin \mixcoefeps^4}$.\\
For the second term, we again use inequality~\eqref{eq:obsratio} and
Lemma~\ref{lem:filterstab} to observe that
\begin{align*}
&\sup_x \frac{|p(x_{i+k+1}^{n}|x_{i+1}^{i+k}) -
    p(x_{i+k+1}^{n}|x_{1}^{i+k})|}{p(x_{i+k+1}^{n}|x_{1}^{i+k})}\\
 &\leq
  \sup_x \frac{\sum_{z_i} p(x_{i+k+1}^{n}\mid
    z_{i+k})|p(z_{i+k}|x_{i+1}^{i+k}) -
    p(z_{i+k}|x_{1}^{i+k})|}{\sum_{z_i} p(x_{i+k+1}^{n}|z_{i+k})
    p(z_{i+k}|x_{1}^{i+k})}\\ &\leq \sup_x \frac{\sup_{z}
    p(x_{i+k+1}^{n}|z_{i+k})}{\inf_{z}p(x_{i+k+1}^{n}|z_{i+k})}
  \sum_{z_{i+k}} |p(z_{\blockrightind}|x_{i+1}^{i+k}) - p(z_{\blockrightind}|x_{1}^{i+k})|\leq
  C \mixcoefeps^{-4}  \mixcoefeff^k \nstates,\\
 & \sup_x \frac{|p(x_{1}^{\blockleftind-1} \mid x_{\blockleftind}^i)- p(x_{1}^{\blockleftind-1} \mid x_{\blockleftind}^{\blockrightind})|}{ p(x_{1}^{\blockleftind-1} \mid x_{\blockleftind}^{\blockrightind}) } \leq  C\mixcoefeps^{-4} \mixcoefeff^k \nstates
\end{align*}
and $\frac{p(x_{i+1}^{\blockrightind})}{p(x_{\blockleftind}^{\blockrightind}|x_{\blockleftind}^{i})}
  \leq \nstates \mixcoefeps^{2}$ as well as $\frac{p(x_{1}^{\blockleftind-1} \mid x_{\blockleftind}^i)}{ p(x_{1}^{\blockleftind-1} \mid x_{\blockleftind}^{\blockrightind})} \leq \nstates \mixcoefeps^{-2}$.
The second term is therefore bounded by 
\begin{align*}
&\left|
  \frac{p(x_{\blockrightind+1}^{n}|x_{i+1}^{\blockrightind})p(x_{1}^{\blockleftind-1}|x_{\blockleftind}^{i})
  }{p(x_{\blockrightind+1}^{n},x_{1}^{\blockleftind-1}|x_{\blockleftind}^{\blockrightind})}
  -  1\right|\frac{p(x_{i+1}^{\blockrightind})}{p(x_{\blockleftind}^{\blockrightind}|x_{\blockleftind}^{i})}
  \frac{1}{p(z_i)}\\
\leq &\frac{|p(x_{i+k+1}^{n}|x_{i+1}^{i+k}) -
  p(x_{i+k+1}^{n}|x_{1}^{i+k})|}{p(x_{i+k+1}^{n}|x_{1}^{i+k})}
\frac{p(x_{1}^{\blockleftind-1} \mid x_{\blockleftind}^i)}{
  p(x_{1}^{\blockleftind-1} \mid x_{\blockleftind}^{\blockrightind})}
+ \frac{|p(x_{1}^{\blockleftind-1} \mid x_{\blockleftind}^i)-
  p(x_{1}^{\blockleftind-1} \mid
  x_{\blockleftind}^{\blockrightind})|}{ p(x_{1}^{\blockleftind-1}
  \mid x_{\blockleftind}^{\blockrightind}) } \\
\leq &\frac{C \nstates^2
  \mixcoefeff^k}{\mixcoefeps^6}
\end{align*}


\subsection{Proof of Lemma~\ref{lem:expfammixing0}}
\label{sec:expfammixing}

The latter inequality is valid in our particular case because
\begin{align*}
|\cov (z_0, z_{\ell} \mid x_0,\dots, x_k)| &= |\sum_{z_0, z_{\ell}}
z_0 z_{\ell} p(z_{\ell} \mid z_0,x) p(z_0|x) - \sum_{z_0} z_0 p(z_0|x)
\sum_{z_{\ell}} z_{\ell} p(z_{\ell}|x)| \\
&= | \sum_{z_0, z_{\ell}} z_0 z_{\ell} p(z_0|x) (p(z_{\ell}|z_0,x) -
p(z_{\ell} \mid x))| \\ 
& \leq \sup_{z_{\ell},z_0} |p(z_{\ell}|z_0,x) - p(z_{\ell} \mid x)|
\sum_{z_0} \sum_{z_{\ell}} |z_0 z_{\ell}| p(z_0 \mid x)
\end{align*}

Let us now show that $\sup_{z_{\ell},z_0} \big| p(z_{\ell} \mid z_0,x)
- p(z_{\ell} \mid x) \big| \leq \mixcoef^{\ell}$.  Introducing the
shorthand $\Delta(\ell) = p(z_{\ell} = 1 \mid z_{0} = 1, x) -
p(z_{\ell} = 1 \mid z_0 =-1, x)$, we first claim that
\begin{align}
\label{eq:maxatzero} 
|\Delta(1)| & \leq \mixcoef
\end{align}

To establish this fact, note that 
\begin{align*}
\Delta(1) & = \left| \frac{p(x \mid z_{\ell} =1) }{p(x
  \mid z_{\ell-1}=1)} p(z_{\ell} =1 \mid z_{\ell-1}=1) -
\frac{p(x \mid z_{\ell}=1)}{p(x \mid z_{\ell-1}= -1)}
p(z_{\ell} = 1 \mid z_{\ell-1} = -1) \right| \\
& = \frac{ap}{ap + b(1-p)} -\frac{a(1-p)}{a(1-p)+bp} \\
& = \frac{ab}{(ap + b(1-p))(a(1-p)+bp)} (2p-1)
\end{align*}
where we write $a = p( x \mid z_{\ell}=1)$ and $b = p(x \mid z_{\ell}= -1)$.
The denominator is minimized at $p=1$ so that
inequality~\eqref{eq:maxatzero} is shown.  The same argument shows
that $|\Delta(-1)| \leq \mixcoef$.

\emph{Induction step:} Assume that $\Delta(\ell-1) \leq \mixcoef^{\ell-1}$.  
It then follows that
\begin{align*}
&|p(z_{\ell}=1 \mid z_0=1,x) - p(z_{\ell}=1 \mid z_0=-1,x)| \\ 
& = |\sum_{z_{\ell-1}} p(z_{\ell}
  =1 \mid z_{\ell-1},x)p(z_{\ell-1} \mid z_0=1,x) - 
  p(z_{\ell}=1 \mid z_{\ell-1},x)p(z_{\ell-1}|z_0
  =-1,x)| \\ 
%
&= \Delta(1) \Delta(\ell-1) \leq \mixcoef^{\ell}
%
\end{align*}
Since 
\begin{equation*}
p(z_{\ell} = 1 \mid z_{0}= -1,x) - p(z_{\ell}=1 \mid z_{0}=
  1,x) = - p(z_{\ell} = -1 \mid z_{0}= -1,x) + p(z_{\ell}= -1 \mid z_{0}=
  1,x)
\end{equation*}
we use the shorthand $s = p(z_0 = 1 \mid x)$ to obtain
\begin{align*}
&\sup_{z_{\ell},z_{0}} \mid p(z_{\ell} \mid z_{0},x) - p(z_{\ell} \mid x)| \\
& = \sup_{b_{\ell},b_{0}} p(z_{\ell}=b_{\ell} \mid z_{0}=b_{0},x) -
  [(p(z_{\ell}=b_{\ell} \mid z_{0}=1,x)s +
    p(z_{\ell}=b_{\ell} \mid z_{0}=-1,x)(1-s)] \\ 
%
%
&\leq (1-s) |\Delta(\ell)| \leq \mixcoef
%
\end{align*}
which proves the bound for $\condcov{Z_0}{Z_1}{\paramgamma}$.

For the two state mixing we define $\widetilde{\Delta}(\ell) = p(z_{\ell}=1 \mid z_{1}z_{0}= 1,x) - p(z_{\ell} =
1 \mid z_{1}z_{0} = -1,x)$
and can readily see that $|\widetilde{\Delta}(1)|\leq \mixcoef$ and 
\begin{align*}
&|p(z_{\ell+1}z_{\ell+2}=1 \mid z_{\ell} z_{\ell-1}=1,x) -
  p(z_{\ell+1}z_{\ell+2} = 1 \mid z_{\ell} z_{\ell-1}=-1,x)| \\
 & = [p(z_{\ell+2} = 1 \mid z_{\ell+1}=1,x) -p(z_{\ell+2}=-1 \mid z_{\ell+1}=-1,x)] \tilde{\Delta}(2) 
\end{align*}
Using equation~\eqref{eq:maxatzero}, we obtain
\begin{align}
\label{eq:useformixedmixing} 
|\widetilde{\Delta}(2)| = |p(z_1 = 1 \mid z_0 = 1, x) - p(z_1 = -1 \mid z_0 = -1, x)|\widetilde{\Delta}(1)  \leq \mixcoef 
\end{align}
from which it directly follows that
\begin{align*}
|p(z_{\ell+1}z_{\ell+2}=1 \mid z_{\ell} z_{\ell-1}=1,x) -
p(z_{\ell+1}z_{\ell+2} = 1 \mid z_{\ell} z_{\ell-1}=-1,x)| \leq \mixcoef
\end{align*}
The rest follows the same arguments as above and the bound for
$\condcov{Z_0Z_1}{Z_\ell Z_{\ell+1}}{\paramgamma}$ in
inequality~\eqref{eq:expfamonestatemixing} is shown. \\

Finally, the bound for $\condcov{Z_0}{Z_\ell Z_{\ell+1}}{\paramgamma}$ in
inequality~\eqref{eq:expfamonestatemixing} follows in a straightforward way
using the relation~\eqref{eq:useformixedmixing} and induction with
equation~\eqref{eq:maxatzero}, as above.

\bibliographystyle{myIEEEtran} 
\bibliography{bibliography}

\begin{thebibliography}{10}
\providecommand{\url}[1]{#1}
\csname url@rmstyle\endcsname
\providecommand{\newblock}{\relax}
\providecommand{\bibinfo}[2]{#2}
\providecommand\BIBentrySTDinterwordspacing{\spaceskip=0pt\relax}
\providecommand\BIBentryALTinterwordstretchfactor{4}
\providecommand\BIBentryALTinterwordspacing{\spaceskip=\fontdimen2\font plus
\BIBentryALTinterwordstretchfactor\fontdimen3\font minus
  \fontdimen4\font\relax}
\providecommand\BIBforeignlanguage[2]{{%
\expandafter\ifx\csname l@#1\endcsname\relax
\typeout{** WARNING: IEEEtran.bst: No hyphenation pattern has been}%
\typeout{** loaded for the language `#1'. Using the pattern for}%
\typeout{** the default language instead.}%
\else
\language=\csname l@#1\endcsname
\fi
#2}}

\bibitem{durbinbook}
\BIBentryALTinterwordspacing
R.~Durbin, \emph{Biological Sequence Analysis: Probabilistic Models of Proteins
  and Nucleic Acids}.\hskip 1em plus 0.5em minus 0.4em\relax Cambridge
  University Press, 1998.
\BIBentrySTDinterwordspacing

\bibitem{rabinerbook}
\BIBentryALTinterwordspacing
L.~Rabiner and B.~Juang, \emph{Fundamentals of Speech Recognition}, ser.
  Pearson Education Signal Processing Series.\hskip 1em plus 0.5em minus
  0.4em\relax Pearson Education, 1993.
\BIBentrySTDinterwordspacing

\bibitem{controlbook}
\BIBentryALTinterwordspacing
R.~Elliott, L.~Aggoun, and J.~Moore, \emph{{H}idden {M}arkov {M}odels:
  Estimation and Control}, ser. Applications of Mathematics.\hskip 1em plus
  0.5em minus 0.4em\relax Springer, 1995.
\BIBentrySTDinterwordspacing

\bibitem{econbook}
\BIBentryALTinterwordspacing
C.~Kim and C.~Nelson, \emph{State-space Models with Regime Switching: Classical
  and Gibbs-sampling Approaches with Applications}.\hskip 1em plus 0.5em minus
  0.4em\relax MIT Press, 1999.
\BIBentrySTDinterwordspacing

\bibitem{bickel1998}
P.~Bickel, Y.~Ritov, and T.~Ryd{\'e}n, ``Asymptotic normality of the
  maximum-likelihood estimator for general {H}idden {M}arkov {M}odels,''
  \emph{The Annals of Statistics}, vol.~26, no.~4, pp. 1614--1635, 08 1998.

\bibitem{cryptohmm}
S.~Terwijn, ``On the learnability of {H}idden {M}arkov {M}odels,'' in
  \emph{Proceedings of the 6th International Colloquium on Grammatical
  Inference: Algorithms and Applications}, ser. ICGI '02.\hskip 1em plus 0.5em
  minus 0.4em\relax London, UK, UK: Springer-Verlag, 2002, pp. 261--268.

\bibitem{Baum70}
L.~Baum, T.~Petrie, G.~Soules, and N.~Weiss, ``A maximization technique
  occurring in the statistical analysis of probabilistic functions of {M}arkov
  chains,'' \emph{The Annals of Mathematical Statistics}, pp. 164--171, 1970.

\bibitem{Dempster77}
A.~Dempster, N.~Laird, and D.~Rubin, ``Maximum likelihood from incomplete data
  via the {EM} algorithm,'' \emph{J. R. Stat. Soc. B}, pp. 1--38, 1977.

\bibitem{mossel2006}
E.~Mossel and S.~Roch, ``Learning nonsingular phylogenies and {H}idden {M}arkov
  {M}odels,'' \emph{The Annals of Applied Probability}, vol.~16, no.~2, pp.
  583--614, 05 2006.

\bibitem{Siddiqi10}
S.~M. Siddiqi, B.~Boots, and G.~J. Gordon, ``Reduced-rank {H}idden {M}arkov
  {M}odels,'' in \emph{Proc. 13th International Conference on Artificial
  Intelligence and Statistics}, 2010.

\bibitem{Hsu12}
D.~Hsu, S.~Kakade, and T.~Zhang, ``A spectral algorithm for learning {H}idden
  {M}arkov {M}odels,'' \emph{J. Comput. Syst. Sci.}, vol.~78, no.~5, pp.
  1460--1480, 2012.

\bibitem{Kontorovich13}
L.~A. Kontorovich, B.~Nadler, and R.~Weiss, ``On learning parametric-output
  {HMM}s,'' in \emph{Proc. 30th International Conference Machine Learning},
  June 2013, pp. 702--710.

\bibitem{Chaganty13}
A.~Chaganty and P.~Liang, ``Spectral experts for estimating mixtures of linear
  regressions,'' \emph{arXiv preprint arXiv:1306.3729}, 2013.

\bibitem{BalWaiYu14}
S.~Balakrishnan, M.~J. Wainwright, and B.~Yu, ``Statistical guarantees for the
  {EM} algorithm: From population to sample-based analysis,'' \emph{arXiv
  preprint arXiv:1408.2156}, 2014.

\bibitem{YiCar15}
X.~Yi and C.~Caramanis, ``Regularized {EM} algorithms: A unified framework and
  provable statistical guarantees,'' \emph{arXiv preprint arXiv:1511.08551},
  2015.

\bibitem{WangLiu14}
Z.~Wang, Q.~Gu, Y.~Ning, and H.~Liu, ``High dimensional
  expectation-maximization algorithm: Statistical optimization and asymptotic
  normality,'' \emph{arXiv preprint arXiv:1412.8729}, 2014.

\bibitem{Moulines_HMM}
O.~Capp{\'e}, E.~Moulines, and T.~Ryd{\'e}n, ``{H}idden {M}arkov {M}odels,''
  2004.

\bibitem{vanHandel_HMM}
R.~van Handel, ``{H}idden {M}arkov {M}odels,'' \emph{Unpublished lecture
  notes}, 2008.

\bibitem{baum1966}
L.~Baum and T.~Petrie, ``Statistical inference for probabilistic functions of
  finite state {M}arkov chains,'' \emph{The Annals of Mathematical Statistics},
  pp. 1554--1563, 1966.

\bibitem{Wu83}
J.~C. Wu, ``On the convergence properties of the {EM} algorithm,'' \emph{The
  Annals of Statistics}, pp. 95--103, 1983.

\bibitem{dasgupta}
S.~Dasgupta, ``Learning mixtures of {G}aussians,'' in \emph{40th Annual
  Symposium on Foundations of Computer Science, {FOCS} '99, 17-18 October,
  1999, New York, NY, {USA}}, 1999, pp. 634--644.

\bibitem{vempala}
\BIBentryALTinterwordspacing
S.~Vempala and G.~Wang, ``A spectral algorithm for learning mixture models,''
  \emph{J. Comput. Syst. Sci.}, vol.~68, no.~4, pp. 841--860, June 2004.
\BIBentrySTDinterwordspacing

\bibitem{belkin}
M.~Belkin and K.~Sinha, ``Toward learning {G}aussian mixtures with arbitrary
  separation,'' in \emph{{COLT} 2010 - The 23rd Conference on Learning Theory,
  Haifa, Israel, June 27-29, 2010}, 2010, pp. 407--419.

\bibitem{moitra}
\BIBentryALTinterwordspacing
A.~Moitra and G.~Valiant, ``Settling the polynomial learnability of mixtures of
  {G}aussians,'' in \emph{Proceedings of the 2010 IEEE 51st Annual Symposium on
  Foundations of Computer Science}, ser. FOCS '10.\hskip 1em plus 0.5em minus
  0.4em\relax Washington, DC, USA: IEEE Computer Society, 2010, pp. 93--102.
\BIBentrySTDinterwordspacing

\bibitem{hsumog}
\BIBentryALTinterwordspacing
D.~Hsu and S.~Kakade, ``Learning mixtures of spherical {G}aussians: Moment
  methods and spectral decompositions,'' in \emph{Proceedings of the 4th
  Conference on Innovations in Theoretical Computer Science}, ser. ITCS
  '13.\hskip 1em plus 0.5em minus 0.4em\relax New York, NY, USA: ACM, 2013, pp.
  11--20.
\BIBentrySTDinterwordspacing

\bibitem{Yu94}
B.~Yu, ``Rates of convergence for empirical processes of stationary mixing
  sequences,'' \emph{The Annals of Probability}, pp. 94--116, 1994.

\bibitem{NobDem93}
A.~Nobel and A.~Dembo, ``A note on uniform laws of averages for dependent
  processes,'' \emph{Statistics \& Probability Letters}, vol.~17, no.~3, pp.
  169--172, 1993.

\bibitem{Frank01}
F.~Kschischang, B.~Frey, and H.-A. Loeliger, ``Factor graphs and the
  sum-product algorithm,'' \emph{IEEE Trans. Info. Theory}, vol.~47, no.~2, pp.
  498--519, February 2001.

\bibitem{WaiJor08}
M.~J. Wainwright and M.~I. Jordan, ``Graphical models, exponential families and
  variational inference,'' \emph{Foundations and Trends in Machine Learning},
  vol.~1, no. 1--2, pp. 1---305, December 2008.

\bibitem{vdVWell}
A.~W. Van Der~Vaart and J.~A. Wellner, \emph{Weak Convergence and Empirical
  Processes}.\hskip 1em plus 0.5em minus 0.4em\relax Springer, 1996.

\bibitem{Led96}
M.~Ledoux, ``On {T}alagrand's deviation inequalities for product measures,''
  \emph{ESAIM: Probability and statistics}, vol.~1, pp. 63--87, 1997.

\bibitem{LedTalBanach}
M.~Ledoux and M.~Talagrand, \emph{{P}robability in {B}anach {S}paces:
  {I}soperimetry and {P}rocesses}.\hskip 1em plus 0.5em minus 0.4em\relax
  Springer Science \& Business Media, 2013, vol.~23.

\end{thebibliography}

\end{document}